\documentclass[12pt, 
english,
]{article}
\usepackage{custom_tex}

\title{Minimax Statistical Estimation under Wasserstein Contamination}

\author{
Patrick Chao
and Edgar Dobriban\footnote{
The authors are with the Department of Statistics and Data Science, University of Pennsylvania. E-mail addresses: 
\texttt{pchao@wharton.upenn.edu},
\texttt{dobriban@wharton.upenn.edu}.}
}
\date{\today}

\begin{document}
\maketitle
\abstract{

Contaminations are a key concern in modern statistical learning, as small but systematic perturbations of all datapoints can substantially alter estimation results. 
Here, we study \emph{Wasserstein-$r$ contaminations} ($r\ge 1$) in an $\ell_q$ norm ($q\in [1,\infty]$), in which each observation may undergo an adversarial perturbation with bounded cost, complementing the classical Huber model, corresponding to total variation norm, where only a fraction of observations is arbitrarily corrupted. 
We study both independent and joint (coordinated) contaminations and develop a minimax theory under $\ell_q^r$ losses.

Our analysis encompasses several fundamental problems: location estimation, linear regression, and pointwise nonparametric density estimation. For joint contaminations in location estimation and for prediction in linear regression, we obtain the \emph{exact minimax risk}, identify \emph{least favorable contaminations}, and show that the \emph{sample mean} and \emph{least squares predictor} are respectively minimax optimal. 
For location estimation under independent contaminations, we give sharp upper and lower bounds, including exact minimaxity 
in the Euclidean Wasserstein contamination case, 
when $q=r=2$. For pointwise density estimation in any dimension, we derive the optimal rate, showing that it is achieved by kernel density estimation with a bandwidth that is possibly larger than the classical one.

Our proofs leverage powerful tools from optimal transport developed over the last 20 years, including the dynamic Benamou--Brenier formulation. 
Taken together, our results suggest that in contrast to the Huber contamination model, for norm-based Wasserstein contaminations, classical estimators may be nearly optimally robust. 
}
\clearpage
\tableofcontents
\medskip

\section{Introduction}

In any data collection process, there may be
contaminations that adversely affect the fidelity of the data, e.g., 
historical data may not represent the  population of interest, or a sensor may be systematically biased. 
Despite the presence of contaminations, we aim for accurate statistical estimation.
Concretely, suppose we are interested in estimating a parameter $\theta$ from an observation $Z\sim P_\theta$, but we observe \textit{contaminated} 
data $Z'$ from $\nu$, where $P_\theta$ and $\nu$ are ``close", i.e., $d(P_\theta,\nu)\le \ep$ for some metric $d$ and perturbation size $\ep\ge 0$. Is it still possible to estimate $\theta$ given data
perturbed by possibly adversarial
contamination? 
What are the best possible robust estimators $\htheta$?

Total variation (TV) perturbations---the choice $d=\mathrm{TV}$---correspond to Huber's well studied $\ep$-contamination model  \citep[e.g.,][etc]{huberLocation,huberRatio,huber2004robust,hampel2005robust,chen2016general, chen2018robust,zhu2020robuststat}. In this setting, a fraction $\ep$ of the data may be arbitrarily corrupted, representing possible outliers in the data collection process.
Standard estimators such as the sample mean are very sensitive to this contamination, and other estimators, such as median-type methods, are required. 

In the Huber $\ep$-contamination model, only a fraction $\ep$ of the data points is modified. 
Motivated by a number of settings (see e.g., Section \ref{app: motivation wasserstein}),
we consider a different regime, where it is possible for \textit{all observations} to be perturbed by a small quantity. This can be captured by different choices of $d$, of which 
one of the canonical choices
is the Wasserstein distance from optimal transport \citep[e.g.,][etc]{villani2003topics,villani2008optimal}. 
For a clean observation $Z$ and perturbed observation $Z'$ in 
$\R^p$,
and some
a Wasserstein-$r$ contamination ($r\ge 1$)
in the $\ell_q$ 
norm ($q\in [1,\infty]$)
keeps
$(\mathbb{E}{\|Z-Z'\|_q^r})^{1/r} \le \ep$,
whereas a total variation contamination bounds $\PP{Z\neq Z'} \le \ep$.

\begin{table}[t]
\centering
\caption{Summary of our results on minimax risks under Wasserstein-$r$ contamination in $\ell_q$, $r\ge 1$, $q\in[1,\infty]$.
JC denotes joint contaminations where perturbations can be coordinated across datapoints, and IC denotes independent contaminations across data points; with $\mathcal M_J$ and $\mathcal M_I$ respectively denoting the corresponding Minimax risks. 
Here $=$ denotes an exact expression; $\asymp$ denotes an order-equivalence (rate). 
For location estimation, $E_i$, $i=1,\ldots,n$ denote the noise variables, and allow noise large enough that $\big(\E\|\bar E\|_q^r\big)^{1/r} \le s_n$ for some $s_n>0$. 
Also, 
in the linear regression problem, $Y=X\theta+E$,
allow noise large enough that  
$\big(\E\|P_XE\|_q^r\big)^{1/r} \le s_{\rm pred}$ for some $s_{\rm pred}>0$, where $P_X$ is the projection to the column span of $X$.}
\label{sum}
\scalebox{1}{
\begin{tabular}{@{}p{4.2cm}p{3cm}p{9.5cm}@{}}
\toprule
\textbf{Problem} & \textbf{Contamination} & \textbf{Minimax risk (exact / rate)} \\
\midrule

\textbf{Location estimation (JC)}; loss $\|\hat\theta-\theta\|_q^r$
& $W_{q,r}$ in $\ell_q$, \newline budget $\ep$
& $\mathcal M_J(\varepsilon,q,r) = (s_n+\varepsilon)^r$
\\

\midrule

\textbf{Location estimation (IC)}; loss $\|\hat\theta-\theta\|_q^r$
&  $W_{q,r}$ in $\ell_q$, \newline budget $\ep$
& \textbf{(A)} $\mathcal M_I(\varepsilon,q,r) \asymp (s_n+\varepsilon)^r.$
\vspace{0.5em}
\newline
\textbf{(B)} Exact risk when $q=r=2$, with $\E\|E_1\|_2^2 \le \sigma^2$:
\newline
$\mathcal M_I(\varepsilon,2,2)=
\begin{cases}
\frac{1}{n}(\varepsilon+\sigma)^2, & \varepsilon^2 \le \sigma^2/(n-1)^2,\\
\varepsilon^2 + \sigma^2/(n-1), & \varepsilon^2 > \sigma^2/(n-1)^2.
\end{cases}$
\\

\midrule

\textbf{Linear regression, \newline prediction (JC)};
\newline loss $\|\hat m(Y')-X\theta\|_q^r$
& Output-only $W_{q,r}$ on $\R^n$, budget $\rho$
& \textbf{(A)} $(s_{\rm pred}+\rho)^r \le \mathcal M_J^{\rm pred}(\rho,q,r) \le (s_{\rm pred}+\|P_X\|_{q\to q}\rho)^r.$
\vspace{0.5em}
\newline
\textbf{(B)} Exact risk when $q=2$: $\mathcal M_J^{\rm pred}(\rho,2,r) = (s_{\rm pred}+\rho)^r$.
\\

\midrule

\textbf{Density estimation (IC)};
loss $\Phi(|\hat f(x_0)-f(x_0)|)$
& $W_{q,r}$ in $\ell_q$, \newline budget $\ep$
& $\mathcal M_I(\varepsilon,s,p,r) \asymp \Phi\!\Big(n^{-s/(2s+p)} \vee \varepsilon^{s/(s+1+p/r)}\Big).$
\\

\bottomrule
\end{tabular}
}
\end{table}

These types of contaminations 
may be reflected in a number of applied domains.
Consider an example
where we aim to classify tissue datapoints from hospitals as normal or cancerous;  such as in the Camelyon17 dataset \citep{Bndi2019FromDO}.
Different hospitals may exhibit differences in their imaging techniques---a systematic difference that ``shifts" every observation a bit, as opposed to creating outliers. 
In this setting, the contaminations between hospitals might more reasonably be characterized by---for instance---a Wasserstein contamination error rather than by Huber contamination.\footnote{See Section \ref{app: motivation wasserstein} in the Appendix for additional motivation.} 

These considerations have led to a large body of work on 
distributionally robust optimization (DRO) and estimators to overcome such challenges, see for instance, \cite{kuhn2019wasserstein,rahimian2019distributionally,chen2020distributionally}, and Section \ref{subsec: related work} for an overview.
However, the fundamental limits of statistical estimation under Wasserstein-type contaminations, as measured through minimax risks, have been studied only in a few works, such as \cite{zhu2020robuststat,Liu2021RobustWE}. 
Our work aims to deepen our understanding of this area. 

In our work, 
we consider an analyst who studies a dataset, aiming to estimate a parameter in the
$\ell_q^r$ loss. 
However, the dataset is perturbed by
a Wasserstein-$r$ contamination 
in the $\ell_q$ 
norm.
In this paper, we provide the following: 
\begin{enumerate}
    \item A formulation of two types of contaminations of an independent and identically distributed (i.i.d.) sample, including a strong form of \emph{Joint Contamination} (JC), where the adversary can perturb observations in a coordinated way. 
    We also consider a lighter form of \emph{i.i.d.~Contamination} (IC), where the adversary can only perturb each data point's distribution in an i.i.d.~way. 
    \item We give a precise analysis of several fundamental
    statistical problems under these forms of contamination: location estimation (including normal mean estimation), linear regression, and non-parametric density estimation; see Table \ref{sum} for a summary.

    \item For JC location estimation, we find the \emph{exact minimax risks under Wasserstein $\ell_q^r$ perturbations} (Theorem~\ref{thm:JC-Lqr}),
    showing that the sample mean is minimax optimal in $\ell_q^r$ loss.  Specifically, denoting by $\M_J(\ep,q,r)$ the minimax risk in a location estimation problem, under JC Wasserstein $\ell_q^r$ perturbations of size $\ep\ge0$, we show that
     \[
\M_J(\ep,q,r)^{1/r}
=
\M_J(0,q,r)^{1/r}
+ \ep.
\]
 We emphasize that this result is \emph{exact in any finite sample}. 
 It has no unknown or hidden constants,
 it is non-asymptotic, and it is not just about the rate, 
 but rather gives the precise value of the minimax risk. 
This expression shows that perturbations increase $\M_J(0,q,r)^{1/r}$ linearly in $\ep$. 
The proofs of the lower bound have certain elements of novelty in the literature on robustness.
 They consist of a 
a construction of a feasible class of contaminations,  
and reductions that let us pass from arbitrary estimators
to equivariant estimators based on the Hunt-Stein theorem; which can then be handled analytically.

    For IC location estimation, the sample mean remains approximately optimal (Theorem ~\ref{thm:IC-struct}), while being still exactly optimal for $q=r=2$ (Theorem~\ref{thm: gaussian mean opt}).  
    Specifically, denote by $\M_I(\ep, 2,2)$ the
    minimax risk measured in an $\ell_2^2$ loss,
    and noise variables such that 
    $ \E\|E_1\|_2^2 \le \sigma^2$,
    under IC $\ell_2^2$ Wasserstein perturbations of size $\ep\ge0$.
    Then, for $n\ge 2$, 
the risk takes one of two forms depending on the relative sizes of $\varepsilon, \sigma^2$, and $n$: 
    \begin{align*}
        \M_I(\ep, 2,2) & =\begin{cases}
            \frac{1}{n}\left(\ep+\sigma\right)^2 & \; \ep^2 \le \frac{\sigma^2}{(n-1)^2}, \\
            \ep^2 + \frac{\sigma^2}{n-1}                   & \; \ep^2 > \frac{\sigma^2}{(n-1)^2}.
        \end{cases}
    \end{align*}
When $n=1$, $\M_I(\ep,2,2) =
            \left(\ep+\sigma\right)^2$.
Observe that there is a transition at $\ep = \sigma/(n-1)$.
 It turns out that below this perturbation level, it is a stronger contamination to add noise, whereas above it, it is stronger to add bias.
These insights arise out of a careful exact analysis of the minimax risk and worst-case perturbations.
 Again, we emphasize that this is an exact formula, and not just an approximation or a rate.

    These results present an intriguing contrast with robustness in the Huber contamination model: if each observation is perturbed a bit, we may not be able to do much better than using classical averaging-based estimators.

    \item For linear regression under JC perturbations, we show a similar phenomenon: the least squares predictor remains approximately minimax optimal in prediction error (Theorems~\ref{thm:lmr-master} and \ref{thm:minimax-pred}). 
    Specifically, denoting by $\M_J^{\rm pred}(\rho,q,r)$ 
    the minimax prediction error in a linear regression
    problem, under JC Wasserstein $\ell_q^r$ perturbations of size\footnote{As explained in more detail later, the notation $\rho$ is used here because it is not always straightforward to compare the meaning of the magnitude of $\rho$ to that of a perturbation in location estimation.} 
    $\rho\ge0$, we show that
$$
1 \le\
\frac{\M_J^{\rm pred}(\rho,q,r)^{1/r}-\M_J^{\rm pred}(0,q,r)^{1/r}}{\rho}
 \le\
\|P_X\|_{q\to q},
$$
where $\|P_X\|_{q\to q}$ is the induced $\ell_q\to\ell_q$ norm of the
 projection matrix $P_X$ onto the column span of 
 the feature matrix
 $X$. 
When $q=2$, this norm is equal to unity, implying that 
the prediction risk is exactly characterized by 
 $
\M_J^{\rm pred}(\rho,2,r)^{1/r}
=
\M_J^{\rm pred}(0,2,r)^{1/r}
+\rho.
$
This is yet another exact formula. 

Analogously, 
we  show 
that the OLS estimator
is approximately optimal in estimation error (Theorem~\ref{thm:lmr-master}).
Denoting by 
$\M_J^{\rm pred}(\rho,q,r)$ 
    the minimax prediction error in a linear regression
    problem, under JC Wasserstein $\ell_q^r$ perturbations of size $\rho$, we characterize the minimax estimation risk up to constant factors, showing that 
$$ 
\frac{1}{2}
(\M_J^{\rm est}(0,q,r)^{1/r} + \rho) \le
\M_J^{\rm est}(0,q,r)^{1/r} \le \M_J^{\rm est}(\rho,q,r)^{1/r} + \rho.$$

Analogous results hold under IC contaminations as well (see Theorem~\ref{thm:master} and Corollary~\ref{cor:IC-q2r1}).

    \item For non-parametric density estimation, we find the minimax optimal rate of estimation.  
    Specifically, 
    for a convex non-decreasing loss $\Phi$, H\"older smooth density $s$ in dimension $p$, and Wasserstein-$r$ perturbations, 
    the minimax optimal risk 
    $\M_{I}(\ep, s, p, r)$
    under Wasserstein $\ep$-IC perturbations in an $\ell_q^r$ norm is
        is the maximum of the classical non-parametric rate and a term depending on the perturbation size $\ep$ (Theorem~\ref{thm: pointwise density bounds}):  
\begin{equation*}
\M_{I}(\ep, s, p, r)
\ \asymp\
\Phi\!\Big(n^{-\frac{s}{2s+p}} \vee\ \ep^{\frac{s}{ s+1+p/r }}\Big).
\end{equation*}
    Moreover, this rate is achieved by kernel density estimation with a bandwidth 
    \(
h^\star \asymp
  n^{-\frac{1}{2s+p}}
   \vee 
  \ep^{\frac{1}{ s+1+p/r }}.
  \)
    possibly larger than the classical one, depending on the perturbation size (Theorem~\ref{thm:upper-Wr-general-Phi-hetero-d}).
 Intriguingly, this result holds even if we allow the perturbed density to have a different smoothness level than the original one. 

\item In terms of technical innovation, our proofs leverage several properties of optimal transport that are perhaps not widely used in the statistical literature, to our knowledge. These include the dynamic Benamou-Brenier formulation of optimal transport, as well as boundedness properties of the probability flow induced by optimal transport maps. 
\end{enumerate}

\subsection{Notations}
For a positive integer $n$, we denote $[n]:=\{1,2,\ldots,n\}$.
We denote by $\mathcal{Z}$ the observation space,
a subset of some Euclidean space.
All random variables we consider are defined on a common probability space $(\Omega,\mathcal{F},\mathbb{P} )$.
For a family of probability distributions $\cP$ of interest, we will denote the estimand of interest by the functional $\theta:\mathcal{P} \to \R^p$, and the clean data distribution by $\mu\in\cP$, implying the parameter of interest is $\theta(\mu)$.
When the parameter is identifiable and parametrizes the probability distributions under consideration, we may also denote the data distribution as $P_\theta$ for $\theta\in\Theta$.

We arrange $n$ observations $Z_i\in \mathcal{Z}$ (resp., $Z_i'\in \mathcal{Z}$) into\footnote{In such cases, we will not distinguish between row and column vectors.} $Z=(Z_1,\ldots,Z_n)\in \mathcal{Z}^n$ (resp. $Z'=(Z_1',\ldots,Z_n')$). 
When $Z  = (Z_1,\ldots, Z_n)$ where each $Z_i \sim \mu$ are i.i.d.~random vectors, we will write $Z\sim \mu^n$.
For a vector $v\in \R^k$ and $q \in [1,\infty]$,
we denote by $\|v\|_q$ the standard $\ell_q$ norm.

Let $\mu$ and $\nu$ be probability measures on $\R^k$, $k\ge 1$.
For $r\ge 1$, 
the Wasserstein-$r$ distance associated with the norm $\|\cdot\|$ is defined as\footnote{In our work, all probability distributions considered will be such that the displayed moments are finite; this will sometimes be kept implicit, and no confusion shall arise. }
\[W_{\|\cdot\|, r}(\mu,\nu)=\inf_{(Z,Y)\sim \Pi(\mu,\nu) } \left\{\E [\|Z-Y\|^r]\right\}^{1/r}, \]
where the infimum is taken over all couplings $\Pi(\mu,\nu)$ between $\mu$ and $\nu$.
For the special case of the $\ell_q$ norm,  
$q\in[1,\infty]$, 
we will write
$W_{q, r}$.
The optimal coupling between $\mu$ and $\nu$
exists in all cases we study \citep[e.g.,][]{villani2008optimal,santambrogio2015optimal,figalli2021invitation}.

Let $\mathcal P_{q,r}$ denote the set of probability laws on the relevant Euclidean space with finite $r$th moment in the $\ell_q$ norm, i.e. $\mu\in\mathcal P_{q,r}$ iff $\int \|z\|_q^r,d\mu(z)<\infty$. Throughout, all “clean” and “perturbed” laws we consider lie in $\mathcal P_{q,r}$, so all Wasserstein distances $W_{q,r}$ we invoke are finite and all $r$th–moment risks we compute are well defined. 
This will not be repeated further.

\subsection{Contamination Models}\label{subsec: dist shift models}
\label{mo}

Given a family  $\cP$ of distributions on $\mathcal{Z}$,
in a standard statistical setting we observe
i.i.d.\footnote{While we work with an i.i.d.~sample, our definitions and results include arbitrary, potentially non-i.i.d.~data distributions as a special case by taking the sample size $n=1$; and viewing the one observation as the dataset.}~$Z_i\sim \mu$  for $i\in[n]$ sampled from an unknown $\mu\in \cP$.
Instead, here we
observe potentially \emph{contaminated}
$Z_i'$, $i\in [n]$ that are close in distribution to $Z_i$.

For a family $\cP$ of clean distributions
and a family $\cPo^n$ of possible perturbed distributions,
we consider an \emph{ambiguity set}\footnote{Following terminology in \cite{lee2017minimax}.}  $\cV\subset \cP\times \cPo^n=\{(\mu,\vec{\nu}):\mu\in \cP, \vec{\nu}\in\cPo^n\}$,
which represents the possible contaminations under consideration.
For  $( \mu,\vec{\nu})\in \cV$,
the clean data is $Z = (Z_1,\ldots, Z_n)$
with i.i.d.~$Z_i\sim \mu$,
but we observe $Z' = (Z_1', \ldots, Z_n') \sim \vec{\nu}$.
We consider the following contaminations and associated ambiguity sets.
\begin{definition}[Contaminations]\label{def: ds}
    For a family of distributions $\cP \subset \mathcal P_{q,r}$ where $\mu\in\cP$, 
    define the joint and i.i.d. contamination ambiguity sets as follows:
    \begin{enumerate}

        \item {\bf I.I.D.~Contamination (IC):} 
        For a perturbation strength $\ep\ge 0$,
        we observe i.i.d.~$Z_i'\sim \nu$, $i\in[n]$, where 
        $W_{q, r}(\mu,\nu) \le \ep$, for $i\in [n]$, so 
        $\cV_{I}(\ep)\coloneqq \{(\mu,\nu^n):  \mu\in\cP,  W_{q, r}(\mu,\nu)\le \ep\}.$
    
        \item {\bf Joint Contamination (JC):} 
        Endow $\mathcal{Z}^n$ with the scaled norm, where $r\ge 1$, $q\in[1,\infty]$:
\begin{equation}\label{pn}
    \|(z_1,\dots,z_n)-(z'_1,\dots,z'_n)\|
:=\Big(\frac1n\sum_{i=1}^n \|z_i-z'_i\|_q^{ r}\Big)^{1/r}.
\end{equation}
        For some $\ep\ge 0$,
        we observe $Z'\sim \vec{\nu}$, such that 
        $W_{\|\cdot\|, r}(\mu^n,\vec\nu) \le \ep$,
        so 
        $\cV_{J}(\ep)\coloneqq \{(\mu,\vec\nu):  \mu\in\cP,  W_{\|\cdot\|, r}(\mu^n,\vec\nu)\le \ep\}.$

    \end{enumerate}
\end{definition}

Examples of such contaminations are in Table \ref{table: dist shift perturbations}.
The scaling in \eqref{pn} 
ensures that the average per-datapoint perturbation of each observation is at most $\ep$, and so ensures that the meaning of $\ep$ is comparable across the two settings.
 Moreover, since joint contaminations can coordinate across observations, they are stronger than independent contaminations,
 so that $\cV_{I}( \ep) \subset \cV_{J}(\ep)$; see Section \ref{relc} for a brief argument.

{
\renewcommand{\arraystretch}{1.2}
\begin{table}[t]
    \caption{Examples of joint (JC) and independent (IC) contaminations. We consider real-valued observations.
    Here $\bar Z$ is the sample mean, $Z_{(1)}$ is the minimum of the real-valued observations, $\psi$ and $\psi'$ are scalar parameters and $\mathbbm{1}$ denotes the indicator function.}
    \centering
    \begin{tabular}{l l }
        \toprule
        Contamination & \qquad\qquad\qquad\qquad Examples \\
        \midrule
        JC                & $\begin{array}{l}
                Z_i'=Z_i+\psi(\bar Z-\theta);  \quad
                Z_i'=Z_i+\psi'\mathbbm{1}\{Z_i=Z_{(1)}\}
            \end{array}$   \\[0.1cm]
        IC                & $\begin{array}{l}
                Z_i'=Z_i+\psi(Z_i-\theta);  \quad
                Z_i'=Z_i+\psi' \mathbbm{1}\{Z_i>\theta\}
            \end{array}$   \\[0.1cm]
        \bottomrule
    \end{tabular}
    
    \label{table: dist shift perturbations}
\end{table}
}

\subsection{Minimax Estimators and Risk}\label{subsec: minimax est and risk}

Consider  estimators $\htheta:\cZ^n \to \Theta$,
measurable functions from the domain $\cZ^n$ to the parameter space $\Theta$.
Next, we define the minimax risk under consideration.

\begin{definition}[Minimax risk and estimators]\label{def: minimax risk and est}
    The minimax risk for the parameter $\theta$, ambiguity set $\cV$, and loss function $\mathcal{L}: \Theta\times \Theta \to [0,\infty)$, is the largest value that the risk can take over the entire ambiguity set $\mathcal{V}$:
    \begin{equation}\label{eq: def minimax risk}
        \M(\cV,\theta)
        \coloneqq
        \inf_{\htheta:\cZ^n\to \Theta } 
        \sup_{(\mu,\vec{\nu})\in \mathcal{V}} \mathbb{E}_{Z'\sim \vec{\nu}}{\cL\bigl(\htheta(Z'), \theta(\mu)\bigr)}.
    \end{equation}
    Define $\M_J$ and $\M_I$ as the minimax risks for the ambiguity sets from Definition \ref{def: ds}.
    An estimator $\htheta$ is a JC $\ep$--\emph{minimax estimator} if it achieves the JC $\ep$--minimax risk $\M_J$,
    \[
        \sup_{(\mu,\vec{\nu})\in \mathcal{V}_J(\ep)} \mathbb{E}_{Z'\sim \vec{\nu}}{\cL\bigl(\htheta(Z'), \theta(\mu)\bigr)} = \M_J(\ep,\theta),
    \]
    and similarly for IC.
\end{definition}
An estimator is a JC $\ep$--\emph{minimax rate-optimal estimator} if the risk is of the order $\M_J$ as a function of $n$, $\ep$ and other specified problem hyperparameters, and similarly for IC.\footnote{
After contamination, 
the original value  $\theta(P)$ of the  parameter $\theta$ may not be identifiable; 
however one can still evaluate the error in estimating its original (identifiable) value, with the understanding that the error may not converge to zero as the sample size grows.}

We will often study 
loss functions of the form
$\mathcal{L}:(a,b)\mapsto \mathcal{L}(a,b) := \Phi (\|a-b\|_q)$
for
a function $\Phi:[0,\infty)\to[0,\infty)$ increasing on $[0,\infty)$---such as the power functions $\Phi(x)=x^r$, $r\ge 1$, for all $x\ge 0$.
When the context is clear, we may omit the dependence on some quantities, 
such as the function $\theta$ determining the parameter, the sample size $n$, etc.,
i.e., we may write $\M_J(\ep)$ and refer to $\ep$-minimax estimators. 

Under the contamination--free setting, the risk is $R(\htheta,\theta(\mu))=
\mathbb E\mathcal L(\hat\theta(Z),\theta(\mu))$. Since $\M_J(0) = \M_I(0)$
equal the standard minimax risk without contamination,
we denote this quantity by $\M(0)$:
\(\mathcal{M}(0)=\inf_{\htheta} \sup_{\mu \in \cP} R(\htheta,\theta(\mu)).\)
The basic observation that JC is a stronger perturbation 
implies that
    for any $\ep\ge 0$, we have
    $\M(0)\le \M_I(\ep) \le \M_J(\ep)$.

\subsection{Related Work}\label{subsec: related work}
Contamination and robust statistics are well-studied, with a large literature focusing on the Huber $\ep$-contamination model, see e.g., \cite{huberLocation,huberRatio,huber2004robust,hampel2005robust,chen2016general, chen2018robust,zhu2020robuststat}, etc. 
Given a distribution $P$, the data are sampled from a distribution $Q= (1-\ep)P+\ep P'$, where $P'$ is unconstrained in the standard formulation.

Under general forms of contamination, \cite{Donoho1998MD} show that the \emph{population analogue} of the minimax risk is determined by the modulus of continuity (also known as the gauge function). 
In contrast, in our work, we are interested in finite sample results.

The Wasserstein distance and optimal transport
\citep{villani2003topics,villani2008optimal,ambrosio2008gradient,santambrogio2015optimal,figalli2021invitation,peyre2019computational}
has recently seen growing popularity.
Similarly, Wasserstein DRO and contamination have become widely studied, see e.g., \cite{esfahani2015dro,lee2017minimax,blanchet2017distributionally,singh2018minimax,manole2019minimax,weed2019estimation,blanchet2019robust,chen2020distributionally,staib2020learning,chen2020distributionally,shafieezadeh2020wasserstein,hutter2021minimax,duchi2021statistics}.
These works typically focus on solving numerical minimax optimization problems algorithmically, where estimators minimize the least favorable risk over a class of contaminations. In contrast, we aim to show minimaxity in the sense of statistical decision theory.

\cite{zhu2020robuststat} investigate estimation under total variation and  $W_{2,1}$ contaminations. They show that in the infinite-data---population---limit,  projecting the perturbed distribution to $\cP$ is optimal up to constant factors, and the minimax rate is the modulus of continuity. 
They bound the rate of the estimation error using the modulus of continuity, and consider estimation and prediction error for random-design linear regression. 
In the linear regression example, 
we in contrast provide the \emph{exact minimax risk} for prediction error for JC $W_{q,r}$ perturbations, a sharp bound for JC estimation error, for IC contaminations; along with least favorable perturbations.
Moreover, we additionally
provide results for location and density estimation. 
The IC and JC contaminations are related to, but different from, the oblivious and adaptive corruptions from \cite{zhu2020robuststat}, where the latter perturbs a sample from the \emph{empirical}---not population---distribution.
See Section \ref{rxy} for additional discussion and comparisons. 

In closely related work, \cite{Liu2021RobustWE} 
study IC $W_{2,1}$ contaminations, while we study more general $W_{q,r}$
contaminations. 
They study classical and sparse Gaussian location estimation, covariance matrix estimation, and linear regression. 
They fit a Wasserstein generative adversarial network (W-GAN) and use the generator to estimate the parameters of interest. They show that the W-GAN gives minimax rate-optimal squared error, with high probability. 
See Section \ref{rxy} for additional discussion end comparison. 
For the problems that fall into the intersection of our scopes, our results are consistent. 
However, more broadly speaking, we study some different problems (such as density estimation), and provide different types of results (e.g., exact minimax risks for location estimation and prediction error
in linear regression under JC $W_{q,r}$
contaminations; instead of in-probability rates under IC $W_{2,1}$ contaminations. 
See Section \ref{rxy} for additional discussion end comparisons.

Next, we go on to present our results for 
three statistical problems: 
location estimation,
linear regression, and
pointwise nonparametric density estimation under H\"older smoothness.

\section{Location Parameter Estimation}

We study the problem of estimating 
a location parameter, which includes, in particular, normal mean estimation. 
This is one of the most fundamental statistical problems, widely studied,
and present in various versions in major textbooks \citep[e.g.,][etc]{casella2024statistical,lehmann1998theory}. 
Let $Z_1,\dots,Z_n$ be i.i.d.\ 
random vectors on
$\R^p$ from a location family $\mathfrak{f}_\theta(\cdot)=\mathfrak{f}_0(\cdot-\theta)$ with unknown 
location parameter
$\theta\in\Theta:=\R^p$.
Write $E_i\sim \mathfrak{f}_0$, i.i.d. for $i\in  [n]$ 
such that the location family can be represented as $Z_i = \theta + E_i$, $i\in [n]$; 
we denote the distribution of $(Z_i)_{i\in [n]}$ by $P_\theta$.
 We do not assume that we know the noise distribution, 
 only that it has certain bounded moments.

For the special case of Gaussian mean estimation, 
a fundamental result is that
the sample mean is exactly minimax optimal
for estimating $\theta$
under any $\ell_q$, $q\in [1,\infty]$, loss
\citep{lehmann1998theory}.
Inspired by this exact minimaxity result, in this work, we
consider the contaminated case, for more general location families.
We study both JC and IC contaminations, starting with JC. 

\subsection{Joint Contaminations}
We
consider
JC contaminations 
$\cV_{J}(\ep)$
and the loss function 
$\cL\bigl(\htheta(Z'), \mu\bigr) 
= \|\htheta(Z') - \theta\|_q^r$.
We show that 
for JC,
regardless of the perturbation level $\ep$, 
the \emph{sample mean remains exact minimax optimal},
 but with an increased risk.

 Observe that the above conditions do not impose 
any centering condition such  
 that the mean 
 of $E_i$ is zero.
 However, assuming such conditions is reasonable. 
Throughout our work, we will therefore impose certain types of centering conditions chosen
 in a way that 
 is both
mild and simplifies the statements. 
One such condition that we will consider is the following. 

\begin{condition}[Centering]\label{ctr}
    For
$E_1, \ldots E_n \sim \mathfrak{f}_0$,
$c\mapsto\E\|\bar E+c\|_q^r$ is minimized at $c=0_p\in \R^p$.
\end{condition}

This condition, 
 is  achieved in many cases of interest,
 for instance when the distribution of $\bar E$ is symmetric about zero.
Moreover, in our minimax result, we 
will consider
 location families such that  $(\E_{E_1, \ldots E_n \sim \mathfrak{f}_0} \|\bar E\|_q^r)^{1/r} \le s_n$ for some finite hyperparameter $s_n>0$.
Hence,
 the class of clean distributions will be 
\begin{align}\label{pj}
\mathcal{P}
= \Big\{
  &\,Z_i = \theta + E_i, i\in[n], \textnormal{ i.i.d.:}\;
    \theta \in \R^p,\;
    (\E \|\bar E\|_q^r)^{1/r} \le s_n,
  \;\textnormal{Condition \ref{ctr} holds}
\Big\}.
\end{align}

 The minimax risk is defined as the best possible worst-case risk over all location models and bounded $W_{q,r}$ perturbations:
  $$
 \M_J(\ep,q,r)
:=\inf_{\htheta} \sup_{{P_\theta\in\mathcal{P}, \, 
W_{q,r}(P_\theta, Q) \le \ep}}
\E_{Z'\sim Q}\|\widehat \theta(Z')-\theta\|_q^{ r}.
 $$
 
\begin{figure}[t]{}
    \centering
    \includegraphics[width=0.3\columnwidth]{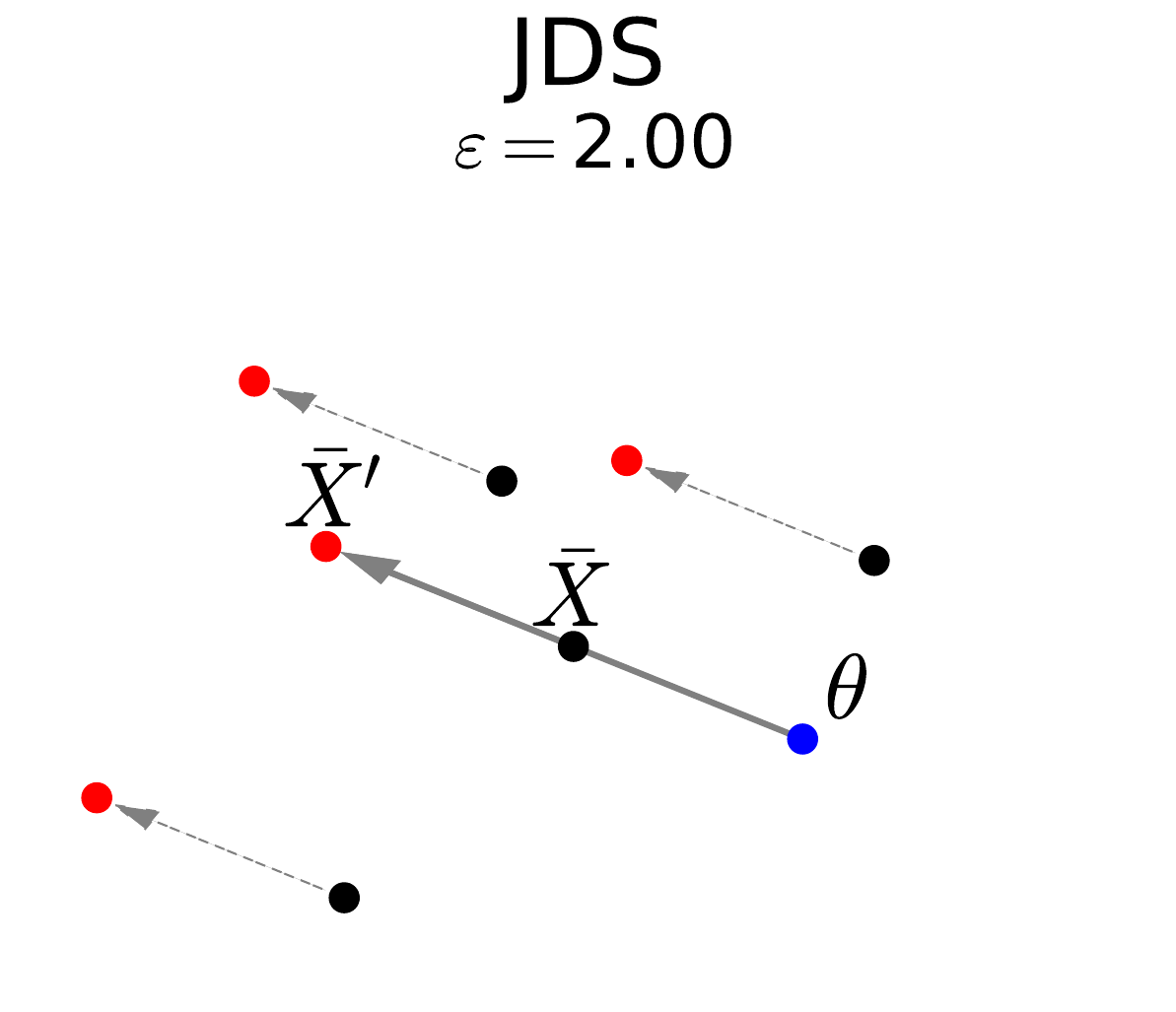}
    \includegraphics[width=0.3\columnwidth]{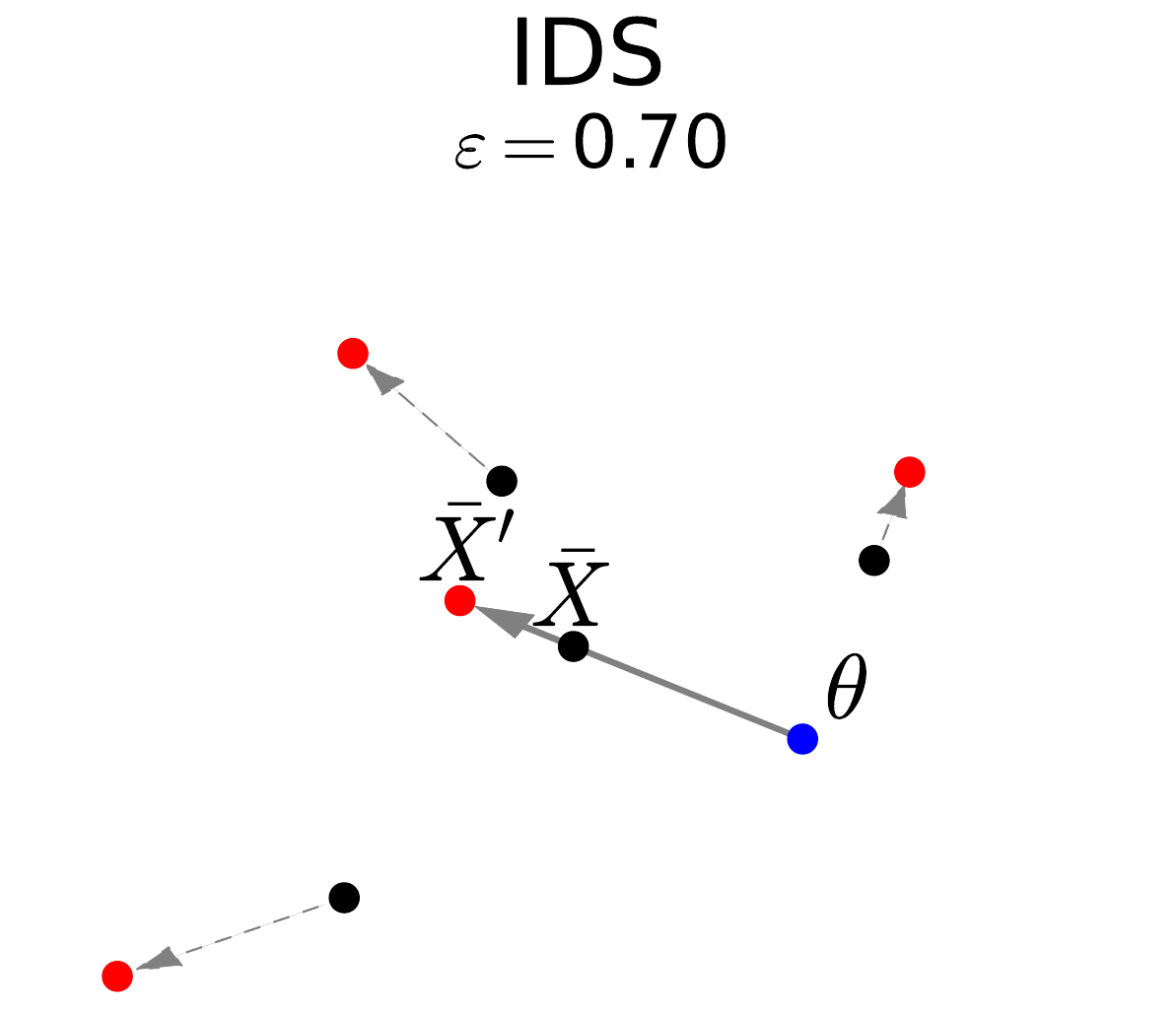}
    \includegraphics[width=0.3\columnwidth]{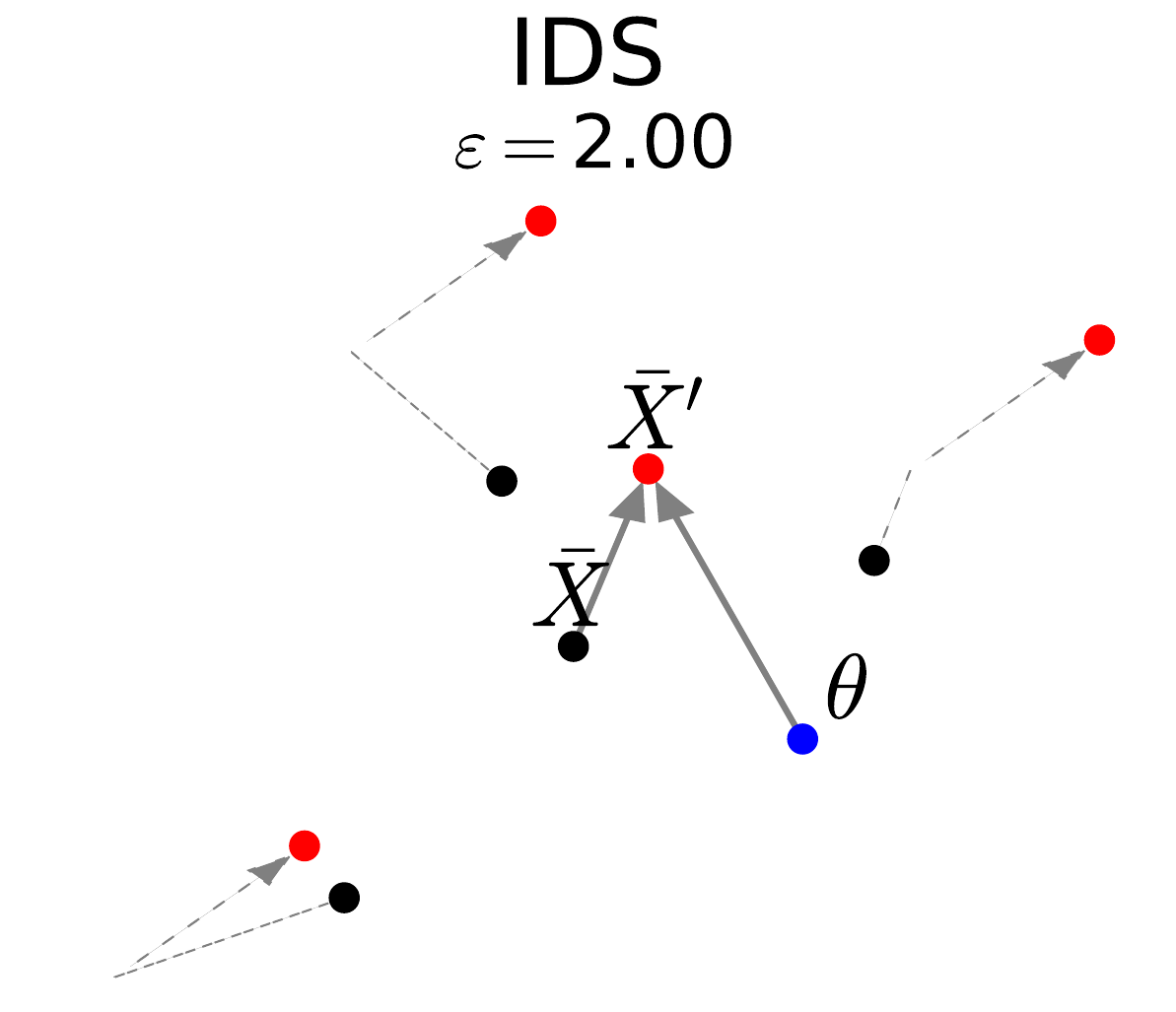}
    \caption{Representation of IC/JC adversarial perturbations for location estimation with $W_{2,2}$ perturbations and $\ell_2^2$ loss, for  $n=3$, as per \eqref{lf}. 
    We plot the true $\theta$ (blue), the unobserved clean observations (black), the observed perturbed observations (red), as well as the corresponding sample means. The second plot uses a small value of $\ep$ where $\zeta>0$ and $\psi=0$, while the third plot uses a larger value of $\ep$ where $\zeta>0$ and $\psi>0$.}
    \label{fig: gaussian perturbation}
\end{figure}

Our next result characterizes finds the value of this minimax risk. 
It also implies that, remarkably, the sample mean is minimax optimal in
for \emph{all perturbation strengths} $\ep$.
\begin{theorem}[Location parameter estimation under JC]\label{thm:JC-Lqr}
Consider the location model where
 $Z_1,\dots$, $Z_n \sim \mathfrak{f}_\theta$ are i.i.d. for some unknown 
 location parameter
$\theta\in\R^p$,
and such that for some $s_n>0$, $(\E \|\bar E\|_q^r)^{1/r} \le s_n$,
under loss $\|\hat\theta-\theta\|_q^{ r}$, where $r\ge 1$, $q\in [1,\infty]$.
 Then under joint $W_{q,r}$ contaminations of size $\ep$ (as per Definition \ref{def: ds}), the minimax risk from Definition \ref{def: minimax risk and est} equals 
\[
\M_J(\ep,q,r)
=
\Big( 
s_n + \ep\Big)^{\!r}.
\]
Moreover, the sample mean $\hat\theta=\bar Z$ is minimax optimal.
\end{theorem}

This presents an intriguing contrast with robustness in the Huber $\ep$-contamination model, where the sample mean is usually not optimal.
 Moreover, 
 by setting $\ep=0$,
 the above result also recovers the statement that the sample mean is exact minimax optimal without any perturbations, achieving risk $\M_J(0,q,r) 
 = 
 s_n$.
 Hence the above result can also be interpreted as the 
 characterizing the cost of robustness 
 by writing it as
 \[
\M_J(\ep,q,r)^{1/r}
=
\M_J(0,q,r)^{1/r}
+ \ep.
\]

 Theorem \ref{thm:JC-Lqr} 
 follows from our more general results on linear regression, see Section \ref{pfthm:JC-Lqr}; and thus we will defer the explanation of the proofs to that section.

\subsection{Independent Contaminations}
\label{subsec:ic-location-general}

Next, we consider the setting of independent contaminations.
Denote by $\mu$ the distribution of $Z_i$ and by $\nu$ the distribution of the observed and perturbed $Z_i'$, $i\in[n]$.
 The minimax risk is defined as 
\begin{equation}\label{mid}
 \M_I(\ep,q,r)
:=\inf_{\htheta} \sup_{{\mu^n\in\mathcal{P}, \, 
W_{q,r}(\mu, \nu) \le \ep}}
\E_{Z'\sim \nu^n}\|\widehat \theta(Z')-\theta\|_q^{ r}.
\end{equation}

We provide both upper and lower bounds on the minimax risk $\M_I(\ep,q, r)$. 
 
\begin{theorem}[Location family IC]
\label{thm:IC-struct}
Consider the location model from Theorem \ref{thm:JC-Lqr}.
For every $\ep\ge0$,
 the minimax risk $\M_I(\ep,q, r)$ under i.i.d. contaminations
 is bounded as 
\begin{equation}\label{eq:IC}
\max\{{\ep^r,(s_n +\ep/n^{1/2})^r }\}
\le \M_I(\ep,q, r) \le (s_n+\ep)^{r}.
\end{equation}

\end{theorem}

In fact, the upper bound 
follows from our result on joint contaminations in Theorem \ref{thm:JC-Lqr}.  
The proof of the lower bound
is similar to the ones for linear regression; see Section \ref{pfthm:IC-struct}. Again, we will discuss these proof techniques in more detail for linear regression. 

Moreover, the lower bounds
imply that
$\M_I(\ep,q,r)\ \ge 
2^{-r} (\ep+s_n)^r$.
 This 
 establishes the minimax rate for the estimation error as $\M_I(\ep,q,r)\asymp (\ep+s_n)^r$, which is the same rate as that for joint perturbations.
 Therefore, up to the rate, joint and independent contaminations coincide. 
 However, as we will show next, in certain cases, they are still clearly distinct and separated.

\subsubsection{Euclidean Wasserstein-2 perturbations}
\label{euw2}

Next, we find the \emph{exact minimax risk}
for Euclidean Wasserstein-2 IC perturbations, i.e., when 
$r = q = 2$. 
This corresponds to the most widely studied loss function in statistics, the squared $\ell_2$ error, and is thus of special significance.  
Our result will also show a clear separation and distinction between IC and GC perturbations. 
In this setting, we study the following class of distributions: 
\begin{align}\label{pj2}
\mathcal{P}'
= \Big\{
  &\,Z_i = \theta + E_i, i\in[n], \textnormal{ i.i.d.:}\;
    \theta \in \R^p,\;
    (\E \|E\|_2^2)^{1/2} \le \sigma,
  \;\textnormal{Condition \ref{ctr} holds}
\Big\}.
\end{align}

The sample mean remains optimal, but its risk exhibits an intriguing phase transition depending on the size of $\ep$. 

\begin{theorem}[IC Location Estimation, $r = q = 2$]
\label{thm: gaussian mean opt}
In the setting of Theorem \ref{thm:IC-struct}, 
with $r=q=2$,
the sample mean estimator $Z'\mapsto\bar Z'$ is \emph{exact minimax optimal},
achieving minimax risk, for $n\ge 2$,
    \begin{align*}
        \M_I(\ep, 2,2) & =\begin{cases}
            \frac{\left(\ep+\sigma\right)^2}{n} & \; \ep^2 \le \frac{\sigma^2}{(n-1)^2}, \\
            \ep^2 + \frac{\sigma^2}{n-1}                   & \; \ep^2 > \frac{\sigma^2}{(n-1)^2}.
        \end{cases}
    \end{align*}
When $n=1$, $\M_I(\ep,2,2) =
            \left(\ep+\sigma\right)^2$.
\end{theorem}

In many cases, we are interested in zero-mean noise, in which case 
 $\sigma^2 = \Tr \Sigma$ is the trace of the covariance matrix $\Sigma  = \mathrm{Cov}{Z_1}$.
The upper bound follows from a nuanced analysis
 of the risk, which in particular expands the squared error and also keeps the 
 cross terms that are not captured by our proof of the more general setting from 
 Theorem \ref{thm:IC-struct}.
Since we believe this proof has interest, we provide it later below.
 
{\bf Least favorable perturbations.}
 The proof of Theorems \ref{thm:IC-struct} and \ref{thm: gaussian mean opt} 
 above also characterize the
 least favorable perturbations.
 These are given, for $i\in[n]$, by
    \begin{align}\label{lf}
    Z_i'=
        Z_i+\begin{cases}
            \ep\sqrt{\frac{n}{\sigma^2}}(\bar Z-\theta) & \text{ for JC}, \\
            \zeta (Z_i-\theta)+\psi\delta                 & \text{ for IC}, \\
        \end{cases}
    \end{align}
    where
    \begin{align}\label{eq: zeta psi}
        \zeta =\min\left(\sqrt{\frac{\ep^2}{\sigma^2}},\frac{1}{n-1}\right), \quad
        \psi =\sqrt{\max\left(0,\ep^2-\frac{\sigma^2}{(n-1)^2}\right)},
    \end{align}
    and $\delta$ is any unit vector.

In the JC setting, the least favorable perturbation in \eqref{lf} shifts the 
sample
mean away from the true parameter, in the direction of the realized noise $\bar Z -\theta$.
This requires coordinating the shift jointly over all datapoints, and makes the contaminated data dependent. 

{\bf Phase transition.}
There is an intriguing phase transition. 
For the IC setting, we see from \eqref{eq: zeta psi} that the least favorable perturbation can take two possible forms.
For small values of $\ep \le \sigma/(n-1)$, 
the perturbation scales every observation away from the true parameter, in the direction of the realized noise $Z_i-\theta$, introducing purely noise and variance.
For larger values of $\ep$, the perturbation focuses on shifting every observation in a fixed direction, as $\psi\gg \zeta$ for $\ep\to\infty$, introducing mostly bias. 
In \cref{fig: gaussian perturbation}, we plot $\zeta$ and $\psi$ for varying values of $\ep$.

\begin{figure}[t]
    \centering
    \begin{subfigure}{0.45\columnwidth}
        \centering
        \includegraphics[width=\columnwidth]{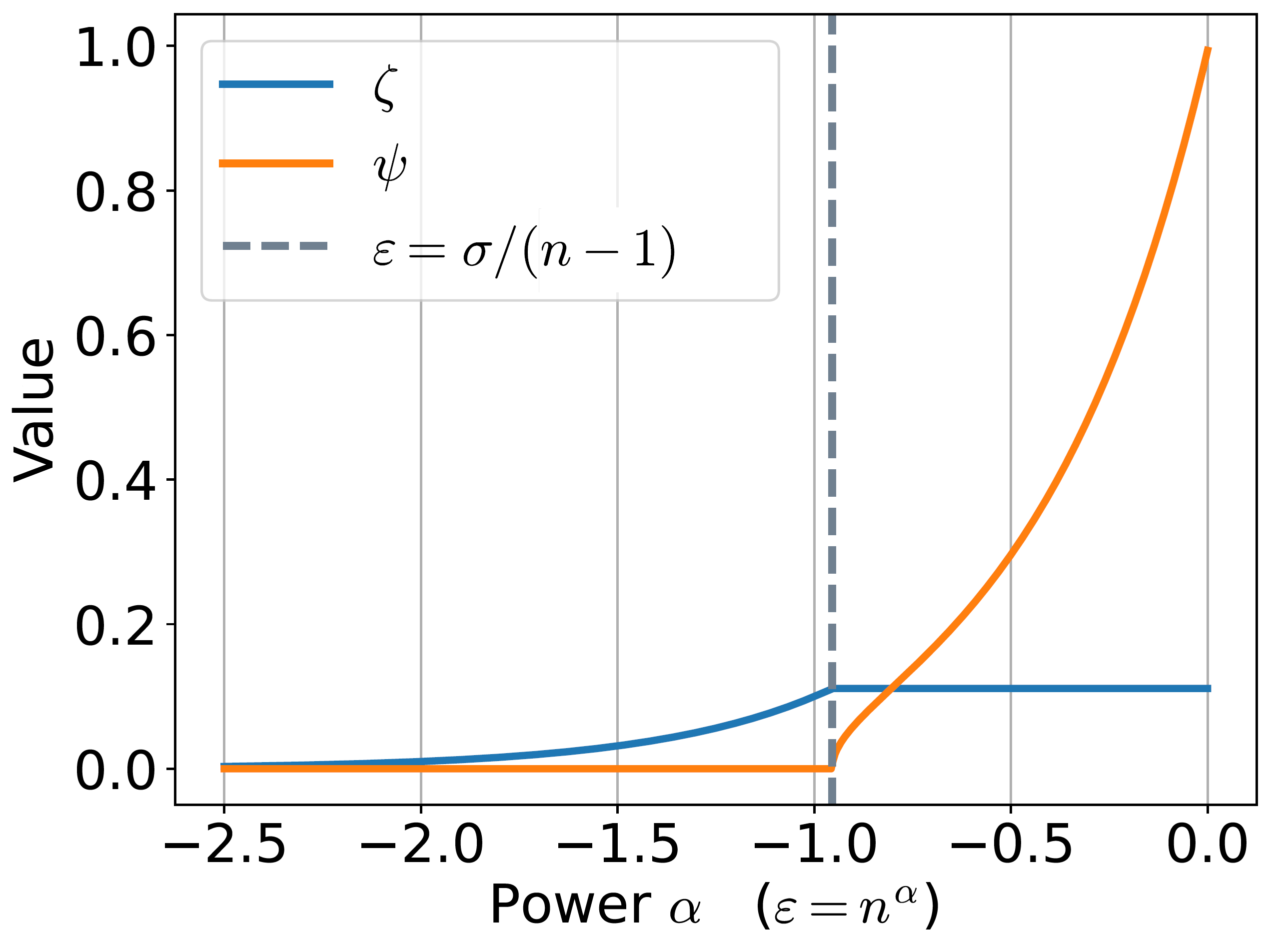}
    \end{subfigure}
    \begin{subfigure}{0.45\columnwidth}
        \centering
        \includegraphics[width=\columnwidth]{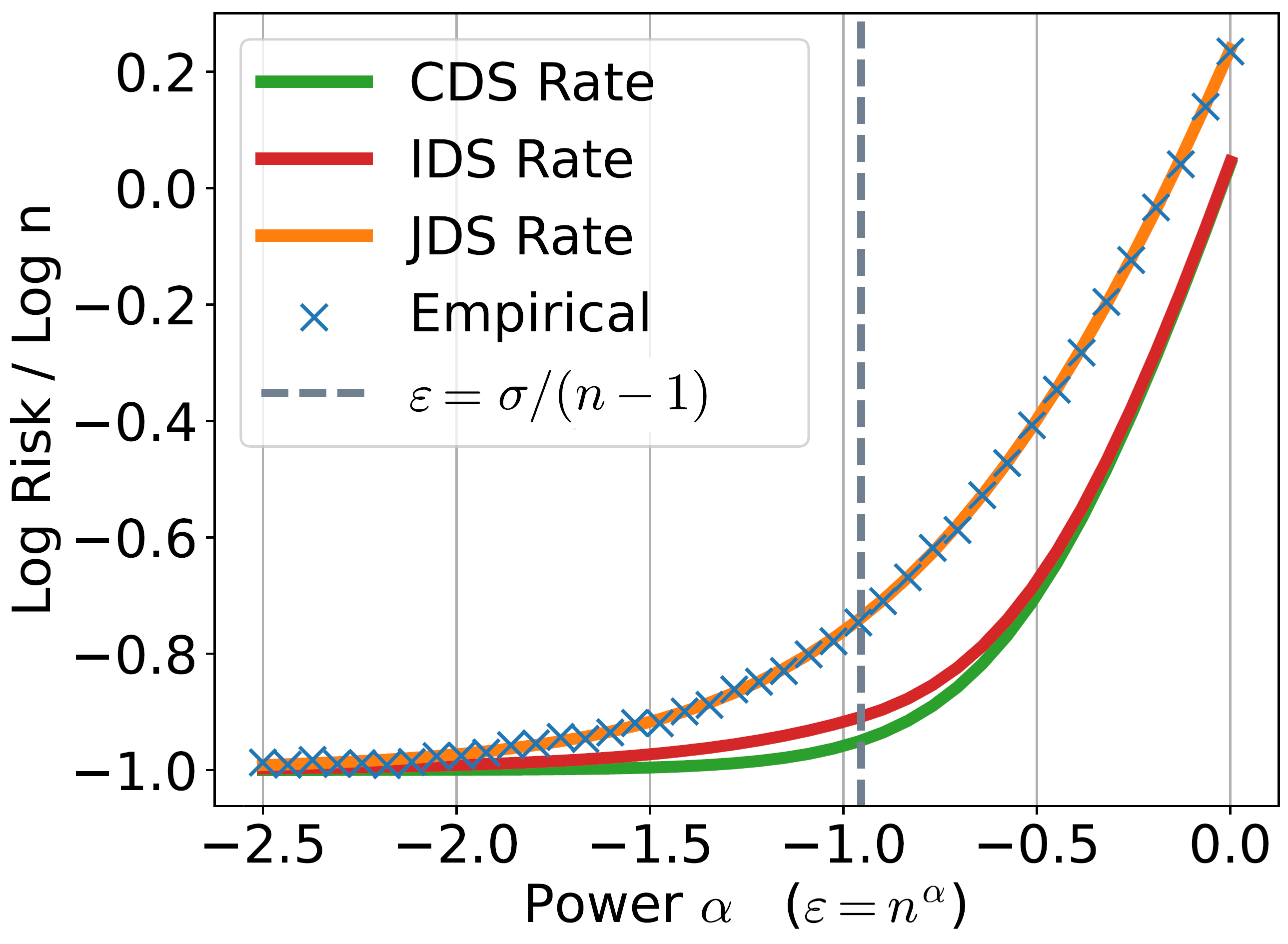}     
    \end{subfigure}
    \caption{\textbf{Left:} Values of $\zeta$ and $\psi$ for IC perturbations for varying $\ep$, as per \eqref{eq: zeta psi}. 
    We see that for small values of $\ep$, $\zeta$ increases while $\psi$ is zero, and past the transition (gray) of $\ep=\sigma/(n-1)$, $\zeta$ remains constant at $1/(n-1)$ while $\psi$ increases with $\ep$.
    \textbf{Right:}
    For the IC Gaussian mean estimation from Section \ref{euw2}, 
    we plot the log squared error divided by $\log n$,
    for $n=10$, $p=3$, and $\Sigma=6^{-1}\mathrm{diag}(1,2,3)$ (so that $\Tr(\Sigma)=1$), and the IC transition represented with the gray dashed line. For small values of $\ep$, we see a rate of $n^{-1}$, as expected.}
    \label{fig: gaussian perturbation risk}
\end{figure}
To provide some intuition of the transition in the IC setting, consider estimating $\theta$ from observations $Z_i\sim \cN(\theta,1)$ for $i\in[n]$ where $\theta=0$.
Then consider the two perturbation classes $Z_{i,1}'=(1+\ep)Z_i$ and $Z_{i,2}'=Z_i+\ep \delta$ for a unit vector $\delta$, for $i\in[n]$. 
These represent the scaling and shifting done by $\psi\delta$ and $\zeta$ in the IC perturbation from \eqref{eq: zeta psi}.
One can verify that $\|Z_{i,1}'-Z_i\|_{\ell_2}=
\|Z_{i,2}'-Z_i\|_{\ell_2}=\ep$; 
hence the $W_{2,2}$ distance is bounded by $\ep$.
Evaluating the risk, we have
$\|\bar Z'_1 - \theta \|_{\ell_2}^2\sim (1+\ep)^2/n$ and  $\|\bar Z'_2 - \theta\|_{\ell_2}^2\sim \ep^2 +  1/n$. For $\ep\lesssim 1/n$, the first---scaling/noise---perturbation class has greater risk, while for larger $\ep$, the second---constant shift/bias---perturbation class has greater risk. 
The transition occurs at $\ep\sim 1/n$, consistent with our theorem.

In the right plot of \cref{fig: gaussian perturbation risk}, we plot the IC and JC rates for varying levels of $\ep$ parameterized as $\ep=n^\alpha$. 
We plot the base-$n$ logarithm of the risk, which represents the exponent of the risk, e.g., $\ep=0$ leads to a risk of $n^{-1}$ and a plotted value of $-1$. 
We also plot the empirical risk of the sample mean estimator under our least favorable adversarial perturbations, and find that it matches the theoretical value. 
See \cref{subsec: gaussian simulations} for further experimental details.

\begin{proof}[Proof of Theorem \ref{thm: gaussian mean opt}.]
{\bf Upper bound.}
For the upper bound, 
choose couplings $(Z_i,Z_i')$ such that $\mathbb{E}{\|Z_i'-Z_i\|_2^2} = W_{2,2}(Z_i',Z_i)^2$ $\le \ep^2$, $i\in[n]$, 
and consider the sample mean $\htheta(Z') = \bar Z'$. 
    Decomposing the risk,
    \begin{align*}
        \M_I\le \mathbb{E}{\|\bar Z' - \theta\|_2^2}
         & = \mathbb{E}{\left\|\bar Z' - \bar Z\right\|_2^2} +  \frac{\sigma^2}{n}+2\mathbb{E}{(\bar Z'-\bar Z)^\top(\bar Z-\theta)}.
    \end{align*}
    For the cross term, since $(Z_i,Z_i')$ and $(Z_j,Z_j')$ are mutually independent and identically distributed for $i\neq j$,
    \begin{align*}
        \mathbb{E}{(\bar Z'-\bar Z)^\top(\bar Z-\theta)}
         & = \frac{1}{n^2}\sum_{i=1}^n \mathbb{E}{(Z_i'-Z_i)^\top(Z_i-\theta)}+\frac{1}{n^2}\sum_{i\neq j} \mathbb{E}{(Z_i'-Z_i)^\top(Z_j-\theta)} \\
         & = \frac{1}{n} \mathbb{E}{(Z_1'-Z_1)^\top(Z_1-\theta)}.
    \end{align*}

    Therefore
    \[\M_I - \frac{\sigma^2}{n}\le \mathbb{E}{\left\|\bar Z'-\bar Z\right \|_2^2}+\frac{2}{n} \mathbb{E}{(Z_1'-Z_1)^\top(Z_1-\theta)}.\]

    Define, for all $i\in[n]$, $D_i=Z_i'-Z_i$ and $E_i=Z_i-\theta$. It follows that $E_i\simiid \mathfrak{f}_0$, $D_i$ are i.i.d., and $\mathbb{E}{\|D_i\|_2^2}\le \ep^2$.
    Denoting $\mathcal{I}\coloneqq  \M_I  - \frac{\sigma^2}{n}$,
    \[n^2\mathcal{I}\le \mathbb{E}{\left\|\sum_{i=1}^n D_i \right \|_2^2}+ 2n \mathbb{E}{D_1^\top E_1}.\]
    Consider the following decomposition, with terms defined below:
    \begin{align*}
        D_i = \underbrace{\mathbb{E}{D_i}}_{\psi\delta} + \underbrace{(\mathbb{E}({D_i\mid E_i}) - \mathbb{E}{D_i})}_{g(E_i)} + \underbrace{(D_i - \mathbb{E}({D_i\mid E_i}))}_{b(E_i,D_i)}.
    \end{align*}
    Thus, we define
    $\psi\in \R$ and
    the deterministic unit vector $\delta$
    by
    $\psi\delta = \mathbb{E}{D_i}$, setting $\delta$ arbitrarily if $\psi=0$.
    Also, $g(E_i) =\mathbb{E}({D_i\mid E_i}) - \mathbb{E}{D_i}$, and $b(E_i,D_i)=D_i - \mathbb{E}({D_i\mid E_i})$, representing the bias of $D_i$. Next, we can bound
    \begin{align*}
        n^2\mathcal{I} & \le \mathbb{E}{\left\|\sum_{i=1}^n (\psi\delta + g(E_i) + b(E_i,D_i))\right \|_2^2}+ 2n \mathbb{E}{(\psi\delta + g(E_1) + b(E_1,D_1))^\top E_1} \\
                    & = \mathbb{E}{\left\|\sum_{i=1}^n (\psi\delta + g(E_i) + b(E_i,D_i))\right \|_2^2}+ 2n \mathbb{E}{(\psi\delta + g(E_1))^\top E_1},
    \end{align*}
    where the last equality uses the fact that $\mathbb{E}{b(E_i,D_i)^\top E_1}=0$ from the law of total expectation.

    We aim to choose $g$, $b$, $\psi$, and $\delta$ to maximize the following optimization problem,
    \begin{align*}
        \sup_{g,b,\psi,\delta} & \quad  \mathbb{E}{\left\|\sum_{i=1}^n (\psi\delta + g(E_i) + b(E_i,D_i))\right \|_2^2}+ 2n \mathbb{E}{(\psi\delta + g(E_1))^\top E_1} \\
        \text{s.t.}            & \quad  \mathbb{E}{\|\psi\delta + g(E_i) + b(E_i,D_i)\|_2^2}\le \ep^2 \text{ for all } i\in[n].
    \end{align*}
    Intuitively, since the expected squared norm of $\psi \delta + g(E_i) + b(E_i,D_i)$ is bounded by $\ep^2$,
    we should set $b(E_i,D_i)=0$.
    Formally, this is shown via Lemma \ref{lemma: optim set to zero} by choosing $T_i=\psi\delta + g(E_i)$ and $V_i=b(E_i,D_i)$ and noticing that
    \begin{align*}
        \mathbb{E}[{(\psi \delta + g(E_i))^\top b(E_i,D_i)}] & =\mathbb{E}[{\mathbb{E}({D_i \mid E_i})^\top (D_i - \mathbb{E}({D_i\mid E_i}))}]                        \\
                                                    & =\mathbb{E}[{\mathbb{E}({D_i \mid E_i})^\top D_i] 
                                                    - \mathbb{E}[{\mathbb{E}({D_i \mid E_i})^\top D_i\mid E_i}}]=0
    \end{align*}
    from the law of total expectation.
    Therefore we have
    \begin{align*}
        n^2\mathcal{I} \le \sup_{g,\psi,\delta} & \quad  \mathbb{E}{\left\|\sum_{i=1}^n (\psi\delta + g(E_i))\right \|_2^2}+ 2n \mathbb{E}{(\psi\delta + g(E_1))^\top E_1} \\
        \text{s.t.}                          & \quad  \mathbb{E}{\|\psi\delta + g(E_i)\|_2^2}\le \ep^2 \text{ for all } i\in[n].
    \end{align*}
    Directly expanding the objective and constraints, using the facts $\mathbb{E}{g(E_i)}=\mathbb{E}{E_i}=0$, $\|\delta\|_2=1$, and $\mathbb{E}{g(E_1)^\top E_1} \le \sqrt{\mathbb{E}{\|g(E_1)\|_2^2 } \mathbb{E}{\|E_1\|_2^2}}$ from the Cauchy-Schwarz inequality,
    \begin{align*}
         & \mathbb{E}{\left\|\sum_{i=1}^n (\psi\delta + g(E_i))\right \|_2^2}+ 2n \mathbb{E}{(\psi\delta + g(E_1))^\top E_1}=(n\psi)^2  + \sum_{i=1}^n \mathbb{E}{\|g(E_i)\|_2^2}+2n \mathbb{E}{g(E_1)^\top E_1} \\
         & \qquad\qquad\le (n\psi)^2  + \sum_{i=1}^n \mathbb{E}{\|g(E_i)\|_2^2}+2n \sqrt{\mathbb{E}{\|g(E_1)\|_2^2} \sigma^2}.
    \end{align*}
    Further,
    $\mathbb{E}{\|\psi\delta + g(E_i)\|_2^2} = \psi^2 + \mathbb{E}{\|g(E_i)\|_2^2}$.
    Thus,
    
    \begin{align*}
        n^2\mathcal{I} \le \sup_{\psi} & \quad n^2\psi^2  + n(\ep^2-\psi^2)+2n \sqrt{(\ep^2-\psi^2)\sigma^2} \quad
        \text{s.t.}               \quad \psi\le \ep.
    \end{align*}
    Solving this using calculus,
    for $n\ge 2$,
    the optimal value of $\psi$ is
    $
        \psi =  \sqrt{\max\left(0,\ep^2-\frac{\Tr \Sigma}{(n-1)^2}\right)}$
    and our optimization objective becomes 
    \begin{align*}
         \M_I  - \frac{\sigma^2}{n} & = \mathcal{I} \le 
        \begin{cases}
        \frac{\ep^2}{n} + \frac{2\ep \sigma}{n} 
             & \ep^2 \le \frac{\sigma^2}{(n-1)^2}, \\
              \ep^2 + \frac{\sigma^2}{n(n-1)} & \ep^2 > \frac{\sigma^2}{(n-1)^2},
        \end{cases}
    \end{align*}
    This leads to the desired upper bound.
    For $n=1$, the maximum is always at $\psi=0$, which leads to the desired claim.
    To obtain equality in the Cauchy-Schwarz inequality used above, we need $g(E_i)=\zeta E_i$, leading to
    \(\zeta  = \sqrt{\frac{\ep^2-\psi^2}{\sigma^2}} = \min\left(\sqrt{\frac{\ep^2}{\sigma^2}},\frac{1}{n-1}\right),\)
    which matches \eqref{eq: zeta psi}.

{\bf Lower bound.}
Fix an arbitrary estimator $\widehat\theta:\mathcal{Z}^n\to\mathbb{R}^p$. Let $Z_i=\theta+E_i$
with i.i.d. $E_i\sim \N(0,\sigma^2)$.
Throughout, we write $\bar E:=\frac1n\sum_{i=1}^n E_i$.

\medskip
\noindent\textit{Feasible pair of contaminations.}
Fix any unit vector $\delta\in\mathbb{R}^p$ and any numbers $\zeta,\psi\ge 0$ such that $\zeta^2 \sigma^2+\psi^2\le \ep^2$. Define two i.i.d.\ contaminated laws $\nu_\pm$ by declaring that, under $\nu_\pm$,
\begin{equation*}
Z_i' \;=\; \theta \pm \psi \delta + (1+\zeta) E_i,
\qquad i=1,\dots,n.
\end{equation*}
For each $i$, couple $Z_i=\theta+E_i$ with $Z_i'$ as above. Then $\mathbb{E}\|Z_i'-Z_i\|_2^2=\mathbb{E}\|\zeta E_i\pm \psi\delta\|_2^2=\zeta^2\mathbb{E}\|E_i\|_2^2+\psi^2=\zeta^2 \sigma^2+\psi^2\le \ep^2$. Hence $W_{2,2}(\nu_\pm,\mathfrak{f}_\theta)\le \ep$, so $\nu_\pm$ are admissible IC contaminations.

For the laws $\nu  =\nu_{\pm}$ and any $\theta$,
\begin{equation*}
\mathbb{E}_{\nu}\|\widehat\theta(Z'_1,\dots,Z'_n)-\theta\|_2^2
=
\mathbb{E}_{\nu}\Big[\mathbb{E}_{\nu}\big(\|\widehat\theta(Z')-\theta\|_2^2\mid \bar Z'\big)\Big]
\;\ge\;
\mathbb{E}_{\nu}\big\|\mathbb{E}_{\nu}[\widehat\theta(Z')\mid \bar Z']-\theta\big\|_2^2,
\end{equation*}
where the inequality is Jensen’s applied conditionally on $\bar Z'$.
 Moreover, since 
 $Z'$ follows a normal distribution, 
 the distribution of $Z'$ given $\bar Z$ 
 does not depend on $\theta$, and so the map $\bar Z' \mapsto g(\bar Z'):=\mathbb{E}_{\nu}[\widehat\theta(Z')\mid \bar Z']$ defines a valid estimator. 
Thus, for the purpose of a lower bound, we may replace $\widehat\theta$ by an estimator $g$ depending only on the sample mean $\bar Z'$.

Under $\nu_\pm$ we have $\bar Z'_\pm=\theta \pm \psi\delta+(1+\zeta)\bar E$. For any $g$ and any $\theta$,
\begin{align*}
\max\!\big\{\mathbb{E}_{\nu_+}\|g(\bar Z'_+)-\theta\|_2^2, \mathbb{E}_{\nu_-}\|g(\bar Z'_-)-\theta\|_2^2\big\}
\;\ge\;
\frac12\Big(\mathbb{E}_{\nu_+}\|g(\bar Z'_+)-\theta\|_2^2+\mathbb{E}_{\nu_-}\|g(\bar Z'_-)-\theta\|_2^2\Big).
\end{align*}
Introduce a Rademacher sign $S\in\{\pm 1\}$ independent of $\bar E$, and set $U:=S \psi\delta+(1+\zeta)\bar E$. Then the average on the right equals $\mathbb{E}\|g(\theta+U)-\theta\|_2^2$. Therefore,
\begin{equation*}
\sup_{\nu\in\{\nu_+,\nu_-\}} \sup_{\theta\in\mathbb{R}^p} \mathbb{E}_{\nu}\|\widehat\theta-\theta\|_2^2
\;\ge\;
\sup_{\theta\in\mathbb{R}^p} \inf_{g:\mathbb{R}^p\to\mathbb{R}^p} \mathbb{E}\|g(\theta+U)-\theta\|_2^2.
\end{equation*}

By Lemma~\ref{lem:equiv-minimax} (Hunt–Stein for the translation group), for any noise $U$ with $\mathbb{E}\|U\|_2^2<\infty$,
\begin{equation*}
\inf_{g} \sup_{\theta} \mathbb{E}\|g(\theta+U)-\theta\|_2^2
\;=\;
\inf_{c\in\mathbb{R}^p} \mathbb{E}\|U+c\|_2^2.
\end{equation*}
Due to our assumptions. 
the minimum is attained at $c=0$ and equals $\mathbb{E}\|U\|_2^2$. Using independence of $S$ and $\bar E$,
\begin{equation*}
\mathbb{E}\|U\|_2^2=
\mathbb{E}\|(1+\zeta)\bar E\|_2^2+\Tr\mathrm{Cov}(S \psi\delta)
=
\frac{(1+\zeta)^2}{n} \sigma^2+\;\psi^2 .
\end{equation*}
Therefore, for every estimator $\widehat\theta$,
\begin{equation*}
\sup_{\nu\in\{\nu_+,\nu_-\}} \sup_{\theta} \mathbb{E}_{\nu}\|\widehat\theta-\theta\|_2^2
\;\ge\;
\frac{(1+\zeta)^2}{n} \sigma^2+\;\psi^2.
\end{equation*}
Thus, we obtained the lower bound
\begin{equation*}
\inf_{\widehat\theta} \sup_{\substack{\mathfrak{f}_0:\ \E\|E_1\|_2^2 = \sigma^2\\ \nu:\ W_{2,2}(\nu,\mathfrak{f}_\theta)\le \ep}} \mathbb{E}_{\nu}\|\widehat\theta-\theta\|_2^2
\;\ge\;
\sup_{\substack{\zeta,\psi\ge 0, \zeta^2 \sigma^2+\psi^2\le \ep^2}}
\left(\frac{(1+\zeta)^2}{n} \sigma^2+\;\psi^2\right).
\end{equation*}

Impose the budget $\psi^2=\ep^2-\zeta^2 \sigma^2$
with $\zeta\in[0,\ep/\sigma]$. The objective reduces to
\begin{equation*}
F_T(\zeta)\;=\;\frac{(1+\zeta)^2}{n} \sigma^2+\;\ep^2-\zeta^2 \sigma^2
\;=\;
\ep^2+\frac{\sigma^2}{n}+ \frac{2\zeta \sigma^2}{n} - \zeta^2 \sigma^2\!\left(1-\frac{1}{n}\right).
\end{equation*}
For $n\ge 2$,
this is a concave quadratic in $\zeta$, maximized at $\zeta^\star=\frac{1}{n-1}$ provided $\frac{1}{n-1}\le \ep/\sigma$; otherwise the maximum is at the boundary $\zeta=\ep/\sigma$. Thus
\begin{equation*}
\max_{\zeta,\psi} \left(\frac{(1+\zeta)^2}{n} \sigma^2+\;\psi^2\right)
=
\begin{cases}
\dfrac1n\bigl(\sigma+\ep\bigr)^2, & \ep^2\le \dfrac{\sigma^2}{(n-1)^2},\\[6pt]
\ep^2+\dfrac{\sigma^2}{ n-1 }, & \ep^2> \dfrac{\sigma^2}{(n-1)^2}.
\end{cases}
\end{equation*}
When $n=1$, the first branch always applies. 
This is exactly the claimed lower bound. 

\end{proof}

\subsubsection{Gaussian Case, Euclidean Wasserstein-1 Contaminations}

Next, we also 
particularize our 
IC bound from Theorem \ref{thm:IC-struct}
for the case
 of Euclidean Wasserstein-1 contaminations 
where 
$q=2$ and $r=1$, 
and for Gaussian mean estimation.
This enables us to compare our results with those from the prior work by \cite{Liu2021RobustWE}. 

In this setting, we study 
normal noise, 
and similarly to  $\mathcal{P}'$ from \eqref{pj2},
we consider
the normal location family: 
\begin{align}\label{pj2g}
\mathcal{P}'_{\mathcal{N}}
= \Big\{
\,Z_i = \theta + E_i, i\in[n], \textnormal{ i.i.d.:}\;
    \theta \in \R^p,\;
E_i \sim \N(0,\Sigma), \Sigma \succ 0_{p\times p}, \E\|E_i\|_2 \le \sigma\Big\}.
\end{align}
The minimax risk
 $\M_{I,\N}(\ep,q,r)$  is defined as $ \M_I(\ep,q,r)$ in \eqref{mid}, but replacing $\mathcal{P}'$ with $\mathcal{P}'_{\mathcal{N}}$.

Our result shows that the sample mean is minimax optimal within a factor of (at least) 1/2. 

\begin{corollary}[IC location under $W_{2,1}$ for spherical Gaussian noise]\label{cor:IC-q2-r1-Gaussian-sph}
In the setting of Theorem \ref{thm:IC-struct}, 
for normal noise
with $W_{2,1}(\nu,\mathfrak{f}_\theta)\le\ep$ and loss $\mathcal L(\hat\theta,\theta)=\|\hat\theta-\theta\|_2$, one has, for all $\ep\ge0$,
\begin{equation}\label{eq:IC-q2-r1-Gauss-spherical-sim}
\frac{1}{2}
\Big\{ \ep + \frac{\sigma}{\sqrt{n}}\Big\}
\ \le\
\M_{I,\N}(\ep,1,2)
\ \le\
\frac{\sigma}{\sqrt{n}}+\ep.
\end{equation}
\end{corollary}
See section \ref{pfthm: gaussian mean opt} for the proof.

Our result characterizes the expected minimax risk within a factor of 1/2, and moreover shows that it is attained by the sample mean.
 Standard normal noise corresponds to $\Sigma = I_p$, 
 in which case
 $c_p=\E\|E_1\|_2=\sqrt{2} \Gamma\!\big(\tfrac{p+1}{2}\big)/\Gamma\!\big(\tfrac{p}{2}\big)$; 
 so that this case can be studied by taking 
 $\sigma = c_p$ above.
 Since  $\sqrt{p-1}\le c_p \le \sqrt{p}$, 
 we conclude that the minimax risk for estimation under standard Gaussian noise and $W_{2,1}$ $\ep$-perturbations satisfies
 \begin{equation*}
\frac{1}{2}
\Big\{ \ep + \sqrt{\frac{p-1}{n}}\Big\}
\ \le\
\M_{I,\N}(\ep,1,2)
\ \le\
\sqrt{\frac{p}{n}}+\ep.
\end{equation*}
 We now compare with Theorem 4 of \cite{Liu2021RobustWE}, which shows that 
 with exponentially high probability,
 a certain WGAN estimator 
 satisfies $\|\hat{\theta}-\theta\|_2 \le C(\sqrt{p/n}+\ep)$ for an unspecified constant $C$.
 Moreover, their Theorem 5 shows that
 for any estimator $\hat \theta$,
 the loss $\|\hat{\theta}-\theta\|_2$ is lower bounded 
with a constant probability 
 by 
 $C(\sqrt{p/n}+\ep)$
 for an unspecified constant $C$; thus providing an in-probability minimax rate.
 In contrast,
 our result does not depend on unspecified constants, 
 while showing that even the sample mean achieves this particular rate (and thus, for this particular setting, the WGAN does not have a clear theoretical advantage).

\section{Linear Regression}
\label{lr}

We now consider the important problem of linear regression. 
We observe an $n\times p$ data matrix $X$,
and an $n\times 1$ outcome vector $Y$.
We work in the linear model, where  $Y\sim P_\theta$
is distributed as 
\begin{equation}\label{lre}
Y=X\theta+E,\qquad X\in\R^{n\times p}  \text{ with} \rank(X)=p,
\end{equation}

where $n \ge p$ and $X$ is of full rank.
Here $E$ is an $n\times 1$ noise vector; 
we will not assume that we know the distribution of $E$,
instead, we will only impose certain moment conditions. 

\subsection{Outcome Contaminations, Fixed Design}
\label{oc-fd}

 We will first consider $X$ as fixed, but we will later also generalize our results to a random design setting, see Section \ref{subsec:XY-JC}. 
Since our contamination models were originally defined for an i.i.d.~sample, we need to adapt them slightly.
As we are not assuming that the coordinates of $Y$ are independent, 
we choose to view $Y$ as a single observation. 
Thus, we apply the contamination definitions from \Cref{def: ds} for a sample size of one, with $Y$ as the observation.
Viewed this way, JC and IC are equivalent. 
For consistency in notation, we denote the minimax risk by $\M_J$. 

To be specific, 
for $q\in[1,\infty]$ and $r\in[1,\infty)$,
and a radius $\rho\ge0$,
the unperturbed data is $(X,Y)$, 
but we observe $(X,Y')$ with $W_{q,r}(Y,Y') \le \rho$.
Here $W_{q,r}$ is computed on $\mathbb R^n$ 
with respect to the 
$\ell_q$ norm.
Letting $Q$ denote the distribution of $Y'$,
 the 
output--space $(q,r)$–Wasserstein ball 
for $Y$
around $Y\sim P_\theta$ is denoted by\footnote{For general $r$ and $q$, it is not completely straightforward to compare this observation model with our previous JC or IC models. 
However, for $r=q$, independent $W_{q,q}$ perturbations of size $\ep$ to the coordinates $Y_i$ imply a joint
$W_{q,q}$ perturbation of size $\rho = n^{1/q}\ep$. This shows that in this special case, the JC perturbations of size $n^{1/q}\ep$  considered in this section are stronger than independent perturbations of size $\ep$.}
\(
\Ball_{\rm J}^{ q,r}(\theta,\rho):=\Big\{Q: W_{q,r}(Q,P_\theta)\le \rho\Big\}.
\)
This
posits that there is a coupling of $Y,Y'$ such that $(\E\|Y-Y'\|_q^r)^{1/r} \le \rho$.

Write  $P_X:=XX^\dagger$ and
for two 
control hyperparameters $s_{\rm pred}>0$, $s_{\rm est}>0$,
 we will consider noise distributions such that 
the following moments of the noise are bounded 
\begin{equation}\label{spe}
\Big(\E\|P_XE\|_q^r\Big)^{1/r} \le s_{\rm pred},\qquad
\Big(\E\|X^\dagger E\|_q^r\Big)^{1/r} \le s_{\rm est}.
    \end{equation}
    
 As indicated below, our results for the estimation error 
 consider the regime where
 $s_{\rm pred}>0$,
 which holds unless we are in the trivial setting that the noise $E$ is completely orthogonal to the column space of $X$.

Again, it will be helpful to ensure a 
certain form of centering for the noise. 
As for the case of location estimation, this will be imposed by assuming that a certain expectation is minimized at zero. 

\begin{condition}[Centering for linear regression]\label{ctrlr}
For any $X$,
the map $c\mapsto\E\|P_XE+c\|_q^r$ is minimized at $c=0_n\in \R^n$.
\end{condition}

As for the case of location models, if we know the distribution of the noise, then this condition does not impose a substantive assumption, and it can be achieved by centering the noise. 
Moreover 
this condition 
holds in many typical cases of interest, 
for instance when the distribution of $P_XE$ is symmetric about zero.

To summarize,
 the class of clean distributions 
$P_\theta$ for $Y$
 will be characterized by
\begin{align}\label{plr}
\mathcal{P}
= \Big\{
  &\,Y = X\theta + E:\;
    \theta \in \R^p,\;
    \big(\E\|P_X E\|_q^r\big)^{1/r} \le s_{\rm pred},\;
    \big(\E\|X^\dagger E\|_q^r\big)^{1/r} \le s_{\rm est}, \\
  &\;\textnormal{Condition \ref{ctrlr} holds}
\Big\}.\nonumber
\end{align}

Furthermore, we consider two losses:
estimation error
\(\|\widehat\theta(Y)-\theta\|_q^{ r},
\)
and prediction error
\(
\|\widehat m(Y)-X\theta\|_q^{ r}\).
 The corresponding minimax risks are defined as 
  $$
 \M_J^{\rm est}(\rho,q,r)
:=\inf_{\htheta} \sup_{{P_\theta\in\mathcal{P}, \, Q\in\Ball_{\rm J}^{ q,r}(\theta,\rho)}}
\E_{Y'\sim Q}\|\widehat \theta(Y')-\theta\|_q^{ r}
 $$
and
 $$
 \M_J^{\rm pred}(\rho,q,r)
:=\inf_{\widehat m} \sup_{{P_\theta\in\mathcal{P},  Q\in\Ball_{\rm J}^{ q,r}(\theta,\rho)}}
\E_{Y'\sim Q}\|\widehat m(Y')-X\theta\|_q^{ r}.
 $$

We find 
bounds on the minimax estimation and prediction error which are 
tight up to 
factors that depend on $X$ and the distribution of the noise. 
They show that that the least squares estimator 
and the associated predictor 
are  nearly minimax optimal. 
Moreover, 
for $q=2$,
the least squares predictor
is \emph{exactly minimax optimal} for prediction error for all $\ep \ge 0$.
For a matrix $M \in \R^{k\times k'}$ and $q,q'\ge 1$, we denote by $\|M\|_{q\to q'}$
the induced operator norm.
Define $\widehat m_{\rm LS}(y):=P_Xy$ and $\widehat\theta_{\rm LS}(y):=X^\dagger y$, for all $y$.

\begin{theorem}[Linear regression under JC]
\label{thm:lmr-master}
Fix $q\in[1,\infty]$, $r\in[1,\infty)$ and $\rho\ge0$. 
Consider the linear regression model from \eqref{lre}.
Then, the minimax prediction risk satisfies
\begin{equation}\label{eq:lmr-master-pred}
{\big(s_{\rm pred}+\rho\big)^{r}
 \le\
\M_J^{\rm pred}(\rho,q,r)
 \le\
\big(s_{\rm pred}+\|P_X\|_{q\to q} \rho\big)^{r},
\qquad}
\end{equation}
and $\widehat m_{\rm LS}$ attains the upper bound.
Further, 
if $s_{\rm pred}>0$,
the minimax estimation risk satisfies
\begin{equation}\label{eq:lmr-master-est}
{\qquad
\bigg(\max\Big\{\,s_{\rm est}+\frac{s_{\rm est}}{s_{\rm pred}} \rho\,,\ \|X^\dagger\|_{q\to q} \rho\Big\}\bigg)^{r}
\ \le\
\M_J^{\rm est}(\rho,q,r)
\ \le\
\big(s_{\rm est}+\|X^\dagger\|_{q\to q} \rho\big)^{r},
\qquad}
\end{equation}
and $\widehat\theta_{\rm LS}$ attains the upper bound.

\end{theorem}

As a consequence
for $q=2$, 
and for noise distributions such that $\big(\E\|P_XE\|_2^r\big)^{1/r} \le s_{\rm pred}$,
we have
\begin{align}
&
\M_J^{\rm pred}(\rho,2,r)
=
{\
\big(s_{\rm pred}
+\rho\big)^r\ }\quad\text{(attained by }\widehat m_{\rm LS}\text{)}
\label{eq:lmr-master-pred-22}.
\end{align}
For $q=r=2$, 
denoting $B_X:=X(X^\top X)^{-2}X^\top$,
$\Sigma = \Cov{E}$,
and the smallest singular value by $s_{\min}$, 
the minimax estimation error 
for noise distributions such that 
$ \sqrt{\Tr(B_X\Sigma)}\le s_{\rm est}$
and
$ \sqrt{\Tr(P_X\Sigma)} \le s_{\rm pred}$
satisfies
\begin{align}
\Big(\max\big\{s_{\rm est}+
\frac{s_{\rm est}}{s_{\rm pred}}
\rho,\ \tfrac{\rho}{s_{\min}(X)}\big\}\Big)^2
\ \le\ 
\M_J^{\rm est}(\rho,2,2)
\ \le\
{\ \big(s_{\rm est}+\tfrac{\rho}{s_{\min}(X)}\big)^2\ }\!.
\label{eq:lmr-master-est-22}\nonumber
\end{align}

Similarly to our discussion above,
 the exact result \eqref{eq:lmr-master-pred-22} for 
 $q = 2$ 
 on the prediction risk can be interpreted as capturing the cost of robustness and can be written in the  form
 $$
\M_J^{\rm pred}(\rho,2,r)^{1/r}
=
\M_J^{\rm pred}(0,2,r)^{1/r}
+\rho,
$$
which presents the effect of the perturbation in a linearized way. 
In the more general case from equation \eqref{eq:lmr-master-pred}, 
a similar consideration leads to rewriting it in a form that bounds the slope of the effect of the perturbation on $\M_J^{\rm pred}(\rho,q,r)^{1/r}$:
$$
1 \le\
\frac{\M_J^{\rm pred}(\rho,q,r)^{1/r}-\M_J^{\rm pred}(0,q,r)^{1/r}}{\rho}
 \le\
\|P_X\|_{q\to q}.
$$
Moreover,
 similarly to the case of estimation error under independent contaminations from Section \ref{subsec:ic-location-general}, 
\eqref{eq:lmr-master-pred}  
 implies that
$
\M_J^{\rm est}(\rho,q,r)
\ \ge\
2^{-r}\big(s_{\rm est}+\|X^\dagger\|_{q\to q} \rho\big)^{r}$.
 This 
 establishes the minimax rate for the estimation error as $\M_J^{\rm est}(\rho,q,r)\asymp \big(s_{\rm est}+\|X^\dagger\|_{q\to q} \rho\big)^{r}$.

Since the lower bound technique  has some elements of novelty in the robustness literature, we present it below in Section \ref{subsec:JC-minimax-ellqr}.
The upper bounds are proved in Section \ref{subsec:JC-q-r}.

\begin{figure}[t]{}
    \centering
    \begin{subfigure}[h]{0.49\columnwidth}
        \centering
        \begingroup
       \newlength{\tikzwidth}
        \setlength{\tikzwidth}{0.49\columnwidth}
        \scalebox{1}{\begin{tikzpicture}[scale = 0.8]
    \coordinate (O) at (0,0);
    \coordinate (Y) at (4,4);
    \coordinate (Y') at (6,4);
    \coordinate (Xtheta) at (2.5,1.2);
    \coordinate (Proj) at (4,1.2);
    \coordinate (Par1) at (-1.5,-0.5);
    \coordinate (Par2) at (-0.5,3);
    \coordinate (Par3) at (6,3);
    \coordinate (Par4) at (5,-0.5);

    \fill[gray!15] (Par1) -- (Par2)-- (Par3) -- (Par4) -- cycle;
    \node[xshift=0.8cm, yshift=-0.5cm] at (Par2) {$\mathrm{Col}(X)$};

    \draw[-{Latex[length=2mm, width=2mm]},thick] (O) -- (Y) node[above] {$Y$};

    \draw[-{Latex[length=2mm, width=2mm]},thick] (O) -- (Xtheta) node[below] {$X\theta$};

    \draw[-,dotted] (Y) -- (Proj) node[below right]{$P_X Y$};
    \draw[-{Latex[length=2mm, width=2mm]},dotted,->-] (Y) -- (Y') node[above] {$Y'$} node[midway, below] {$\sim\!\sqrt{n}\ep$};
    \draw[-{Latex[length=2mm, width=2mm]},dotted,->-] (Xtheta) -- (Proj) node[below] {};

    \draw ($(Proj) + (0,0.2)$) -- ($(Proj) + (0.2,0.2)$) -- ($(Proj) + (0.2,0)$);
\end{tikzpicture}}
        \endgroup
    \end{subfigure}
    \begin{subfigure}[h]{0.49\columnwidth}
         \centering
         \includegraphics[width=0.9\columnwidth]{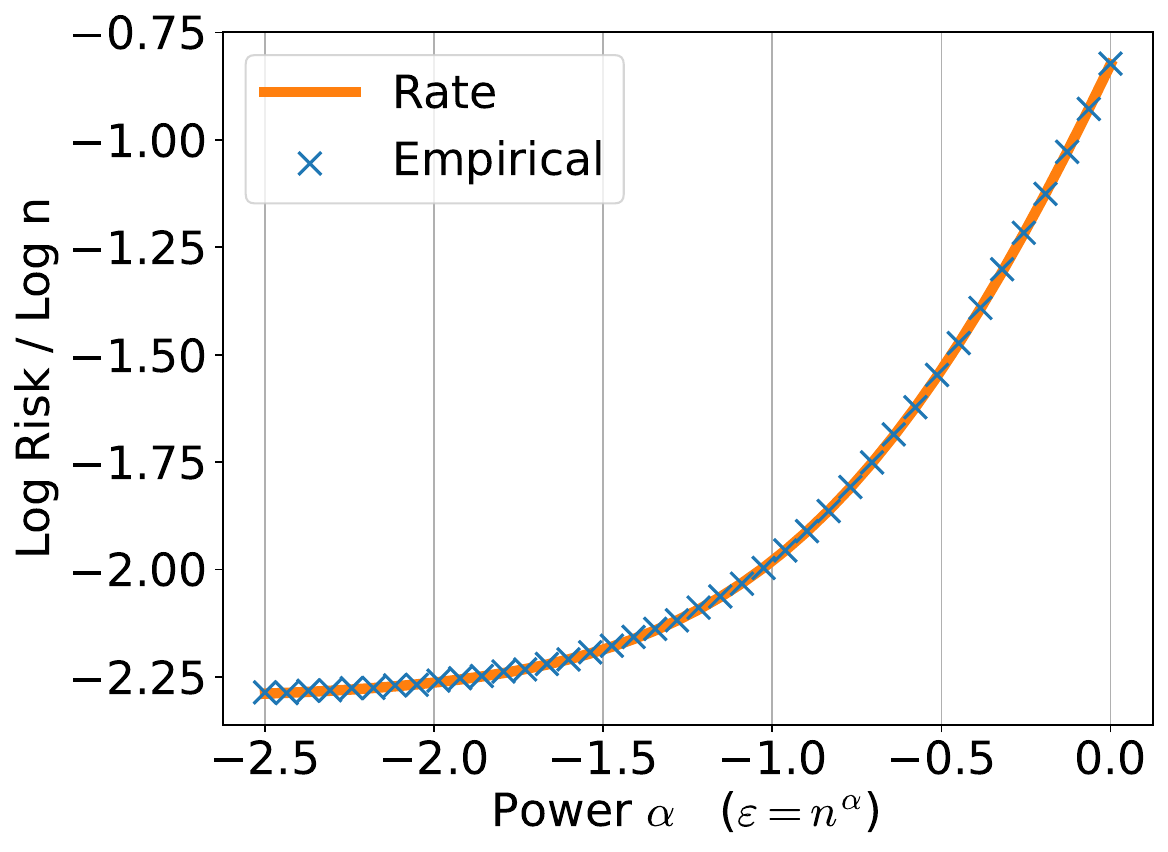}
    \end{subfigure}
    \hfill
    \caption{\textbf{Left:} Graphical representation of linear regression with an adversarial perturbation with the least squares estimate $P_X Y$ of $X\theta$.
    Here, $Y'$ is constructed so that $Y'-Y$ is a scaled multiple of $P_X Y-X\theta$.
    The length is proportional to the perturbation budget, and is indicated for the case $r=q=2$ on the figure.
     \textbf{Right:} Empirical evaluation of the risk in a linear regression problem with $r=q=2$ under homoskedastic error and prediction loss with $n=10$, $p=5$. The empirical performance matches the theoretical rate.}
    \label{fig: lr}
\end{figure}

In the left plot of \cref{fig: lr}, we provide graphical intuition for our least favorable perturbation 
for $q=2$ c
under prediction error. 
The prediction error is governed by the norm of $X(\htheta(Y)-\theta)$. 
Our least favorable perturbation constructs $Y'$ by shifting $Y$ in the direction of the prediction error of the least squares estimator without contamination, $Y'-Y\propto X(\htheta(Y)-\theta)=P_X Y-X\theta$, a construction similar to the JC perturbation for the location estimation problem.
In the right plot of \cref{fig: lr}, we show that the theoretical rate and empirical performance of the least squares estimator under contamination match for $q=r=2$.
For more details on the experiments, see \cref{subsec: lr simulations}.

\subsubsection{Minimax Lower Bounds for Linear Regression}
\label{subsec:JC-minimax-ellqr}

This section presents the
 proofs of the 
lower bounds from Theorem \ref{thm:lmr-master}. 
We begin with 
a construction of a feasible class of contaminations,  
and present reductions that let us pass from arbitrary estimators
to equivariant estimators based on the Hunt-Stein theorem.

Let $\C:=\mathrm{col}(X)$.
 In our proofs for linear regression throughout this paper, 
we will use the following
family of perturbations.

\begin{lemma}[Scaling $+$ signed shift along $\C$]\label{lem:adversary}
Consider $Y\sim P_\theta \in \mathcal{P}$.
Fix $\rho>0$ and let $\zeta,\psi\ge0$, $\delta\in\C$ with $\|\delta\|_q=1$. Let $S$ be a Rademacher random variable independent of everything else. For each $\theta$, set
\(
Y' := Y+\Delta\),
\(\Delta:=\zeta\,P_X\bigl(Y-X\theta\bigr)+S\psi\,\delta=\zeta\,P_XE+S\psi\,\delta.
\)
Then $Q_\theta:=\mathsf{Law}(Y')$ is feasible under JC and
\(
W_{q,r}(Q_\theta,P_\theta)
 \le \zeta\,s_{\rm pred}+\psi.
\)
In particular, $Q_\theta\in\Ball_{\rm J}^{q,r}(\theta,\rho)$ whenever $\zeta\,s_{\rm pred}+\psi\le\rho$.
\end{lemma}

\begin{proof}
Couple deterministically via $Y'=Y+\Delta$. By Minkowski's inequality in $L_r$,
\(
\bigl(\E\|\Delta\|_q^r\bigr)^{1/r}
\le \bigl(\E\|\zeta P_XE\|_q^r\bigr)^{1/r}+\bigl(\E\|S\psi\delta\|_q^r\bigr)^{1/r}
\le\zeta s_{\rm pred}+\psi,
\)
 hence the result follows. 
\end{proof}

{\bf Reduction to equivariant procedures.}
Next, 
we will 
use the classical Hunt–Stein theorem to find a lower bound on the minimax risk via a reduction to equivariant procedures for translation
invariant decision problems (see, e.g., 
\cite{kiefer1966multivariate},
\cite{eaton1989group} or \cite{lehmann1998theory}.

\begin{lemma}[Corollary of Hunt-Stein; Equivariant minimaxity in translation models]\label{lem:equiv-minimax}
Let $p\in\mathbb{N}$ and consider the location model $V=u+W$ on $\R^p$, where
$W$ has a law $\nu$ with $\E\|W\|_q^r<\infty$,
and loss
$L(\hat u, u)=\|\hat u-u\|_q^r$ ($q\in[1,\infty]$, $r\ge 1$).
Let $\mathcal{H}$ be the class of all measurable estimators $m:\R^p\to\R^p$.
Then
\[
\inf_{m\in\mathcal{H}} \sup_{u\in\R^p}\E\|m(V)-u\|_q^r
=
\inf_{c\in \mathbb{R}^p} \E\|W+c\|_q^r.
\]
Moreover, the optimum is achieved for the estimator
$m(v) = v+c$, for all $v$, for the $c$ that minimizes the right hand side.
\end{lemma}

The above result follows from the Hunt-Stein theorem, see e.g., Theorem 
1 of \cite{eaton2021charles} for a modern and accessible statement.
 As is well known, the group of translations on $\R^p$ 
 satisfies the required amenability conditions \citep[see e.g.,][]{bondar1981amenability}.

We are now ready to prove the main lower bound.

\begin{theorem}[Prediction: minimax lower bound for general $(q,r)$ under JC]
\label{thm:minimax-pred}
Fix $q\in[1,\infty]$, $r\ge 1$, and $\rho>0$.
Then
\begin{equation}\label{eq:minimax-pred}
\M_J^{\rm pred}(\rho,q,r)
\ \ge\
\bigl(s_{\rm pred}+\rho\bigr)^{r}.
\end{equation}
Moreover, when $q=2$, one has equality in \eqref{eq:minimax-pred}.
\end{theorem}

\begin{proof}
By the definition of $\M_J^{\rm pred}$ and due to Lemma~\ref{lem:adversary},
for any predictor $\widehat m$,
\[
\sup_{\theta}\sup_{Q\in \Ball_{\rm J}^{ q,r}(\theta,\rho)}
\E_Q\|\widehat m(Y)-X\theta\|_q^r
\ \ge\
\sup_{\theta} \E_\theta\|\widehat m(Y_\theta')-X\theta\|_q^r,
\]
where $Y'=Y_\theta'$ denotes the adversarially perturbed variable of Lemma~\ref{lem:adversary} with 
$\psi=0$ and
$\zeta= \rho/s_{\rm pred}$.
 Now, choose the noise to
 satisfy 
 $P_X^\perp E = 0$,
 and $\Big(\E\|P_XE\|_q^r\Big)^{1/r}= s_{\rm pred}$.
 This choice ensures that the resulting distribution $P_\theta$ is valid
 for all $\theta$, so that $P_\theta \in \mathcal{P}$.
 Moreover, we have 
 $Y_\theta' = P_XY_\theta'$. 

Write $u:=X\theta\in\C$ and $W:=(1+\zeta)P_XE$, so that
$P_XY_\theta'=u+W$ with $W$ independent of $u$ and
$(\E_\theta\|W\|_q^r)^{1/r}=(1+\zeta) s_{\rm pred}=s_{\rm pred}+\rho$.
Hence,
\[
\sup_{\theta} \E_\theta\bigl\|\widehat{m}\bigl(P_XY_\theta'\bigr)-X\theta\bigr\|_q^r
=\sup_{u\in\C} \E_u\|\widehat{m}(u+W)-u\|_q^r.
\]
By the consequence of the Hunt-Stein theorem stated in Lemma~\ref{lem:equiv-minimax},
it follows that we can restrict to equivariant estimators $\widehat m$, and that the right-hand side is bounded below by
$\min_c\E\|W+c\|_q^r$.
 By assumption, the minimum is achieved at $c=0$, and hence the resulting value equals 
$(s_{\rm pred}+\rho)^r$.
This yields \eqref{eq:minimax-pred}.

If $q=2$, equality holds because 
our lower bounds 
show that the least-squares predictor
$\widehat m_{\rm LS}(y)=P_Xy$ has worst-case risk upper bounded by $(s_{\rm pred}+\rho)^2$.
\end{proof}
See Section \ref{subsec:JC-q-r-app} for the proofs of the lower bound for estimation.

\subsection{Random Designs and Joint X,Y-Perturbation}
\label{rxy}

In Section \ref{subsec:XY-JC}, we generalize
 our results to the setting where the feature matrix $X$
 can be random,
 and both the features and the outcome can be contaminated.
 We provide results that are roughly similar 
to the ones presented above, 
 except a bit more involved to state; and therefore we relegate them to the appendix. 

These results have the advantage that they 
enable us to 
more directly 
compare our results with some prior work. 
Specifically, 
as discussed in more detail later,
these results are consistent with the rates from \cite{Liu2021RobustWE}.\footnote{As confirmed via personal communication.} 
Their work studies linear regression under $W_{2,1}$  IC perturbations that can affect both the features and the outcomes, 
under an $\ell_2$ loss.
They obtain the minimax estimation error rate
$\max(\sqrt{p/n}, \ep)$---up to some unspecified constants---and our rate agrees with this under reasonable  conditions, see Section \ref{ic-reg} for more details.
 Specifically, our results can match this when the feature matrix is well-conditioned, which is reasonable because linear regression involves matrix inversion.

However, our results have some additional differences from those in \cite{Liu2021RobustWE} that we think make them valuable:
    Our results are more explicit and do not involve unknown constants. 
   Further, they are applicable to fixed or generally distributed design matrices  $X$, whereas they require normally distributed feature matrices.
    Finally, we have results for arbitrary $W_{q,r}$ perturbations, compared to $W_{2,1}$ in \cite{Liu2021RobustWE}. 

Moreover, conceptually, 
our results
 have a somewhat different message, as we 
argue that even classical methods like linear regression can be robust in certain regimes, while also clarifying the reason why they might not be robust (e.g., poor conditioning of the feature matrices).
Since linear regression a widely used statistical method, we believe that any results about its robustness ought to be of broad interest.

Zhu, Jiao, and Steinhardt
\cite{zhu2020robuststat}
 also study estimation of a regression parameter when the true data follows a linear model, using an IC $W_{2,1}$ perturbation and a $\ell_2$ prediction error that---for well-conditioned population-level feature covariance matrices---are  up to order equivalent to a squared estimation error. 
Their  $W_{2,1}$
 results depend on assuming that $k$th
  moments of the residuals exists, for $k>2$---see e.g., their Theorem 4.4---whereas we do not have such an assumption, 
thus, our rates are not straightforward to compare. 
Moreover, 
our contributions differ quite significantly,
 as we aim to provide bounds without unknown constants (whereas their results have unspecified universal constants), and for 
  for arbitrary $W_{q,r}$ (not just $W_{2,1}$)
  perturbations.

\section{Pointwise Nonparametric Density Estimation}
We now consider the problem of nonparametric density estimation under IC shifts. 
Fix a dimension $p\in\mathbb{N}$.
Following the notation in \cite{tsybakov2009nonparametric}, 
for 
a smoothness parameter 
$s>0$, write $s=l+\alpha$, where $l$ is 
a non-negative integer and $\alpha\in(0,1]$. 
A $p$-dimensional multi-index $\beta = (\beta_1,\ldots,\beta_p) \in \{0,1,\ldots\}^p$ is a vector of non-negative indices,
and we denote
 its order by 
$|\beta| = \sum_{j=1}^p \beta_j$. 
The operator $D^\beta=\partial^{\beta_1}_{x_1} \ldots \partial^{\beta_p}_{x_p}$ denotes taking partial derivatives of order $\beta$.
The $p$–dimensional Hölder class is
\begin{equation*}
\Sigma_p(s,L)
~\coloneqq~
\Bigl\{ f:\R^p\to\R \ \Big|\
|D^\beta f(x)-D^\beta f(y)|\le L\|x-y\|_2^\alpha \ \text{for all }x,y\in\R^p, |\beta|=l
\Bigr\}.
\end{equation*}
Let $\cP_p(s,L)$ be the class of probability densities on $\R^p$ belonging to $\Sigma_p(s,L)$:
\begin{equation*}
\cP_p(s,L)
~\coloneqq~
\Bigl\{f:\R^p\to[0,\infty)\ \Big|\ \int_{\R^p}f(x) dx=1, f\in\Sigma_p(s,L)\Bigr\}.
\end{equation*}

We will consider IC perturbations $\tilde f$ of $f$ such that $W_{q,r}(f,\tilde f)\le \ep$.
We observe $X_1',\dots,X_n'$ $\simiid \tilde f$,
where\footnote{Additional moment assumptions, such as the existence of a suitable $r$-th moment, which are used in other parts of the paper, are not required here due to our use of a bounded kernel in the estimation method.} 
$\tilde f\in\cP_p(s',L')$
is allowed to have a different smoothness level $s'$ (typically lower than $s$), 
as well as a different, typically higher, Lipschitz constant $L'$. 

We fix $x_0\in\R^p$.
For a convex nondecreasing $\Phi:\R_{\ge 0}\to\R_{\ge 0}$ with $\Phi(0)=0$, 
define the IC minimax risk for estimating the true density $f(x_0)$ at the point $x_0$ based on the contaminated data $X_1',\dots,X_n'\simiid \tilde f$:
\begin{equation}\label{demi}
\M_{I}(\ep;n, q,r, s,L; s',L', \Phi)
~\coloneqq~
\inf_{\widehat f}
\sup_{{f\in\cP_p(s,L), \tilde f\in\cP_p(s',L'),\, W_{q,r}(f,\tilde f)\le \ep}}
\E_{\tilde f}\Big[\Phi\big(|\widehat f(x_0)-f(x_0)|\big)\Big].
\end{equation}
We will abbreviate the above risk as 
$\M_{I}(\ep, s, p, r)$.
Below $\asymp$
will denote equivalence up to problem hyperparameters
other than $n$ or $\ep$. 
Further, we will consider 
loss functions $\Phi$ 
that satisfy the doubling condition 
$\Phi(2x)\le D_\Phi\Phi(x)$ for all $x\ge0$, i.e., they do not grow faster than exponential. 
This includes our main loss functions of interest,  
such as $\Phi(x)=x^r$, $r\ge 1$, for all $x\ge 0$.

Our main result shows that the 
rate of the 
minimax risk
is determined by the maximum of two terms:
the standard rate and
a term dependent on $\varepsilon$.

\begin{theorem}[Density estimation under Wasserstein contamination]\label{thm: pointwise density bounds}
Let $p\in\mathbb{N}$, $s>0$, $s'>0$, $L,L'>0$, 
$r\in[1,\infty)$, and let $\Phi$ be convex,
nondecreasing, $\Phi(0)=0$, and satisfy $\Phi(2x)\le D_\Phi\Phi(x)$ for all $x\ge0$. Then for all $n\ge1$ and $\ep\in(0,1]$,
the minimax risk \eqref{demi} for $s$-H\"older smooth density estimation in $p$ dimensions 
from
 $\ep$-bounded 
$W_{q,r}$ i.i.d.~Wasserstein contaminated data
has the rate 
\begin{equation*}
\M_{I}(\ep, s, p, r)
\ \asymp\
\Phi\!\Big(n^{-\frac{s}{2s+p}} \vee\ \ep^{\frac{s}{ s+1+p/r }}\Big).
\end{equation*}
\end{theorem}

Moreover, the upper bound is achieved by kernel density estimation with suitably smooth and bounded kernels (as detailed in the proof) and bandwidth \(
h^\star \asymp
  n^{-\frac{1}{2s+p}}
   \vee 
  \ep^{\frac{1}{ s+1+p/r }}.
  \)
This shows that, compared to the unperturbed case,
the minimax risk and the required level of smoothing in the estimator does not change if $\ep$ is sufficiently small that
\(
\ep \lesssim n^{-\frac{s+1+p/r}{2s+p}}.
\)
However, the risk is dominated by the contaminations above that level, and the required level of smoothing is also determined by the magnitude of contaminations. 
For
 the special case of square loss where 
$\Phi(u)=u^2$, for all $u$, we obtain that the rate is 
\(
\M_{I}(\ep, s, p, r)
\asymp n^{-\frac{2s}{2s+p}} \vee \ep^{\frac{2s}{ s+1+p/r }}.
\)

Intriguingly, the rate 
\emph{does not depend on the $\ell_q$ norm used in the perturbation,
nor 
on the sampling smoothness $s'$}.
 An examination of the proof shows that
these only enter the analysis via constants, e.g.,
 $s'$ 
appears only through the 
uniform boundedness of $\tilde f$ used in the variance bound.

In terms of proofs, 
the analysis leverages
various special properties of 
Wasserstein distance;
for instance,
a
representation of the effect of kernel smoothing on the perturbed density via a linear interpolation of the identity map with the optimal transport map, 
combined with the property that the probability flows induced by optimal transport maps have bounded densities.
Also, we leverage
the dynamic Benamou-Brenier formulation
 of optimal transport 
\citep{brenier2004extended}
in the lower bound
in order to show that the perturbations we construct respect the required budget.
Since we believe these proofs have elements of novelty and interest, we present them below.

\subsection{Upper Bound}

We first bound the pointwise error by isolating three contributions: 
the smoothing bias, the stochastic fluctuation, and the adversarial shift controlled by $W_{q,r}$.
Let $r^\star=r/(r-1)$ (with $r^\star=\infty$ when $r=1$).

{\bf Kernel assumptions in $\R^p$.}
Let $K:\R^p\to\R$ be bounded with $\int_{\R^p}K=1$, and moment cancellations $\int_{\R^p} u^\beta K(u) du=0$ for all multiindices $\beta$ with $1\le|\beta|\le l$. 
Assume 
$K\in L^2$, 
$\|\nabla K\|_{L^{r^\star}(\R^p;\|\cdot\|_{q^\star})}<\infty$,
and 
$\int \|u\|_2^s |K(u)|du<\infty$.

{\bf Estimator in $\R^p$.}
For a scalar bandwidth $h\in(0,1]$, 
define the scaled kernel $K_h(u)=h^{-p}K(u/h)$ and define
 the kernel density estimator 
$\widehat f_h(x_0)=n^{-1}\sum_{i=1}^n K_h(x_0-X_i')$.

\begin{theorem}[Upper bound in $\R^p$ with heterogeneous smoothness]\label{thm:upper-Wr-general-Phi-hetero-d}
With the above notations, 
there exists $C=C(p,s,L,s',L',\Phi)$ such that for all $n\ge1$, $\ep\in(0,1]$, $h\in(0,1]$, $r\in[1,\infty)$, and any $f\in\cP_p(s,L)$, $\tilde f\in\cP_p(s',L')$ with $W_{q,r}(f,\tilde f)\le \ep$,
\begin{equation*}
\E_{\tilde f}\Big[\Phi\big(|\widehat f_h(x_0)-f(x_0)|\big)\Big]
\ \le\
C \Phi\!\Big(h^s+\ep h^{-(1+p/r)}+\sqrt{1/(n h^p)}+1/(n h^p)\Big).
\end{equation*} 
Consequently, choosing \(
h^\star \asymp
  n^{-\frac{1}{2s+p}}
   \vee 
  \ep^{\frac{1}{ s+1+p/r }},
  \)
\begin{equation*}
\inf_{h\in(0,1]} \sup_{\substack{f\in\cP_p(s,L), \tilde f\in\cP_p(s',L')\\ W_{q,r}(f,\tilde f)\le\ep}}
\E_{\tilde f}\Big[\Phi\big(|\widehat f_h(x_0)-f(x_0)|\big)\Big]
\ \lesssim\
\Phi\!\Big(n^{-\frac{s}{2s+p}} \vee\ \ep^{\frac{s}{ s+1+p/r }}\Big).
\end{equation*}
\end{theorem}

\begin{proof}
Let $m_h(x_0;g)=\int K_h(x_0-y)g(y) dy$ denote the effect of smoothing by the scaled kernel on a true density at $x_0$. 
Multiindex Taylor expansion with Hölder remainder and the 
order-$l$ moment cancellations yield $|m_h(x_0;f)-f(x_0)|\le C h^s$ uniformly over $f\in\cP_p(s,L)$, see e.g., \cite{tsybakov2009nonparametric}.

For the adversarial shift, apply Lemma~\ref{lem:Wr-shift-d} 
and use uniform boundedness of $f,\tilde f$ over $\cP_p(s,L)$ and $\cP_p(s',L')$:
\(
|m_h(x_0;\tilde f)-m_h(x_0;f)|\le C \ep
\|\nabla K_h\|_{L^{r^\star}(\R^p;\|\cdot\|_{q^\star})}
\le C \ep h^{-(1+p/r)}.
\)

For the fluctuation, $Y_i=K_h(x_0-X_i')-\E_{\tilde f}K_h(x_0-X_i')$ satisfy $|Y_i|\le C h^{-p}$
due to the boundedness of $K$,
and
\(
\Var{Y_i}\le \|\tilde f\|_\infty \|K_h\|_{L^2}^2 \le C(s',L') h^{-p}.
\)
Hence $E=n^{-1}\sum_i Y_i$ is sub-gamma with variance factor $v\le C/(n h^p)$ and scale $b\le C/(n h^p)$. Lemma~\ref{lem:Phi-calibration} gives
\(
\E\Phi(|E|)\le C_\Phi \Phi\big(C(\sqrt{1/(n h^p)}+1/(n h^p))\big).
\)

Putting together the smoothing bias, the adversarial shift, and the stochastic term, and using that $\Phi$ is convex 
and nondecreasing, 
we obtain the claimed $\Phi$–risk bound.

Optimizing $h$ by balancing $h^s$ with $\sqrt{1/(n h^p)}$, $1/(n h^p)$, and $\ep h^{-(1+p/r)}$ yields the displayed rate.
\end{proof}

\begin{lemma}[Adversarial shift bound in $\R^p$]\label{lem:Wr-shift-d}
Let $p\in\mathbb{N}$, $r\in[1,\infty)$.
Fix $x_0\in\R^p$ with $\nabla K\in L^{r^\star}(\R^p;\|\cdot\|_{q^\star})$. For any bounded densities $f,\tilde f$ on $\R^p$, when $r>1$
\begin{equation*}
\big| (K_h\!*\tilde f)(x_0)-(K_h\!*f)(x_0) \big|
\ \le\
 \big(\max(\|f\|_\infty,\|\tilde f\|_\infty)\big)^{1/r^\star}  \|\nabla K_h\|_{L^{r^\star}(\R^p;\|\cdot\|_{q^\star})} W_{q,r}(f,\tilde f).
\end{equation*}
When $r=1$,
\[
\big|(K_h*\tilde f)(x_0)-(K_h*f)(x_0)\big|
\ \le\ 
\|\nabla K_h\|_{L^\infty(\R^p;\|\cdot\|_{q^\star})}\, W_{q,1}(f,\tilde f),
\]
Moreover $\|\nabla K_h\|_{L^{r^\star}(\R^p;\|\cdot\|_{q^\star})}
=h^{-(1+p/r)}\|\nabla K\|_{L^{r^\star}(\R^p;\|\cdot\|_{q^\star})}$.
\end{lemma}

\begin{proof}
First, consider $r>1$.
Let $q\in (1,\infty)$.
Since the cost function $\|x-y\|_q^r$, $r>1$, $q>1$  is strictly convex,
and the density $f$ induces an absolutely continuous measure,
the optimal transport map $T$ pushing $f$ to $\tilde f$, 
exists and is unique almost surely with respect to the density $f$,
see e.g., 
\cite{GangboMcCann1996},
or
Theorem 10.28 or 10.38, and surrounding discussion in \citet{villani2008optimal}. 
Writing 
$\varphi(x)=K_h(x_0-x)$, 
$\Delta(y)=T(y)-y$ and $S_t(y)=(1-t)y+tT(y)$ being the linear interpolation of the identity map with the optimal transport map,
 we obtain
 that $(K_h\!*\tilde f)(x_0)-(K_h\!*f)(x_0)$, the effect of smoothing on the perturbed density equals: 
\begin{equation*}
\int\big(\varphi(T(y))-\varphi(y)\big)f(y) dy
=\int_0^1\!\!\int \nabla\varphi\big(S_t(y)\big)^\top
 \Delta(y)  f(y) dy dt.
\end{equation*}
Since $S_t$ is injective for $t\in (0,1)$ due to Lemma 4.23 of \cite{santambrogio2015optimal}, where we use that our cost function is strictly convex, 
the density of $\nu_t:=(S_t)_\#(f dy)$
exists; denote it by $\rho_t$.
For each $t\in(0,1)$,
\begin{equation*}
\int \big|\nabla\varphi(S_t(y))^\top \Delta(y)\big|  f(y) dy
\ \le\
\Big(\int \|\nabla\varphi(z)\|_{q^\star}^{r^\star} \rho_t(z) dz\Big)^{1/r^\star}\Big(\int \|\Delta(y)\|_q^r f(y) dy\Big)^{1/r}.
\end{equation*}
By
Lemma \ref{maxp},
$\|\rho_t\|_\infty
\le \max(\|f\|_\infty,\|\tilde f\|_\infty)$. 
Thus, 
\begin{align*}
&\big| (K_h\!*\tilde f)(x_0)-(K_h\!*f)(x_0) \big|
\ \le\
\Big(\int_0^1\!\!\int \|\nabla\varphi\|_{q^\star}^{r^\star} \rho_t\Big)^{\!1/r^\star}
\Big(\int_0^1\!\!\int \|\Delta\|_q^r f\Big)^{\!1/r}\\
&\ \le\ \big(\max(\|f\|_\infty,\|\tilde f\|_\infty)\big)^{1/r^\star}
\|\nabla\varphi\|_{L^{r^\star}(\R^p;\|\cdot\|_{q^\star})} W_{q,r}(f,\tilde f).
\end{align*}

Next,
consider $r>1$ but 
$q\in \{1,\infty\}$.
Pick any sequence $q_k\in(1,\infty)$ with $q_k\to q$ and let $q_k^\star$ be the dual exponents. For each $k$, apply the already‑proved inequality with $q_k$.
Now pass to the limit $k\to\infty$.
By equivalence of norms on $\mathbb{R}^p$, for all vectors $u$ and all $k$,
$\|u\|_{q_k^\star}\le p^{|1/q_k^\star-1/q^\star|}\|u\|_{q^\star}$ and $p^{|1/q_k^\star-1/q^\star|}\to1$. Dominated convergence yields
$\|\nabla K_h\|_{L^{r^\star}(\|\cdot\|_{q_k^\star})}\to \|\nabla K_h\|_{L^{r^\star}(\|\cdot\|_{q^\star})}$.
Likewise, by equivalence of the ground norms,
$W_{q_k,r}(f,\tilde f)\to W_{q,r}(f,\tilde f)$, with the two distances sandwiching each other by factors $p^{|1/q_k-1/q|}\to1$.

Taking limits gives the desired $q\in\{1,\infty\}$ bound with the same constant. 
This completes the proof for all $q\in[1,\infty]$.

For $r=1$,
write $\mu=f\,dy$ and $\nu=\tilde f\,dy$. By definition,
\[
(K_h*\tilde f)(x_0)-(K_h*f)(x_0)=\int_{\R^p}\varphi(y)\,(\tilde f(y)-f(y))\,dy
=\int_{\R^p}\varphi\,d(\nu-\mu).
\]
Due to the Kantorovich--Rubinstein duality on the metric space $(\R^p,\ell_q)$,
\[
\Big|\int \varphi\,d(\nu-\mu)\Big|\ \le\ \mathrm{Lip}_q(\varphi)\,W_{q,1}(\mu,\nu), \quad 
\mathrm{Lip}_q(\psi):=\sup_{x\ne y}\frac{|\psi(x)-\psi(y)|}{\|x-y\|_q}.
\]
Now, 
due to the fundamental theorem of calculus,
\[|\varphi(y)-\varphi(x)|
\le \int_0^1 \|\nabla\varphi(\gamma(t))\|_{q^\star}\,dt\,\|y-x\|_q
\le \Big(\sup_{z\in\R^p}\|\nabla\varphi(z)\|_{q^\star}\Big)\|y-x\|_q,
\]
hence $\mathrm{Lip}_q(\varphi)\le \sup_{z}\|\nabla\varphi(z)\|_{q^\star}$. 
Since $\varphi(y)=K_h(x_0-y)$,
 it follows that $\mathrm{Lip}_q(\varphi)$ $= $ 
$\|\nabla K_h\|_{L^\infty(\R^p;\|\cdot\|_{q^\star})}$,
completing the proof for $r=1$.

Finally, $\|\nabla K_h\|_{L^{r^\star}(\R^p;\|\cdot\|_{q^\star})}$
$=$
$h^{-1-p(1-1/r^\star)}\|\nabla K\|_{L^{r^\star}(\R^p;\|\cdot\|_{q^\star})}$ by scaling.
\end{proof}

\subsection{A Matching Lower Bound}

\begin{theorem}[Lower bound in $\R^p$ with heterogeneous smoothness]\label{thm:lower-Wr-Phi-hetero-d}
Fix $p\in\mathbb{N}$, 
$s>0$, $s'>0$, $L,L'>0$, $r\in[1,\infty)$, and any convex nondecreasing $\Phi$ with $\Phi(0)=0$. In pointwise estimation of $f(x_0)$ over $f\in\cP_p(s,L)$ under i.i.d.\ $W_{q,r}$–contamination with sampling density $\tilde f\in\cP_p(s',L')$ and budget $\ep\in(0,1]$,
\begin{equation*}
\M_{I}(\ep, s, p, r)\
\gtrsim\
\Phi\!\Big(n^{-\frac{s}{2s+p}} \vee\ \ep^{\frac{s}{ s+1+p/r }}\Big),
\end{equation*}
where the implied constant depends only on $(p,s,L)$, $(s',L')$, the choice of kernel template below, and $\Phi$.
\end{theorem}

\begin{proof}
 The proof defines two densities
 $f_+,f_-$
 whose Wasserstein distance is bounded but 
 for which the value
 $|f_+(x_0)-f_-(x_0)| $
 is large. 
 This is performed by starting from a normal density
 for which the Lipschitz constant is suitably bounded and which is truncated to a suitable box. 
 Then we add to it two perturbations and control the required quantity.

{\bf A base normal density.}
Write $s=l+\alpha$ with $l$ being a non-negative integer and $\alpha\in(0,1]$, and similarly $s'=l'+\alpha'$.
Let $\mathcal{G}=\phi_\sigma$ be the centered Gaussian density with covariance $\sigma^2 I_p$. 
 For any multiindex $\beta$, $D^\beta \phi_\sigma(x)=\sigma^{-|\beta|-p}P_\beta(x/\sigma)\phi_\sigma(x)$ where $P_\beta$ is a polynomial depending only on $(\beta,p)$.
 Hence, for the $j$-th unit basis vector $e_j$, 
 $\mathrm{Lip}(D^\beta\phi_\sigma)=\|D^{\beta+e_j}\phi_\sigma\|_\infty$ for some $j$, so $\mathrm{Lip}(D^\beta\phi_\sigma)\lesssim_{d,\beta}\sigma^{-|\beta|-1-d}$. 
Taking $\sigma$ sufficiently large, we 
can ensure 
$\mathrm{Lip}(D^\beta\phi_\sigma) \le L/4$
 and the analogous bound with $(s',L')$ in place of $(s,L)$. Thus 
$\mathcal{G}\in\cP_p(s,L/4)\cap\cP_p(s',L'/4)$. 

Fix a box $J$ centered at $x_0$ and a number $h_{\max}\in(0,1]$ such that $J$ contains
\[
I_h := [x_{01}-\tfrac{h}{2}, x_{01}+\tfrac{3h}{2}]\times\prod_{j=2}^p[x_{0j}-\tfrac{h}{2}, x_{0j}+\tfrac{h}{2}]
\]
for $h=h_{\max}$.
Since $\mathcal{G}$ is continuous and strictly positive, $c_\ast:=\inf_{x\in J}\mathcal{G}(x)>0$.

{\bf Defining the perturbations.}
Let  
$K_0:\R\to[0,\infty)$ be 
a smooth and compactly supported kernel. 
defined
by $K_0(u)=\exp(-1/(1-u^2)) \mathbf 1_{\{\|u\|_\le 1\}}$, $u\in \R$,
and define
$
W_\perp(u_2,\dots,u_p) = \prod_{j=2}^p K_0(2u_j)$, 
$B(u_1) = K_0(2u_1)-K_0(2u_1-1)$.
For $u=(u_1,\ldots,u_p)\in\R^p$,
 define the un-normalized perturbation
$$T_{\mathrm{unit}}(u) := B(u_1) \cdot  W_\perp(u_2,\ldots,u_p).$$
Then $T_{\mathrm{unit}}\in C^\infty_c(\R^p)$, 
$\mathrm{supp}(T_{\mathrm{unit}})\subset[-\tfrac12,\tfrac32]\times[-\tfrac12,\tfrac12]^{p-1}$, and
\(
\int_{\R} B(u_1) du_1 = \tfrac12\!\int K_0-\tfrac12\!\int K_0 = 0.
\)
Hence, for every fixed $u_\perp=(u_2,\ldots,u_p)$, $\int_{\R} T_{\mathrm{unit}}(u_1,u_\perp) du_1=0$. In particular $\int_{\R^p} T_{\mathrm{unit}}=0$, and $T_{\mathrm{unit}}(0)=K_0(0)\prod_{j=2}^p K_0(0)>0$. 
For each multiindex $\beta$ let $M_\beta=\|D^\beta T_{\mathrm{unit}}\|_\infty$ and $L_\beta=\mathrm{Lip}(D^\beta T_{\mathrm{unit}})$; 
these are finite and do not depend on  $n,\ep$.

Fix $a>0$ small enough to satisfy
\begin{equation}\label{cnd}
a \|T_{\mathrm{unit}}\|_\infty\ \le\ \frac{c_\ast}{2L h_{\max}^s},\qquad
a (L_\beta+2M_\beta)\ \le\ \frac34\ \ \text{for all } \beta \text{ with } |\beta|=l.
\end{equation}
Define $T:=a T_{\mathrm{unit}}$. For $h\in(0,h_{\max}]$ set $\tau:=L h^s$ and 
define the two perturbations $f_\pm$ through 
\[
g_h(x) := \tau  T\!\big((x-x_0)/h\big),\qquad f_\pm := \mathcal{G}\pm g_h.
\]
Then $\int_{\R^p} g_h=0$ and $\mathrm{supp}(g_h)\subset I_h\subset J$.

We also have
$\|g_h\|_\infty\le \tau a \|T_{\mathrm{unit}}\|_\infty\le L h_{\max}^s a \|T_{\mathrm{unit}}\|_\infty\le c_\ast/2$ by \eqref{cnd}. 
Hence $f_\pm\ge \mathcal{G}-\|g_h\|_\infty\ge c_\ast/2>0$ on $J$; outside $J$, $f_\pm=\mathcal{G}\ge0$. Moreover $\int f_\pm=1$, showing that $f_\pm$ are valid probability densities.

{\bf Verifying that the perturbations satisfy the required Hölder conditions.}
For the Hölder seminorm, write $s=l+\alpha$, $\alpha\in(0,1]$. 
Consider any multi-index $\beta$ with $|\beta|=l$.
If $\|x-y\|_2\le h$,
\begin{align*}
|D^{\beta} g_h(x)-D^{\beta} g_h(y)|
&= \tau h^{-l} a |D^{\beta} T((x-x_0)/h)-D^{\beta} T((y-x_0)/h)|\\
&\le \tau h^{-l} a L_\beta \frac{\|x-y\|_2}{h}
= L a L_\beta \|x-y\|_2  h^{\alpha-1}.    
\end{align*}
Since $\|x-y\|_2\le h$, we have $\|x-y\|_2 h^{\alpha-1}\le \|x-y\|_2^\alpha$, hence the right--hand side is at most $L a L_\beta \|x-y\|_2^\alpha$. If $\|x-y\|_2>h$, then, since  $M_\beta=\|D^\beta T_{\mathrm{unit}}\|_\infty$ so that $\|D^\beta T\|_\infty= M_\beta$,
\[
\begin{aligned}
&|D^{\beta} g_h(x)-D^{\beta} g_h(y)|
 \le\ 2 \|D^{\beta} g_h\|_\infty 
= 2 \sup_{z\in\R^p}\Big|\tau h^{-l} D^\beta T\!\Big(\frac{z-x_0}{h}\Big)\Big|
\ \le\ 2 \tau h^{-l} \|D^\beta T\|_\infty \\[2pt]
&= 2 (L h^s) h^{-l} (a M_\beta)
= 2 L a M_\beta h^{\alpha}
\le\ 2 L a M_\beta \|x-y\|_2^{\alpha},
\end{aligned}
\]
since $\|x-y\|_2>h$ and $\alpha\in(0,1]$.
By \eqref{cnd}, 
$a(L_\beta+2M_\beta)\le 3/4$, so combining both cases yields $|D^{\beta} g_h(x)-D^{\beta} g_h(y)|\le (3/4)L \|x-y\|_2^\alpha$. Adding the contribution of $\mathcal{G}$ (bounded by $L/4$ since $\mathcal{G}\in\cP_p(s,L/4)$) gives
\(
|D^{\beta} f_\pm(x)-D^{\beta} f_\pm(y)|\ \le\ L \|x-y\|_2^\alpha.
\)
Hence
$f_\pm\in\cP_p(s,L)$.

{\bf Bounding the Wasserstein distance.}
For each $h,x$, 
let $G_{1,h}(x)=\int_{-\infty}^{x_1} g_h(s,x_2,\dots,x_p) ds$ and $U_h(x)=(G_{1,h}(x),0,\dots,0)$.
Because $T_{\mathrm{unit}}$ has zero mean in the first coordinate 
for every $(u_2,\ldots,u_p)$, it follows that
 $S_{\mathrm{unit}}(u):=\int_{-\infty}^{u_1}T_{\mathrm{unit}}(t,u_2,\ldots,u_p) dt$
 vanishes when $u_1\le -1/2$ or $u_1\ge 3/2$. 
 Hence $\mathrm{supp}(S_{\mathrm{unit}})\subset [-\tfrac12,\tfrac32]\times[-\tfrac12,\tfrac12]^{p-1}$. 
A change of variables shows $G_{1,h}(x)=L a h^{s+1} S_{\mathrm{unit}}\big((x-x_0)/h\big)$, so $U_h$ is supported in $I_h$ and
\[
\int_{\R^p} \|U_h(x)\|_q^r dx 
= L^r a^r h^{r(s+1)} h^p \int_{\R^p} \|S_{\mathrm{unit}}(u)\|_q^r du
\ \asymp\ L^r h^{rs+r+p}.
\]
Applying Lemma~\ref{lem:local-Wr-d2} with the box $I=I_h$, the density $\mathcal{G}$, as well as $g$, the vector field $U_h$, and $c_I = c_\ast$,
yields
\[
W_{q,r}^r(f_\pm,\mathcal{G})\ \le\ \frac{2^{r-1}}{c_\ast^{ r-1}}\int_{\R^p}\|U_h\|_q^r
\ \lesssim\ L^r h^{rs+r+p},
\]
so $W_{q,r}(f_\pm,\mathcal{G})\lesssim L h^{s+1+p/r}$.

{\bf Characterizing the size of $|f_+(x_0)-f_-(x_0)|$.}
Pick $h\asymp \ep^{1/(s+1+p/r)}$. Then $W_{q,r}(f_\pm,\mathcal{G})\le \ep$ for $h$ small enough. 
At $x_0$,
\[
|f_+(x_0)-f_-(x_0)| = 2|g_h(x_0)| = 2\tau |T(0)|\ \asymp\ L h^{s} \asymp\ \ep^{ s/(s+1+p/r)}.
\]

{\bf Lower bounding the risk.}
Under both $f_+,f_-\in\cP_p(s,L)$,
the sample is i.i.d.\ from the same density $\tilde f=\mathcal{G} \in\cP_p(s',L')$, with $W_{q,r}(f_\pm,\tilde f)\le\ep$. 
Thus, by Lemma~\ref{lem:identifiability-Wr}, for any estimator $\widehat f(x_0)$,
\[
\max\Big\{\E_{\tilde f}\Phi\big(|\widehat f(x_0)-f_+(x_0)|\big), \E_{\tilde f}\Phi\big(|\widehat f(x_0)-f_-(x_0)|\big)\Big\}
\ \ge \Phi\!\Big(\tfrac12 |f_+(x_0)-f_-(x_0)|\Big),
\]
which is $\gtrsim\ \Phi\!\big(\ep^{ s/(s+1+p/r)}\big)$.
This gives the contamination-driven lower bound.

Finally, the classical uncontaminated pointwise minimax lower bound over $\cP_p(s,L)$ 
is of order at least
$\Phi\big(c n^{-s/(2s+p)}\big)$ for some $c>0$; combining the two contributions completes the proof.
\end{proof}

\begin{lemma}[Local $W_{q,r}$ bound in $\R^p$ via the dynamic formulation]\label{lem:local-Wr-d2}
Let $\mathcal{G}$ be a probability density on $\R^p$, let $I=\prod_{j=1}^p[a_j,b_j]\subset\R^p$ be a bounded box, and let $g\in L^1(\R^p)$ satisfy $\mathrm{supp}(g)\subset I$ and $\int_{\R^p} g=0$. 
Assume $\mathcal{G}(x)\ge c_I>0$ on $I$ and $\|g\|_\infty\le c_I/2$. 
For $r\in[1,\infty)$, 
consider 
a vector field $U\in L^r(\R^p;\R^p)$ 
with $\mathrm{supp}(U)\subset I$ and $\nabla\!\cdot U=g$.
Then 
\[
W_{q,r}^r(\mathcal{G}\pm g,\mathcal{G})\ \le\ 
\frac{2^{ r-1}}{c_I^{ r-1}}\int_{\R^p}\|U(x)\|_q^r dx.
\]
\end{lemma}

\begin{proof}
 The proof leverages the 
 the dynamic Benamou--Brenier 
formulation 
between 
$\mathcal{G}$
and
$\mathcal{G} \pm g$
to bound the Wasserstein distance. 

Because $\|g\|_\infty\le c_I/2$ and $\mathrm{supp}(g)\subset I$, we have for $x\in I$ that
\(
\mathcal{G}(x)\pm g(x)\ge c_I-\|g\|_\infty\ge \frac{c_I}{2}>0,
\)
and on $\R^p\setminus I$ one has $\mathcal{G}\pm g=\mathcal{G}\ge0$. 
Moreover $\int_{\R^p}(\mathcal{G}\pm g)=\int \mathcal{G}\pm\int g=1$, so $\mathcal{G}\pm g$ are probability densities.

For $t\in[0,1]$ define $\rho_t:=\mathcal{G}+t g$ and $v_t(x):=-U(x)/\rho_t(x)$. 
Since $\rho_t(x)\ge c_I/2$ on $I$
and $U$ vanishes outside $I$, this is well defined.
We claim that $(\rho_t,v_t)_{t\in[0,1]}$ satisfies the continuity equation
\(
\partial_t\rho_t+\nabla\!\cdot(\rho_t v_t)=0\)
with endpoints $\rho_0=\mathcal{G}$ and $\rho_1=\mathcal{G}+g$.
Indeed $\partial_t\rho_t=g$, 
while
$\nabla\!\cdot(\rho_t v_t)=\nabla\!\cdot(-U)=-\nabla\!\cdot U=-g$,
hence the sum vanishes.

For any $r\in[1,\infty)$ the dynamic Benamou--Brenier 
formulation, 
see e.g., \cite{brenier2004extended}, Theorem 2.2, with $k:z\mapsto\|z\|_q^r$, or \cite{jimenez2008dynamic},
gives
\[
W_{q,r}^r(\mathcal{G}+g,\mathcal{G})
\ \le\ \int_0^1\!\!\int_{\R^p}\|v_t(x)\|_q^r \rho_t(x) dx dt
\]
for every pair $(\rho_t,v_t)_{t\in[0,1]}$ satisfying the continuity equation.
By construction,
$
\|v_t(x)\|_q^r\rho_t(x)=$ $\|U(x)\|_q^r \rho_t(x)^{-(r-1)}.
$
Because $U$ vanishes outside $I$ and $\rho_t(x)\ge c_I/2$ on $I$,
\[
\int_{\R^p}\|v_t\|_q^r\rho_t
= \int_{I}\frac{\|U(x)\|_q^r}{\rho_t(x)^{r-1}} dx
\le \Big(\frac{2}{c_I}\Big)^{r-1}\!\!\int_{I}\|U(x)\|_q^r dx.
\]
Integrating in $t\in[0,1]$ yields the desired bound for 
\(
W_{q,r}^r(\mathcal{G}+g,\mathcal{G}).\)
Replacing $g$ by $-g$ (and $U$ by $-U$) gives the same bound for $W_{q,r}^r(\mathcal{G}-g,\mathcal{G})$.
\end{proof}

We record a simple identifiability lower bound.

\begin{lemma}[Identifiability reduction under i.i.d. $W_{q,r}$–contamination]\label{lem:identifiability-Wr}
Fix $x_0\in\R^p$. 
Let $\Phi:\R_{\ge 0}\to\R_{\ge 0}$ be convex and nondecreasing. 
Consider two densities $f_+,f_-$ such that there exists a common perturbation $\tilde f$ for which
the data $X_1',\dots,X_n'$ are i.i.d. from the same distribution $\tilde f$. 
Let $\widehat f(x_0)$ be any estimator of the density $f(x_0)$ based on $X_1',\dots,X_n'$.
Then
\[
\max\Big\{ \E_{\tilde f} \Phi\big(|\widehat f(x_0)-f_+(x_0)|\big),
           \E_{\tilde f} \Phi\big(|\widehat f(x_0)-f_-(x_0)|\big)\Big\}
\ \ge\
\Phi\!\Big(\frac{|f_+(x_0)-f_-(x_0)|}{2}\Big).
\]
\end{lemma}
\begin{proof}
Set $a=f_{+}(x_0)$ and $b=f_{-}(x_0)$. 
Because the sample law is the same under $\theta_+$ and $\theta_-$, namely $\tilde f^{\otimes n}$, the distribution of the random variable $Y:=\hat f(x_0)$ is the same in both cases. 
Consequently,
\[
\E_{\theta_+,\tilde f}\Phi\big(|\widehat f(x_0)-f_{+}(x_0)\big|)
=\E\Phi\big(| Y-a |\big),
\qquad
\E_{\theta_-,\tilde f}\Phi\big(|\widehat f(x_0)-f_{-}(x_0)\big|)
=\E\Phi\big(| Y-b |\big),
\]
where $\E$ denotes expectation under the common law of $Y$.

Since  $\Phi:\R_{\ge 0}\to\R_{\ge 0}$ is convex and nondecreasing,
we have for all $y,a,b\in\R$:
\[
\frac{\Phi(|y-a|)+\Phi(|y-b|)}{2}
\ \ge \Phi\!\Big(\frac{|y-a|+|y-b|}{2}\Big)
\ \ge \Phi\!\Big(\frac{|a-b|}{2}\Big).
\]
Taking expectations with respect to the common law of $Y$ yields
\(
\frac{1}{2}\Big(
\E \Phi(|Y-a|)+\E \Phi(|Y-b|)
\Big)
\ge \Phi\!\Big(\frac{|a-b|}{2}\Big).
\)
Reinstating the original expectations
completes the proof. \qedhere
\end{proof}

\section{Discussion}

In this paper, we have developed a theoretical analysis of Wasserstein perturbations in several fundamental statistical problems, such as mean estimation, linear regression, and density estimation. 
Focusing on $W_{q,r}$ perturbations, we have shown that classical estimators such as the sample mean, linear regression, and kernel density estimation can be rate-optimal or even exactly optimal under certain conditions.
 However, we have also identified settings where the robustness of classical methods may face challenges, including
 in linear regression 
 for poorly conditioned feature matrices. 
 Going forward, it would be of interest to understand to what extent classical methods can be regularized or adapted in such settings, and to what extent fundamentally new methods are required to be optimal under Wasserstein contamination. 

\section*{Acknowledgements}

We thank Chao Gao,  Zheng Liu, and Yuekai Sun for helpful discussion.
We thank Ankit Pensia for pointing out an issue in a previous version in the definition of joint contaminations. 
This work was supported in part by the NSF and the ARO.
We thank an associate editor and anonymous reviewers
 for helpful suggestions that have helped improve the paper. 

    {\small
        \setlength{\bibsep}{0.2pt plus 0.3ex}
        \bibliographystyle{plainnat-abbrev}
        \bibliography{references}

\begin{thebibliography}{48}
\providecommand{\natexlab}[1]{#1}
\providecommand{\url}[1]{\texttt{#1}}
\expandafter\ifx\csname urlstyle\endcsname\relax
  \providecommand{\doi}[1]{doi: #1}\else
  \providecommand{\doi}{doi: \begingroup \urlstyle{rm}\Url}\fi

\bibitem[Ambrosio et~al.(2008)Ambrosio, Gigli, and Savar{\'e}]{ambrosio2008gradient}
L.~Ambrosio, N.~Gigli, and G.~Savar{\'e}.
\newblock \emph{Gradient flows: in metric spaces and in the space of probability measures}.
\newblock Springer Science \& Business Media, 2008.

\bibitem[B{\'a}ndi et~al.(2019)]{Bndi2019FromDO}
P.~B{\'a}ndi et~al.
\newblock From detection of individual metastases to classification of lymph node status at the patient level: The {C}amelyon17 challenge.
\newblock \emph{IEEE Transactions on Medical Imaging}, 38:\penalty0 550--560, 2019.

\bibitem[Blanchet and Kang(2017)]{blanchet2017distributionally}
J.~Blanchet and Y.~Kang.
\newblock Distributionally robust groupwise regularization estimator.
\newblock In \emph{Asian Conference on Machine Learning}. PMLR, 2017.

\bibitem[Blanchet et~al.(2019)Blanchet, Kang, and Murthy]{blanchet2019robust}
J.~Blanchet, Y.~Kang, and K.~Murthy.
\newblock Robust wasserstein profile inference and applications to machine learning.
\newblock \emph{Journal of Applied Probability}, 56\penalty0 (3):\penalty0 830--857, 2019.

\bibitem[Bondar and Milnes(1981)]{bondar1981amenability}
J.~V. Bondar and P.~Milnes.
\newblock Amenability: A survey for statistical applications of hunt-stein and related conditions on groups.
\newblock \emph{Zeitschrift f{\"u}r Wahrscheinlichkeitstheorie und verwandte Gebiete}, 57\penalty0 (1):\penalty0 103--128, 1981.

\bibitem[Brenier(2004)]{brenier2004extended}
Y.~Brenier.
\newblock Extended monge-kantorovich theory.
\newblock In \emph{Optimal Transportation and Applications: Lectures given at the CIME Summer School, held in Martina Franca, Italy, September 2-8, 2001}, pages 91--121. Springer, 2004.

\bibitem[Casella and Berger(2024)]{casella2024statistical}
G.~Casella and R.~Berger.
\newblock \emph{Statistical inference}.
\newblock Chapman and Hall/CRC, 2024.

\bibitem[Chen et~al.(2016)Chen, Gao, and Ren]{chen2016general}
M.~Chen, C.~Gao, and Z.~Ren.
\newblock A general decision theory for huber’s $\epsilon$-contamination model.
\newblock \emph{Electronic Journal of Statistics}, 10\penalty0 (2):\penalty0 3752--3774, 2016.

\bibitem[Chen et~al.(2018)Chen, Gao, and Ren]{chen2018robust}
M.~Chen, C.~Gao, and Z.~Ren.
\newblock Robust covariance and scatter matrix estimation under huber’s contamination model.
\newblock \emph{The Annals of Statistics}, 46\penalty0 (5):\penalty0 1932--1960, 2018.

\bibitem[Chen et~al.(2020)Chen, Paschalidis, et~al.]{chen2020distributionally}
R.~Chen, I.~C. Paschalidis, et~al.
\newblock Distributionally robust learning.
\newblock \emph{Foundations and Trends{\textregistered} in Optimization}, 4\penalty0 (1-2):\penalty0 1--243, 2020.

\bibitem[Clarkson(1936)]{clarkson1936uniformly}
J.~A. Clarkson.
\newblock Uniformly convex spaces.
\newblock \emph{Transactions of the American Mathematical Society}, 40\penalty0 (3):\penalty0 396--414, 1936.

\bibitem[Courty et~al.(2017)Courty, Flamary, Tuia, and Rakotomamonjy]{Courty2017OTDA}
N.~Courty, R.~Flamary, D.~Tuia, and A.~Rakotomamonjy.
\newblock Optimal transport for domain adaptation.
\newblock \emph{IEEE Transactions on Pattern Analysis and Machine Intelligence}, 39\penalty0 (9):\penalty0 1853--1865, 2017.
\newblock \doi{10.1109/TPAMI.2016.2615921}.

\bibitem[Damodaran et~al.(2018)Damodaran, Flamary, Tuia, and Courty]{Damodaran2018DeepJDOT}
B.~B. Damodaran, R.~Flamary, D.~Tuia, and N.~Courty.
\newblock Deepjdot: Deep joint distribution optimal transport for unsupervised domain adaptation.
\newblock In \emph{Proceedings of the European Conference on Computer Vision (ECCV)}, 2018.
\newblock arXiv:1803.10081.

\bibitem[Donoho and Liu(1988)]{Donoho1998MD}
D.~L. Donoho and R.~C. Liu.
\newblock {The ``automatic'' robustness of minimum distance functionals}.
\newblock \emph{The Annals of Statistics}, 16\penalty0 (2):\penalty0 552 -- 586, 1988.

\bibitem[Duchi et~al.(2021)Duchi, Glynn, and Namkoong]{duchi2021statistics}
J.~C. Duchi, P.~W. Glynn, and H.~Namkoong.
\newblock Statistics of robust optimization: A generalized empirical likelihood approach.
\newblock \emph{Mathematics of Operations Research}, 46\penalty0 (3):\penalty0 946--969, 2021.

\bibitem[Eaton(1989)]{eaton1989group}
M.~L. Eaton.
\newblock \emph{Group invariance applications in statistics}.
\newblock IMS, 1989.

\bibitem[Eaton and George(2021)]{eaton2021charles}
M.~L. Eaton and E.~I. George.
\newblock Charles stein and invariance: Beginning with the hunt--stein theorem.
\newblock \emph{The Annals of Statistics}, 49\penalty0 (4):\penalty0 1815--1822, 2021.

\bibitem[Esfahani and Kuhn(2015)]{esfahani2015dro}
P.~Esfahani and D.~Kuhn.
\newblock Data-driven distributionally robust optimization using the wasserstein metric: Performance guarantees and tractable reformulations.
\newblock \emph{Mathematical Programming}, 171, 05 2015.

\bibitem[Figalli and Glaudo(2021)]{figalli2021invitation}
A.~Figalli and F.~Glaudo.
\newblock \emph{An invitation to optimal transport, Wasserstein distances, and gradient flows}.
\newblock EMS, 2021.

\bibitem[Gangbo and McCann(1996)]{GangboMcCann1996}
W.~Gangbo and R.~J. McCann.
\newblock The geometry of optimal transportation.
\newblock \emph{Acta Mathematica}, 177\penalty0 (2):\penalty0 113--161, 1996.
\newblock \doi{10.1007/BF02392620}.

\bibitem[Hampel et~al.(2005)Hampel, Ronchetti, Rousseeuw, and Stahel]{hampel2005robust}
F.~R. Hampel, E.~M. Ronchetti, P.~J. Rousseeuw, and W.~A. Stahel.
\newblock \emph{Robust statistics}.
\newblock Wiley, 2005.

\bibitem[Huber(2004)]{huber2004robust}
P.~J. Huber.
\newblock \emph{Robust statistics}.
\newblock John Wiley \& Sons, 2004.

\bibitem[Huber(1964)]{huberLocation}
P.~J. Huber.
\newblock {Robust estimation of a location parameter}.
\newblock \emph{The Annals of Mathematical Statistics}, 35\penalty0 (1):\penalty0 73 -- 101, 1964.

\bibitem[Huber(1965)]{huberRatio}
P.~J. Huber.
\newblock {A Robust Version of the Probability Ratio Test}.
\newblock \emph{The Annals of Mathematical Statistics}, 36\penalty0 (6):\penalty0 1753 -- 1758, 1965.

\bibitem[H{\"u}tter and Rigollet(2021)]{hutter2021minimax}
J.-C. H{\"u}tter and P.~Rigollet.
\newblock Minimax estimation of smooth optimal transport maps.
\newblock \emph{The Annals of Statistics}, 49\penalty0 (2):\penalty0 1166--1194, 2021.

\bibitem[Jimenez(2008)]{jimenez2008dynamic}
C.~Jimenez.
\newblock Dynamic formulation of optimal transport problems.
\newblock \emph{Journal of Convex Analysis}, 15\penalty0 (3):\penalty0 593, 2008.

\bibitem[Kiefer(1966)]{kiefer1966multivariate}
J.~Kiefer.
\newblock Multivariate optimality results.
\newblock \emph{Multivariate Analysis.}, pages 255--274, 1966.

\bibitem[Koskela(1979)]{koskela1979}
M.~Koskela.
\newblock Some generalizations of clarkson's inequalities.
\newblock \emph{Publikacije Elektrotehničkog fakulteta. Serija Matematika i fizika}, \penalty0 (634-677):\penalty0 89--93, 1979.

\bibitem[Kuhn et~al.(2019)Kuhn, Esfahani, Nguyen, and Shafieezadeh-Abadeh]{kuhn2019wasserstein}
D.~Kuhn, P.~M. Esfahani, V.~A. Nguyen, and S.~Shafieezadeh-Abadeh.
\newblock Wasserstein distributionally robust optimization: Theory and applications in machine learning.
\newblock In \emph{Operations Research \& Management Science in the Age of Analytics}, pages 130--166. INFORMS, 2019.

\bibitem[Lee and Raginsky(2017)]{lee2017minimax}
J.~Lee and M.~Raginsky.
\newblock Minimax statistical learning with wasserstein distances.
\newblock In \emph{Neural Information Processing Systems}, 2017.

\bibitem[Lehmann and Casella(1998)]{lehmann1998theory}
E.~Lehmann and G.~Casella.
\newblock Theory of point estimation.
\newblock \emph{Springer Texts in Statistics}, 1998.

\bibitem[Liu and Loh(2022)]{Liu2021RobustWE}
Z.~Liu and P.-L. Loh.
\newblock {Robust W-GAN-based estimation under Wasserstein contamination}.
\newblock \emph{Information and Inference: A Journal of the IMA}, 12\penalty0 (1):\penalty0 312--362, 2022.

\bibitem[Manole et~al.(2022)Manole, Balakrishnan, and Wasserman]{manole2019minimax}
T.~Manole, S.~Balakrishnan, and L.~Wasserman.
\newblock {Minimax confidence intervals for the Sliced Wasserstein distance}.
\newblock \emph{Electronic Journal of Statistics}, 16\penalty0 (1):\penalty0 2252 -- 2345, 2022.

\bibitem[Ohta(2009)]{ohta2009finsler}
S.-i. Ohta.
\newblock Finsler interpolation inequalities.
\newblock \emph{Calculus of Variations and Partial Differential Equations}, 36\penalty0 (2):\penalty0 211--249, 2009.

\bibitem[Peyr{\'e} and Cuturi(2019)]{peyre2019computational}
G.~Peyr{\'e} and M.~Cuturi.
\newblock Computational optimal transport: With applications to data science.
\newblock \emph{Foundations and Trends{\textregistered} in Machine Learning}, 11\penalty0 (5-6):\penalty0 355--607, 2019.

\bibitem[Rahimian and Mehrotra(2019)]{rahimian2019distributionally}
H.~Rahimian and S.~Mehrotra.
\newblock Distributionally robust optimization: A review.
\newblock \emph{arXiv preprint arXiv:1908.05659}, 2019.

\bibitem[Santambrogio(2015)]{santambrogio2015optimal}
F.~Santambrogio.
\newblock Optimal transport for applied mathematicians.
\newblock \emph{Birk{\"a}user, NY}, 55\penalty0 (58-63):\penalty0 94, 2015.

\bibitem[Shafieezadeh~Abadeh(2020)]{shafieezadeh2020wasserstein}
S.~Shafieezadeh~Abadeh.
\newblock Wasserstein distributionally robust learning.
\newblock Technical report, EPFL, 2020.

\bibitem[Shafieezadeh-Abadeh et~al.(2018)Shafieezadeh-Abadeh, Nguyen, and Kuhn]{ShafieezadehAbadeh2018}
S.~Shafieezadeh-Abadeh, V.~A. Nguyen, and D.~Kuhn.
\newblock Wasserstein distributionally robust {K}alman filtering.
\newblock In \emph{Advances in Neural Information Processing Systems (NeurIPS) 32}, pages 8474--8485, 2018.

\bibitem[Shen et~al.(2018)Shen, Qu, Zhang, and Yu]{Shen2018WDGRL}
J.~Shen, Y.~Qu, W.~Zhang, and Y.~Yu.
\newblock Wasserstein distance guided representation learning for domain adaptation.
\newblock In \emph{Proceedings of the AAAI Conference on Artificial Intelligence}, volume~32, 2018.

\bibitem[Singh and P{\'o}czos(2018)]{singh2018minimax}
S.~Singh and B.~P{\'o}czos.
\newblock Minimax distribution estimation in wasserstein distance.
\newblock \emph{arXiv preprint arXiv:1802.08855}, 2018.

\bibitem[Staib(2020)]{staib2020learning}
M.~J. Staib.
\newblock \emph{Learning and optimization in the face of data perturbations}.
\newblock PhD thesis, Massachusetts Institute of Technology, 2020.

\bibitem[Tsybakov(2009)]{tsybakov2009nonparametric}
A.~B. Tsybakov.
\newblock \emph{Introduction to Nonparametric Estimation}.
\newblock Springer, 2009.

\bibitem[Vershynin(2018)]{vershynin2018high}
R.~Vershynin.
\newblock \emph{High-dimensional probability: An introduction with applications in data science}, volume~47.
\newblock Cambridge university press, 2018.

\bibitem[Villani(2003)]{villani2003topics}
C.~Villani.
\newblock \emph{Topics in optimal transportation}.
\newblock American Mathematical Soc., 2003.

\bibitem[Villani(2008)]{villani2008optimal}
C.~Villani.
\newblock \emph{Optimal transport: old and new}, volume 338.
\newblock Springer Science \& Business Media, 2008.

\bibitem[Weed and Berthet(2019)]{weed2019estimation}
J.~Weed and Q.~Berthet.
\newblock Estimation of smooth densities in wasserstein distance.
\newblock In \emph{Conference on Learning Theory}, pages 3118--3119. PMLR, 2019.

\bibitem[Zhu et~al.(2022)Zhu, Jiao, and Steinhardt]{zhu2020robuststat}
B.~Zhu, J.~Jiao, and J.~Steinhardt.
\newblock {Generalized resilience and robust statistics}.
\newblock \emph{The Annals of Statistics}, 50\penalty0 (4):\penalty0 2256 -- 2283, 2022.

\end{thebibliography}
    }

\appendix

\clearpage

\section{Motivation for Wasserstein Contamination}
\label{app: motivation wasserstein}
Wasserstein contamination describes settings where 
the data follows a distribution that deviates from a ground truth distribution by only a small amount in terms of the Wasserstein metric (with respect to a certain norm on the space of the data). 
In contrast to Huber’s $\ep$-contamination model---which allows an $\ep$ fraction of arbitrary outliers---the Wasserstein model allows
 \emph{every} data point to be slightly perturbed. 

Examples 
where such contaminations may arise can
include systematic measurement errors and sensor bias.
For instance, 
suppose that each datapoint represents a 
measurement that may be affected by a small unknown readout bias---for example, all readings from a temperature sensor drift by a few degrees. 
The resulting contaminated distribution is a translated version of the true distribution, differing by a small Wasserstein distance. 
In contrast, Huber contamination (which replaces a subset of points arbitrarily) would be unrealistic for such pervasive but bounded distortions.

Additional examples of small Wasserstein shifts may correspond to realistic domain differences between two datasets, such as differences in imaging devices, lighting conditions, geometric 
transformations, or population demographics. 
One source of  motivation 
is the success of optimal transport-based methods
to align training and test distributions in computer vision and medical imaging, see e.g., \cite{Courty2017OTDA,Damodaran2018DeepJDOT,Shen2018WDGRL}, etc. 
As another set of motivating applications in time-series and control, \cite{ShafieezadehAbadeh2018} construct a
Wasserstein distributionally robust Kalman filter that guards against model misspecification by considering all noise distributions within a Wasserstein ball of the nominal distribution, showing improved robustness to model errors.

\section{Proofs}

\subsection{Relation Between Contamination Models}
\label{relc}

Fix $(\mu,\nu^n)\in\cV_I(\ep)$, so $W_{q,r}(\mu,\nu)\le\ep$. Let $\pi\in\Pi(\mu,\nu)$ be an optimal coupling, and set $\Pi:=\pi^{\otimes n}\in\Pi(\mu^n,\nu^n)$. Writing $(Z_i,Z_i')_{i=1}^n\sim\Pi$ with i.i.d.\ pairs $(Z_i,Z_i')\sim\pi$, we have
\begin{equation*}
\Big(\E_\Pi\big\|(Z_1,\ldots,Z_n)-(Z_1',\ldots,Z_n')\big\|^r\Big)^{1/r}
=
\Big(\E_\Pi \frac1n\sum_{i=1}^n \|Z_i-Z_i'\|_q^{\,r}\Big)^{1/r}
\le \ep.
\end{equation*}
By the definition of $W_{\|\cdot\|,r}$ on $\cZ^n$, this shows $W_{\|\cdot\|,r}(\mu^n,\nu^n)\le\ep$, hence $(\mu,\nu^n)\in\cV_J(\ep)$. Therefore $\cV_I(\ep)\subset \cV_J(\ep)$.

\subsection{Location Families}\label{subsec: location proofs app}

\subsubsection{Proof of Theorem \ref{thm:JC-Lqr}}
\label{pfthm:JC-Lqr}
\begin{proof}
We reduce location estimation to a fixed--design linear regression instance covered by Theorem~\ref{thm:lmr-master}, with some minor adaptations.
Let $1_n$ be the all ones vector of size $n$.
Stack the sample into
\(
Y
:= (Z_1^\top,\dots,Z_n^\top)^\top \in \R^{np},
\)
and take the regression design matrix
\(
 X := 1_n\otimes I_p \in \R^{(np)\times p}.
\)
Then
\[
Y =  X\theta + E,
\qquad
E := Y- X\theta
= (Z_1-\theta,\dots,Z_n-\theta),
\]
where $E_i:=Z_1-\theta \sim \mathfrak{f}_0$ are i.i.d.. 
This is a fixed--design linear model with
an $np \times p$ feature matrix and
$\rank( X)=p$.

We can verify that
\(
 X^\dagger
= ( X^\top X)^{-1} X^\top
= (nI_p)^{-1}(1_n^\top\otimes I_p)
= \tfrac1n 1_n^\top\otimes I_p.
\)
Hence, 
\(
\Big(\E\|  X^\dagger E \|_q^{ r}\Big)^{1/r}
= \Big(\E\| \bar E \|_q^{ r}\Big)^{1/r} \le s_n=:s_{\rm est}.
\)

Next, the orthogonal projector onto $\mathrm{col}( X)$ acts as
\(
P_{ X}(z_1,\dots,z_n)=(\bar z,\dots,\bar z)\),
\(\bar z:=\tfrac1n\sum_{i=1}^n z_i,
\)
so 
we have $P_{ X}E=(\bar E,\dots,\bar E)$ with $\bar E=\tfrac1n\sum_i E_i$.

On the product space $(\R^p)^n\cong\R^{np}$, endow the output with the
seminorm
$
\|(z_1,\dots,z_n)\|_{n,q,r}$
$:=\Big(\|\frac1n\sum_{i=1}^n z_i\|_q^{ r}\Big)^{1/r}$.
Then,
\[
\|P_{ X}E\|_{n,q,r}
=\Big(\|\tfrac1n\sum_{i=1}^n \bar E\|_q^{ r}\Big)^{1/r}
=\|\bar E\|_q.
\]
Therefore
\(
\Big(\E\|P_{ X}E\|_{n,q,r}^{ r}\Big)^{1/r}
=\Big(\E\|\bar E\|_q^{ r}\Big)^{1/r}
 \le s_n=:s_{\rm pred}.
\)

Also, we can identify $ X^\dagger$ with the linear map $A:(z_1,\dots,z_n)\mapsto \frac1n\sum_i z_i$ from $\big((\R^p)^n,\|\cdot\|_{n,q,r}\big)$ to $(\R^p,\|\cdot\|_q)$. 
Clearly,
\[
\|A(z_1,\dots,z_n)\|_q
=\Big\|\tfrac1n\sum_{i=1}^n z_i\Big\|_q
\le\|(z_1,\dots,z_n)\|_{n,q,r},
\]
 with equality when all $z_i$ are equal to a unit basis vector. 
Hence
$\| X^\dagger\|_{(n,q,r)\to q}=\|A\|_{(n,q,r)\to q}=1$.

Moreover Condition \ref{ctrlr} clearly reduces to Condition \ref{ctr}.

This shows that 
with this choice, the location 
estimation problem 
with JC perturbations
 with size $s_n$
is 
a special case of the 
linear regression problem with 
output--space 
Wasserstein perturbations used in Theorem~\ref{thm:lmr-master} 
with the output seminorm taken to be $\|\cdot\|_{n,q,r}$
and with
$s_{\rm pred}= s_{\rm est} = s_n$. 
The reason why the location estimation problem may only be a proper subset of the linear regression problem is that the iid noise condition imposes a restriction on the noise, whereas in linear regression we did not have such a restriction. 

Due to this reasoning, 
the
 upper bound in the 
theorem applies verbatim to this slightly different setting because it only uses seminorm properties (homogeneity and Minkowski).
 Moreover, one can readily verify that the lower bound also applies, with a slight modification. 
 Indeed, 
 instead of choosing a noise distribution such that 
 $P_X^\perp E = 0$, 
it is enough to choose it such that---switching to the notation of the proof of Theorem \ref{thm:minimax-pred}---$E \sim \mathcal{N}(0,\sigma^2 I_n )$,
 with $\sigma$ chosen
 so that $\Big(\E\|P_XE\|_q^r\Big)^{1/r}= s_{\rm pred}$.
 This choice ensures that the resulting distribution $P_\theta$ is valid
 for all $\theta$, so that $P_\theta \in \mathcal{P}$.

Observe that $Y_\theta' = P_XY_\theta' + P_{X}^\perp E$, 
and $P_{X}^\perp E$ is independent of $P_{X} E$.
Hence the distribution of $Y_\theta'$ given $P_XY_\theta'$ does not depend on $\theta$.
Therefore, we can
define the map $\tilde m$ via $\tilde m (P_XY_\theta') 
:= \E_\theta(\widehat m(Y_\theta')|P_XY_\theta')$, 
and $\tilde m$ does not depend on the unknown $\theta$, being a valid predictor.

By Jensen's inequality, 
we have that 
\[
\E\|\widehat m(Y_\theta')-X\theta\|_q^r
\ \ge\
\E_\theta\bigl\|\widetilde m\bigl(P_XY_\theta'\bigr)-X\theta\bigr\|_q^r.
\]

 and $P_X E$ saturates the constraint. 
 The remaining steps in the proof of Theorem \ref{thm:minimax-pred} go through unchanged, thus proving the same result. 

Thus, the conclusion follows from 
\eqref{eq:lmr-master-est}, since 
the LS estimator in this regression model is the sample mean, 
and the upper and lower bounds coincide in this case. 
\end{proof}

\subsubsection{Proof of Theorem \ref{thm:IC-struct}}
\label{pfthm:IC-struct}

\begin{proof}
The upper bound 
follows
 from Theorem \ref{thm:JC-Lqr}, as the JC contamination model includes the IC contamination model.

For the lower bound,
fix any unit $\delta\in\R^p$ and any $(\zeta,\psi)\in[0,\infty)^2$ such that
$\zeta s_1+\psi \le \ep$.
Consider the IC perturbation
coupled to the clean sample by
$
Z_i'\ := \theta\ + \psi \delta\ + (1+\zeta) E_i$,
$i=1,\dots,n.
$
By a reasoning analogous to Lemma \ref{lem:adversary} but for the IC location model, this is a feasible set of perturbations. 

Now let 
$E \sim \mathcal{N}(0,\sigma^2 I_n )$,
 with $\sigma$ chosen
 so that $\Big(\E\|E_1\|_q^r\Big)^{1/r}= s_{1}$.
Then $\bar Z'=\theta+(1+\zeta)\bar E+\psi\delta$.
 Moreover, the distribution of $Z'$ given $\bar Z$ 
 does not depend on $\theta$, hence,
 as in the proof of Theorem \ref{thm:JC-Lqr},
 for any estimator $\hat\theta$,
 $\tilde\theta(t):=\E[\hat\theta(Z')\mid \bar Z'=t]$ 
 defines an estimator that by Jensen's inequality can only decrease the risk. 
 Therefore, it suffices to consider estimators that are a function of the sample mean $\bar Z'$. 


Since the minimax risk is lower bounded by the maximum risk over each such family,
 it is also upper-bounded by the average risk when averaging over $\psi$ in an arbitrary set. 
 Therefore, choosing a Rademacher random variable $S$ independent of anything else,
 the minimax risk is lower bounded by that for the contaminations
$\bar Z'=\theta+(1+\zeta)\bar E+S\psi\delta$.

Applying the Hunt–Stein principle (Lemma~\ref{lem:equiv-minimax}) to the location model in $\bar Z'$ with noise $U:=(1+\zeta)\bar E+\psi\delta$ gives 
$$
\sup_{\substack{\zeta,\psi\ge0: \\ \zeta s_1+\psi\le\ep}} 
\E\big\|(1+\zeta)\bar E+\psi \delta\big\|_q^{ r} 
\le \M_I(\ep,q, r).
$$

By the convexity of  $u\mapsto\|u\|_q^r$,
we obtain
$\max\{(1+\zeta)^r s_n^r,\psi^r\}\le \E\|(1+\zeta)\bar E+S\psi\delta\|_q^{ r}$.
Since the first term increases in $\zeta$ and the second decreases in $\zeta$, 
the supremum occurs at an endpoint, so 
using that $s_n  = n^{-1/2} s_1$,
the maximized value over $0\le \zeta\le \ep/s_1$ is
as claimed.
This finishes the proof. 
\end{proof}

\subsubsection{Auxiliary Results for the Proof of Theorem \ref{thm: gaussian mean opt}}

\begin{lemma}\label{lemma: optim set to zero}
    Consider i.i.d.~random vector pairs $(T_i,V_i)$, $i\in[n]$, with distribution $P_{T,V}$ and marginal distribution $P_W$ of $T_i$, and a random vector $E$ with distribution $P_E$. For a constant $\ep>0$, the values of the following two constrained optimization objectives are equal:
    \begin{equation*}
        \begin{split}
            \sup_{P_{V,W},P_E} &\quad  \mathbb{E}{\left\| \sum_{i=1}^n (T_i+ V_i)\right\|_2^2} +2n \mathbb{E}{T_1^\top E}\\
            \textnormal{s.t.} &\quad  \mathbb{E}{\|T_i+V_i\|_2^2}\le \ep^2\\
            &\quad \mathbb{E}{T_i^\top V_i} = 0 \text{ for all } i\in[n].
        \end{split}
        \quad =\quad  \begin{split}
            \sup_{P_W, P_E} &\quad  \mathbb{E}{\left\| \sum_{i=1}^n T_i\right\|_2^2} +2n \mathbb{E}{ T_1^\top E}\\
            \textnormal{s.t.} &\quad  \mathbb{E}{\|T_i\|_2^2}\le \ep^2 \text{ for all } i\in[n].
        \end{split}
    \end{equation*}
\end{lemma}
\begin{proof}
    From expanding the optimization objective and constraints, the problem on the left can be rewritten as
    \begin{align*}
        \sup_{P_{V,W},P_E} & \quad   \sum_{i=1}^n\left(\mathbb{E}{\|T_i\|_2^2}+\mathbb{E}{\|V_i\|_2^2}\right) +2n \mathbb{E}{T_1^\top E} \\
        \textnormal{s.t.}  & \quad  \mathbb{E}{\|T_i\|_2^2}+\mathbb{E}{\|V_i\|_2^2}\le \ep^2,
        \qquad \mathbb{E}{T_i^\top V_i} = 0 \text{ for all } i\in[n].
    \end{align*}

    Consider a choice of $\bigl(T_i^{(1)},V_i^{(1)}\bigr)_{i=1}^n$ where $\mathbb{E}{\|V_i^{(1)}\|_2^2}=\ep_V^2\ge0$, $\mathbb{E}{\|T_i^{(1)}\|_2^2}=\ep_T^2>0$, and
    $\mathbb{E}{{T_1^{(1),\top}} E}\ge 0$ (otherwise choose $-T_1^{(1)}$). Construct $\bigl(T_i^{(2)},V_i^{(2)}\bigr)_{i=1}^n$, where $V_i^{(2)}=0$ and $T_i^{(2)}=c T_i^{(1)}$ with $c=\sqrt{\ep_T^2+\ep_V^2}/\ep_T$.
    Since $c\ge 1$,
    \begin{align*}
         & \sum_{i=1}^n\left(\mathbb{E}{\|T_i^{(1)}\|_2^2}+\mathbb{E}{\|V_i^{(1)}\|_2^2}\right) +2n \mathbb{E}{{T_1^{(1),\top}} E} = \sum_{i=1}^n\left(\ep_T^2+\ep_V^2\right) +2n \mathbb{E}{{T_1^{(1),\top}} E} \\
         & \qquad=\sum_{i=1}^n\left(\mathbb{E}{\|T_i^{(2)}\|_2^2}+\mathbb{E}{\|V_i^{(2)}\|_2^2}\right) +2n \mathbb{E}{{T_1^{(1),\top}} E}                                                                 \\
         & \qquad\le \sum_{i=1}^n\left(\mathbb{E}{\|T_i^{(2)}\|_2^2}+\mathbb{E}{\|V_i^{(2)}\|_2^2}\right) +2n \mathbb{E}{{T_1^{(2),\top}} E}.
    \end{align*}
    Therefore, for any choice of candidates $\bigl(T_i^{(1)},V_i^{(1)}\bigr)_{i=1}^n$ with nonzero $V_i^{(1)}$, we may set $V_i^{(2)}=0$ and do not decrease the optimization objective.
    This shows that the first optimum is less than or equal to the second optimum.
    Conversely, by inspection, the second optimum is less than or equal to the first one.
    Thus the result follows.

\end{proof}

\subsubsection{Proof of Corollary \ref{cor:IC-q2-r1-Gaussian-sph}}
\label{pfthm: gaussian mean opt}

\begin{proof}
The upper bound follows directly from 
Theorem~\ref{thm:IC-struct}.
 Next, for the lower bound, observe that in the proof of 
Theorem \ref{thm:IC-struct},
 we have only used a normal noise for the construction of the lower bound, and hence, the bound stated there remains valid even if the class of distributions that we consider are only normal. 
Therefore, 
it follows that the lower bound is at least
\(
\max\!\Big\{  s_n+\tfrac{\ep}{n^{1/2}}, \ep\Big\}.
\)
 For normal data,
 $s_n=s_1/\sqrt{n}$;
 this leads to the desired result. 
\end{proof}

\subsection{Linear Regression}\label{app: lr proofs}

\subsubsection{Lower Bounds for Linear Regression}
\label{subsec:JC-q-r-app}

\begin{theorem}[Estimation: minimax lower bound with constant shift under JC]\label{thm:minimax-est}
Let $q\in[1,\infty]$, $r\ge1$, and $\rho>0$.
Then
\begin{equation}\label{eq:minimax-est-affine-final}
\M_J^{\rm est}(\rho,q,r)
\ \ge\
\max\Big\{\big(s_{\rm est}+\tfrac{s_{\rm est}}{s_{\rm pred}}\,\rho\big)^{r},\ \big(\|X^\dagger\|_{q\to q}\,\rho\big)^{r}\Big\}.
\end{equation}
\end{theorem}

\begin{proof}

By Lemma~\ref{lem:adversary}, $Y'_\theta=X\theta+(I+\zeta P_X)E+S\psi\delta$ is 
a feasible contamination for any $\theta$,
whenever $\zeta s_{\rm pred}+\psi\le\rho$.
 Now, as in the proof of Theorem \ref{thm:minimax-pred},
choose the noise to
 satisfy 
 $P_X^\perp E = 0$,
 and $\Big(\E\|P_XE\|_q^r\Big)^{1/r}= s_{\rm pred}$.
 Moreover, 
 with this choice,
 $Y_\theta' = P_XY_\theta'$ lies in the column span of $X$.
 Hence,
 any estimator $\hat \theta$ 
 that is a function of $Y_\theta'$ can also be written as a function of $X^\dagger Y_\theta'$, which we now do without any loss of generality. 

 Moreover, 
since $X^\dagger Y'_\theta=\theta+(1+\zeta)X^\dagger E+S\psi X^\dagger\delta$, 
the Hunt–Stein (Lemma~\ref{lem:equiv-minimax}) applied to the $p$-dimensional location model gives 
\begin{equation*}
\M_J^{\rm est}(\rho,q,r)
\ \ge\
\sup_{\substack{\zeta,\psi\ge0:\\ \zeta s_{\rm pred}+\psi\le\rho}}\ 
\inf_{c\in\R^p}\ \E\big\|(1+\zeta)X^\dagger E+S\psi\,X^\dagger\delta+c\big\|_q^r.
\end{equation*}
Due to the distribution of $E$ and $S$, 
the infimum over $c$ is attained at $0$.
Hence, maximizing over $\delta$ and over $\zeta,\psi$ with $\zeta s_{\rm pred}+\psi\le\rho$, yielding
\begin{equation*}
\M_J^{\rm est}(\rho,q,r)
\ \ge\
\sup_{\substack{\zeta,\psi\ge0:\\ \zeta s_{\rm pred}+\psi\le\rho}}\ 
\sup_{\substack{\delta\in\C\\ \|\delta\|_q=1}}\ 
\E\big\|(1+\zeta)X^\dagger E+S\psi\,X^\dagger\delta\big\|_q^r.
\end{equation*}
Hence, 
from convexity of $u\mapsto\|u\|_q^r$ and the symmetry/independence of $S$,
$$
\E\big\|(1+\zeta)X^\dagger E+S\psi\,X^\dagger\delta\big\|_q^r
\ \ge\
\max\Big\{(1+\zeta)^r s_{\rm est}^r,\ \psi^r\|X^\dagger\delta\|_q^r\Big\}.
$$
Maximizing over $\delta$ gives $\|X^\dagger\delta\|_q\le\|X^\dagger\|_{q\to q}$, and optimizing over the linear budget yields \eqref{eq:minimax-est-affine-final}.
\end{proof}

\subsubsection{Upper Bounds for Linear Regression}
\label{subsec:JC-q-r}

This section presents the
 proofs of the 
upper bounds from Theorem \ref{thm:lmr-master}. 
The next lemma 
gives a bound for the worst-case risk
of \emph{linear} functionals of the centered response under $W_{q,r}$-balls.

\begin{lemma}
\label{lem:minkowski}
Consider an integer $m\ge 1$,
let $T:\R^n\to\R^m$ be linear, let $E=Y-X\theta$, and define
\(
s_T:=\bigl(\E\|TE\|_q^{ r}\bigr)^{1/r}.
\)
Then, for every $\rho\ge 0$,
\begin{equation}\label{eq:upper-general-T}
\sup_{Q\in \mathsf{Ball}_{\rm J}^{ q,r}(\theta,\rho)}
\E_Q\big\|T(Y'-X\theta)\big\|_q^{ r}
\le
\bigl(s_T^2+\|T\|_{q\to q} \rho\bigr)^{r}.
\end{equation}
\end{lemma}

\begin{proof}
Fix $\rho\ge 0$ and $Q\in \mathsf{Ball}_{\rm J}^{ q,r}(\theta,\rho)$.
By definition of $W_{q,r}$, 
there exists a coupling
$\pi_\eta\in\Pi(Q,P_\theta)$ of $(Y',Y)$ such that
\(
\Big(\E_{\pi_\eta}\|Y'-Y\|_q^{ r}\Big)^{1/r}\le \rho.
\)
Write $\Delta:=Y'-Y$. Since $T$ is linear and $\|\cdot\|_q$ is a norm,
\[
\|T(Y'-X\theta)\|_q
=\|T(Y-X\theta+\Delta)\|_q
\le \|T(Y-X\theta)\|_q+\|T\Delta\|_q
\le \|TE\|_q+\|T\|_{q\to q} \|\Delta\|_q.
\]
Applying Minkowski’s inequality in $L_r(\pi_\eta)$,
\[
\Big(\E_{\pi_\eta}\|T(Y'-X\theta)\|_q^{ r}\Big)^{1/r}
\le \Big(\E \|TE\|_q^{ r}\Big)^{1/r}
   + \|T\|_{q\to q}\Big(\E_{\pi_\eta}\|\Delta\|_q^{ r}\Big)^{1/r}
\le s_\sigma^2+\|T\|_{q\to q}\rho.
\]
Taking the supremum over $Q$ yields \eqref{eq:upper-general-T}.

\end{proof}

For the upper bound on the 
prediction loss,
we apply Lemma~\ref{lem:minkowski} with $T=P_X$.
 For the estimation loss,
 we use it with we use it with $T=X^\dagger$.

\subsection{Pointwise Density Estimation}
\label{pdepf}

The following lemma presents a variant of well-known density bounds along optimal transport plans. 
Similar results are well-known in the literature, 
but most of them seem to be stated for slightly different assumptions; 
for instance 
Proposition 7.29 of \cite{santambrogio2015optimal} is for Euclidean (as opposed to $\ell_q$) norms,
 while the results of \cite{ohta2009finsler}
 are stated for squared (as opposed to power-$r$) distances. 
Nonetheless, the claim and proof below can be viewed as a folklore statement. 

\begin{lemma}[Density bound along $\Wqr$-geodesics]\label{maxp}
Let $1<q<\infty$ and $p>1$. Let $\mu_0=\rho_0\,dx$ and $\mu_1=\rho_1\,dx$ be probability measures on $\R^p$ with $\rho_0,\rho_1\in L^{\infty}(\R^p)$.
Denote by $(\mu_t)_{t\in[0,1]}$ the constant–speed geodesic in $(\mathcal P_{q,r}(\R^p),\Wqr)$ between $\mu_0$ and $\mu_1$.
Then $\mu_t=\rho_t\,dx$ with $\rho_t\in L^\infty(\R^p)$ for every $t\in[0,1]$, and
\[
\|\rho_t\|_{L^\infty} \le \max\{\|\rho_0\|_{L^\infty},\|\rho_1\|_{L^\infty}\}\quad\text{for all }\infty\in[1,\infty].
\]
\end{lemma}

\begin{proof}
As discussed before,
the constant–speed geodesic connecting $\mu_0$ to $\mu_1$ is
\[
\mu_t=(T_t)_\#\mu_0,\qquad T_t=(1-t)\Id+tT,\quad t\in[0,1],
\]
 where $T$ is the $\mu_0$-unique optimal transport map taking $\mu_0$ to $\mu_1$. 

Write $h(z):=\|z\|_q^r$ for all $z$.
From Theorem 3.7 of \cite{GangboMcCann1996},
there is a 
map $\phi$ such that $T(x)=x-\nabla h^{*}(\nabla\phi(x))$ for $\mu_0$-a.e. $x$. 

Since $\phi$ and the dual potential $\psi$ (introduced below) are convex, they are twice differentiable a.e.\ by Alexandrov's theorem, hence their Hessians exist a.e.  Because $r>1$ and $1<q<\infty$, both $h$ and $h^*$ are strictly convex and $C^2$ away from the origin; moreover, at points where $\nabla\phi(x)\neq0$ we can evaluate $\nabla^2 h^*(\nabla\phi(x))$ classically. At points where $\nabla\phi(x)=0$ we have $T(x)=x$, hence $DT(x)=I$ and the desired inequality holds with equality; we therefore focus on the a.e.\ set where $DT$, $\nabla^2\phi$, and $\nabla^2 h^*(\nabla\phi)$ are well defined.

Fix such an $x$.
Differentiating at $x$ gives
\[
DT(x)=I-HU,\qquad
H:=\nabla^2 h^{*}(\nabla\phi(x))\succ0,\qquad
U:=\nabla^2\phi(x)=U^\top.
\]
Conjugate by $H^{1/2}$ and set
\[
M:=H^{1/2}DT(x)H^{-1/2}=I-H^{1/2}UH^{1/2},
\]
which is symmetric. Then
\(H^{1/2}DT_t(x)H^{-1/2}=(1-t)I+tM.
\)
Due to the concavity of  $\det()^{1/p}$,
\(
\det\big((1-t)I+tM\big)^{1/p}
\ge (1-t)+t\,\det(M)^{1/p}.
\)

Next, we claim that $M\succeq0$.
By $c$-convex duality (e.g. \cite{GangboMcCann1996}), there exists a dual $c$-convex potential $\psi$ such that
\[
\phi(x)+\psi(T(x)) \;=\; c(x,T(x)) \;=\; h(x-T(x)),
\qquad
\nabla\psi(T(x))=-\nabla h\big(x-T(x)\big).
\]
Consider the stationarity relation $\nabla\psi(T(x))=-\nabla h(z)$ with $z:=x-T(x)$. Differentiating in $x$ gives
\[
\nabla^2\psi(T(x))\,DT(x)\;=\;\nabla^2 h(z)\,(I-DT(x)).
\]
Rearranging,
\(
\big(\nabla^2\psi(T(x))+\nabla^2 h(z)\big)\,DT(x)\;=\;\nabla^2 h(z).
\)
Using Legendre duality, $\nabla^2 h(z)=\big(\nabla^2 h^*(\nabla\phi(x))\big)^{-1}=H^{-1}$, so
\(
DT(x)=\big(I+H\,\nabla^2\psi(T(x))\big)^{-1}.
\)
Conjugating by $H^{\pm1/2}$ yields 
\[
M \;=\; H^{-1/2}DT(x)H^{1/2}
\;=\;\Big(I+H^{1/2}\nabla^2\psi\big(T(x)\big)H^{1/2}\Big)^{-1} \succeq 0.
\]

Continuing, set \(
J_t(x):=\det DT_t(x).
\)
Since determinants are invariant under similarity,
it follows that \begin{equation}\label{ji}
J_t(x)^{1/p}\ge (1-t)+t\,J_1(x)^{1/p}.     
\end{equation}

\medskip
To finish the lemma, observe that by the area formula and a.e.\ injectivity of $T_t$, the measure $\mu_t=(T_t)_\#\mu_0$ is absolutely continuous and for a.e.\ $x$,
\[
\rho_t(T_t(x))=\frac{\rho_0(x)}{J_t(x)},\qquad
\rho_1(T(x))=\frac{\rho_0(x)}{J_1(x)}.
\]
Multiplying \eqref{ji} by $\rho_0(x)^{-1/p}$ gives 
\[
\rho_t(T_t(x))^{-1/p}
=\rho_0(x)^{-1/p} J_t(x)^{1/p}
\ge (1-t)\rho_0(x)^{-1/p}
+t\,\rho_1(T(x))^{-1/p}.
\]
If $\rho_0\le M_0$ and $\rho_1\le M_1$ a.e., then $(\dagger)$ implies
\[
\rho_t(T_t(x))^{-1/p}
\ge (1-t)M_0^{-1/p}+tM_1^{-1/p}
\ge \min\{M_0^{-1/p},M_1^{-1/p}\},
\]
hence $\rho_t\le \max\{M_0,M_1\}$ a.e. 
This finishes the proof. 
\end{proof}

\begin{lemma}[Controlling sub-gamma fluctuations]\label{lem:Phi-calibration}
Let $\Phi:\R_{\ge 0}\to\R_{\ge 0}$ be convex, nondecreasing, $\Phi(0)=0$, and satisfy the doubling condition
$\Phi(2x)\le D_\Phi \Phi(x)$ for all $x\ge 0$. If a centered random variable $E$ is sub-gamma with
variance factor $v\ge 0$ and scale $b\ge 0$, i.e.
\(
\log \E e^{\lambda E} \le \frac{\lambda^2 v}{2(1-b\lambda)}, \text{for all }\lambda\in(0,1/b),
\)
then there exists $C>0$ and $C_\Phi\in(0,\infty)$ depending only on $D_\Phi$ such that
\(
\E \Phi(|E|) \le C_\Phi \Phi\big(C(\sqrt v + b)\big).
\)
\end{lemma}

\begin{proof}
For any $t\ge 0$ and any $\lambda\in(0,1/b)$, Chernoff's method gives
\(
\mathbb{P}E\ge t) \le \exp\!\Big(-\lambda t+\frac{\lambda^2 v}{2(1-b\lambda)}\Big).
\)
Optimize the right-hand side by taking
\(
\lambda^\star:=\frac{t}{v+bt}\in(0,1/b)
\)
(since $b\lambda^\star=\frac{bt}{v+bt}<1$). 
This $\lambda^\star$ yields
\(
-\lambda^\star t+\frac{(\lambda^\star)^2 v}{2(1-b\lambda^\star)}
 = -\frac{t^2}{2(v+bt)}.
\)
Hence $\mathbb{P}E\ge t)\le \exp\{-t^2/(2(v+bt))\}$. Applying this to $-E$ and using the union bound,
\begin{equation}\label{eq:subgamma-tail}
\mathbb{P}\big(|E|\ge t\big) \le 2\exp\!\Big(-\frac{t^2}{2(v+bt)}\Big)\qquad\text{for all }t\ge 0.
\end{equation}

\textit{Dyadic decomposition for expectations of monotone functions.}
Fix any threshold $t_0>0$ and define the dyadic grid $I_0=[0,t_0]$ and, for $k\ge 0$,
\(
I_{k+1}:=(2^k t_0, 2^{k+1} t_0].
\)
For any nonnegative random variable $\upsilon$ and any nondecreasing function $\Phi$ with $\Phi(0)=0$,
\begin{equation}\label{eq:dyadic-envelope}
\Phi(\upsilon) \le \Phi(t_0) \mathbf 1_{\{\upsilon\le t_0\}}+\sum_{k=0}^\infty \Phi(2^{k+1}t_0) \mathbf 1_{\{\upsilon>2^k t_0\}}
\ \le \Phi(t_0)+\sum_{k=0}^\infty \Phi(2^{k+1}t_0) \mathbf 1_{\{\upsilon>2^k t_0\}}.
\end{equation}
Taking expectations and using the doubling condition iteratively,
\begin{equation}\label{eq:dyadic-expectation}
\E \Phi(\upsilon) \le \Phi(t_0)+\sum_{k=0}^\infty \Phi(2^{k+1}t_0) \mathbb{P}\upsilon>2^k t_0)
\ \le \Phi(t_0)+D_\Phi\sum_{k=0}^\infty D_\Phi^{k} \Phi(t_0) \mathbb{P}\upsilon>2^k t_0).
\end{equation}
We will apply \eqref{eq:dyadic-expectation} with $\upsilon=|E|$ and bound the tail probabilities using
\eqref{eq:subgamma-tail}.
Take
\begin{equation}\label{eq:t0-choice}
t_0 := 4 \max\{\sqrt v, b\} \le 4(\sqrt v+b).
\end{equation}
For any $t\ge t_0$, because $t\ge t_0\ge \sqrt v$ and $t_0\ge b$, we have
\(
v\ \le t t_0, bt\ \le t t_0\),
\(\Rightarrow
v+bt\ \le 2 t t_0.
\)
Therefore, for $t\ge t_0$, the sub-gamma tail \eqref{eq:subgamma-tail} implies the sub-exponential bound
\begin{equation*}
\mathbb{P}|E|\ge t) \le 2\exp\!\Big(-\frac{t^2}{2(v+bt)}\Big)
\ \le 2\exp\!\Big(-\frac{t}{4t_0}\Big).
\end{equation*}
In particular, at the dyadic points $t=2^k t_0$ ($k\ge 0$),
\begin{equation}\label{eq:dyadic-tail}
\mathbb{P}\big(|E|>2^k t_0\big) \le 2 \exp\!\big(-2^{k-2}\big)\qquad (k\ge 0).
\end{equation}

Combining \eqref{eq:dyadic-expectation} with \eqref{eq:dyadic-tail} and the doubling inequality,
\[
\E \Phi(|E|) \le \Phi(t_0)+D_\Phi \Phi(t_0)\sum_{k=0}^\infty D_\Phi^{k} 2 e^{-2^{k-2}}
 = \Phi(t_0)\Big(1+2D_\Phi\sum_{k=0}^\infty \big(D_\Phi e^{-2^{k-2}}\big)^{\phantom{k}}\Big).
\]
Since $e^{-2^{k-2}}$ decays super-geometrically in $k$, the series
\(
\sum_{k=0}^\infty D_\Phi^{k} e^{-2^{k-2}}
\)
converges for every fixed $D_\Phi<\infty$. Define
\(
C_\Phi := 1+2D_\Phi\sum_{k=0}^\infty D_\Phi^{k} e^{-2^{k-2}} \in(0,\infty),
\)
so that
\begin{equation}\label{eq:main-bound}
\E \Phi(|E|) \le C_\Phi \Phi(t_0).
\end{equation}
By the choice \eqref{eq:t0-choice}, $t_0\le 4(\sqrt v+b)$ and, as $\Phi$ is nondecreasing,
\(
\Phi(t_0) \le \Phi\big(4(\sqrt v+b)\big) \le \Phi\big(C(\sqrt v+b)\big)
\)
for any $C\ge 4$. Inserting this into \eqref{eq:main-bound} yields the claimed inequality
\(
\E \Phi(|E|) \le C_\Phi \Phi\big(C(\sqrt v+b)\big),
\)
with $C_\Phi$ depending only on $D_\Phi$ and a constant $C$.
\end{proof}

\section{Linear Regression: Random Design and Joint Contamination in X,Y}\label{subsec:XY-JC}

We extend the fixed--design results to random designs and joint contaminations in $(X,Y)$.
As before, we consider the linear model from Section \ref{lr}, but we now allow $X$ to be random.
Since the perturbations
can now also affect the features $X$, it is helpful to denote the dependence of estimators and predictors that we consider on the features themselves, writing, for instance, $\widehat\theta(X,Y)$. 
With this notation, the prediction and estimation risks under consideration are:\footnote{A special consideration is that we evaluate the prediction risk on the possibly contaminated sample actually seen; it would also be reasonable to evaluate it on the uncontaminated sample. However, by the triangle inequality this only adds a term of at most $\|(X-X')\theta\|_q$ to the risk. In our results below, we consider parameter spaces with bounded $\theta$, and a term of this order already appears both in the upper and lower bounds. Thus, it turns out that these two situations lead to the same rate for the risk. Therefore, we will proceed with our above choice.}
\[
\cR^{\rm pred}_{q,r}(\widehat m,\theta;Q):=\E_{(X,Y)\sim Q}\|\widehat m(X,Y)-X\theta\|_q^{ r},\quad
\cR^{\rm est}_{q,r}(\widehat\theta,\theta;Q):=\E_{(X,Y)\sim Q}\|\widehat\theta(X,Y)-\theta\|_q^{ r}.
\]

We first present results on joint contaminations over the entire dataset, and
then explain the consequences for independent contaminations across the datapoints in Section \ref{ic-reg}.

\subsection{Joint Contaminations}
\label{jclr}

For a law $\mu$ of $(X,Y)$ on $(\R^p)^n\times\R^n$, define the JC ball with \emph{separated budgets} $\rho_{\mathcal{X}},\rho_{\mathcal{Y}}\ge 0$:
\begin{align}
\label{eq:Ball-XY}
\Ball_{\rm J}^{ q,r}(\mu;\rho_{\mathcal{X}},\rho_{\mathcal{Y}})
:=
\Big\{ & Q:\ \exists \pi\in\Gamma(Q,\mu)\text{ s.t. }\nonumber\\
 &\Big(\E_\pi \frac{1}{n}\sum_{i=1}^n\|X_i'-X_i\|_q^{ r}\Big)^{\!1/r}\le \rho_{\mathcal{X}}, \ 
\big(\E_\pi\|Y'-Y\|_q^{ r}\big)^{1/r}\le \rho_{\mathcal{Y}}
\Big\}.
\end{align}
We choose to parametrize 
perturbations
it in terms of two separate budgets for the size of 
$X$ and $Y$ perturbations.
 The reason is that this makes the current class directly comparable
with the previously studied class of $Y$-only JC contaminations from Section \ref{oc-fd},
which correspond to $\rho_{\mathcal{X}} = 0$ and $\rho_{\mathcal{Y}}=\rho$ there. 

\subsubsection{Summary of Results}

Stating our results requires a number of notations. 
We define 
\(
s_E(q,r):=\Big(\E\|E\|_q^{ r}\Big)^{1/r}.
\)
We also set
\[
\Lambda\equiv\Lambda_{q,r}(X;E):=\frac{s_{\rm pred}(q,r)}{s_E(q,r)},
\]
with the convention $\Lambda=0$ if $s_E(q,r)=0$. 
 The factor $\Lambda_{q,r}(X;E)$ captures how much of a typical noise vector is kept under projection onto $\mathrm{col}(X)$. 
 For $q=r=2$ and mean–zero 
coordinate-wise independent 
 noise with variance $\sigma^2$, one has $s_{\rm pred}(2,2)=(\E\|P_XE\|_2^2)^{1/2}=\sigma\sqrt{n}$
 and $s_E(2,2)=\sigma\sqrt n$, hence $\Lambda_{2,2}(X;E)=\sqrt{p/n}$. 

We also write $s_z(r):=(\E|E_1|^r)^{1/r}$,
$\chi_+(q,r):=\max\{1/r,1/q\}$, 
$\chi_-(q,r):=\min\{1/r,1/q\}$, 
$\kappa_q(n):=n^{2|1/2-1/q|}$, and $K_{n,p,q}:=n^{|1/2-1/q|}p^{|1/2-1/q|}$.

It turns out that if the features can be perturbed, then the minimax risk can scale with the magnitude of the parameters $\theta$. 
Therefore, we will require bounded parameters; as otherwise, the minimax risk can be infinitely large.
Moreover, it turns out that the right norm 
to measure the boundedness of the 
parameters is the $\ell_{q^\star}$ norm 
in the dual exponent $q^\star$  of $q$ ($1/q+1/q^\star=1$, with $1/\infty:=0$).
 Thus, 
for $B\in(0,\infty]$, 
we will consider the bounded parameter space $\Theta_B:=\{\theta:\ \|\theta\|_{q^\star}\le B\}$.
Hence, analogously to \eqref{plr}, we will define the class of clean distributions
\begin{align*}
\mathcal{P}_B
= \Big\{
  &\,Y = X\theta + E:\;
    \|\theta\|_{q^\star}\le B,\;
    \big(\E\|P_X E\|_q^r\big)^{1/r} \le s_{\rm pred},\;
    \big(\E\|X^\dagger E\|_q^r\big)^{1/r} \le s_{\rm est}, \\
  &\;\textnormal{Condition \ref{ctrlr} holds}
\Big\}.\nonumber
\end{align*}
Here Condition \ref{ctrlr} is interpreted as holding conditionally given any realization of $X$, for the conditional distribution of $E$ given $X$.

For the results on estimation, we will also require that both 
the original and the perturbed matrices $X,X'$
have singular vectors bounded below by some $\tau \in (0,\infty]$, 
and the norm
$\|X'\|_{q\to q}\le L$ for some $L \in (0,\infty]$. 
This is equivalent to defining a perturbation set
\begin{align*}
\Ball_{\rm J}^{ q,r}&(\mu;\rho_{\mathcal{X}},\rho_{\mathcal{Y}}, \tau,L)
:=
\Big\{  Q:\ \exists \pi\in\Gamma(Q,\mu)\text{ s.t. }
 \Big(\E_\pi \frac{1}{n}\sum_{i=1}^n\|X_i'-X_i\|_q^{ r}\Big)^{\!1/r}\le \rho_{\mathcal{X}},\\
 \ 
&\big(\E_\pi\|Y'-Y\|_q^{ r}\big)^{1/r}\le \rho_{\mathcal{Y}}, 
s_{\min}(X),   s_{\min}(X')\ge\tau,  \|X'\|_{q\to q}\le L<\infty , \pi\rm{-a.s.}
\Big\}.
\end{align*}
Our results for linear regression with a random design and joint X/Y perturbations can be summarized as follows:

Let $r^\star$ and $q^\star$ be the conjugates of $r$ and $q$, respectively. 
Define $t = \min\{ r, r^\star, q, q^\star\} \ge 1$
and $w_{r,q} = 2^{r(1-1/t)-1}\ge 1/2$.

\begin{theorem}[Random design with joint contamination in $(X,Y)$]\label{thm:master}
Fix $q\in[1,\infty]$ and $r\ge1$, and study linear regression as in \eqref{lre} with an $\ell_q^r$ loss, and a parameter
$\theta\in\Theta_B$ with $B<\infty$.
Consider $W_{q,r}$ JC-perturbations
with separated budgets $\rho_{\mathcal{X}},\rho_{\mathcal{Y}}\ge0$, captured by 
the JC ball 
$\Ball_{\rm J}^{ q,r}(\mu;\rho_{\mathcal{X}},\rho_{\mathcal{Y}})$ 
from \eqref{eq:Ball-XY}.

\noindent{\bf (A) Upper bounds under bounded parameter and joint shift.}
Then, for the least–squares predictor $\widehat m_{\rm LS}(X',Y')=P_{X'}Y'$,
the prediction risk is bounded above by 
\begin{equation}\label{eq:master-upper-pred}
\sup_{Q\in\Ball_{\rm J}^{ q,r}(\mu;\rho_{\mathcal{X}},\rho_{\mathcal{Y}})}
\cR^{\rm pred}_{q,r}(\widehat m_{\rm LS},\theta;Q)
\ \le\
\Big(\kappa_q(n) [  s_E(q,r) + \rho_{\mathcal{Y}} + B  n^{\chi_+(q,r)} \rho_{\mathcal{X}}  ]\Big)^{r}.
\end{equation}
Also,
for the least–squares estimator $\widehat\theta_{\rm LS}$,
the estimation risk
is bounded above by 
\begin{equation}\label{eq:master-upper-est}
\sup_{Q\in\Ball_{\rm J}^{ q,r}(\mu;\rho_{\mathcal{X}},\rho_{\mathcal{Y}},\tau,\infty)}
\cR^{\rm est}_{q,r}(\widehat\theta_{\rm LS},\theta;Q)
\ \le\
\Big(K_{n,p,q} \tfrac{1}{\tau} [  s_E(q,r) + \rho_{\mathcal{Y}} + B  n^{\chi_+(q,r)} \rho_{\mathcal{X}}  ]\Big)^{r}.
\end{equation}

For the lower bounds below, suppose that, conditional on $X$, the distribution of $P_X E$ is symmetric about zero.  

\medskip
\noindent{\bf (B) Lower bounds for prediction.}
For any predictor $\widehat m$ and any $B>0$,
with $\ell:= \Lambda (s_E+\rho_{\mathcal{Y}}) $
the worst-case prediction error is lower bounded as
\begin{equation}\label{eq:JC-lb-joint-struct}
\sup_{\substack{P_\theta\in\mathcal{P}_B
,\\  Q\in\Ball_{\rm J}^{ q,r}(\mu;\rho_{\mathcal{X}},\rho_{\mathcal{Y}})}}
\cR^{\rm pred}_{q,r}(\widehat m,\theta;Q)
\ge\
w_{r,q}
\max\Big\{
\ell^{r},\
 (\Lambda s_E)^r+
\big(\rho_{\mathcal{Y}}+Bn^{\chi_-(q,r)}\rho_{\mathcal{X}}\big)^r
\Big\}=:\vartheta.
\end{equation}

\medskip
\noindent{\bf (C) Lower bounds for estimation.}
For any estimator $\widehat\theta$,
the worst-case estimation error is lower bounded as
\begin{equation}\label{eq:JC-lb-est}
\sup_{\substack{P_\theta\in\mathcal{P}_B,  Q\in\Ball_{\rm J}^{ q,r}(\mu;\rho_{\mathcal{X}},\rho_{\mathcal{Y}},\infty,L)}}
\cR^{\rm est}_{q,r}(\widehat\theta,\theta;Q)
\ge\
L^{-r} \vartheta.
\end{equation}
\end{theorem}

\begin{proof}[Proof of Theorem~\ref{thm:master}]
The prediction and estimation upper bounds \eqref{eq:master-upper-pred}–\eqref{eq:master-upper-est} are Theorems~\ref{thm:pred-joint} and \ref{thm:est-joint}. 
The lower bounds are proved in Section \ref{lbsj}.
\qedhere
\end{proof}

The upper bounds 
depend on the budgets 
$\rho_{\mathcal{X}}$ and $\rho_{\mathcal{Y}}$, 
as well as on the noise magnitude $s_E$, 
as well as on the factors $\kappa_q(n) $, $K_{n,p,q}$, and  
$\chi_+(q,r)$ that mainly arise from moving between various norms in the analysis. 
The lower bounds also depend on the budgets and the noise, 
as well as on the factor $\chi_-(q,r)$. 
In general, there can be a gap between the lower and upper bounds, as it is already true in some settings under the joint contaminations for the Y-only perturbations in Theorem \ref{thm:lmr-master}. 
The lower and upper bound rates match  when $r = q = 2$, 
but can also match for suitable regimes of perturbations even when $q=2$ and $r \neq2$. 

Perturbations of the features make the analysis more challenging it is not clear if the set of techniques used throughout this paper would be able to close the gaps between the lower and the upper bounds. Studying these questions in more detail remains interesting future work. 

\subsubsection{Upper Bounds for OLS under Joint X-Y Perturbations}

In this section, we provide the proofs of the upper bounds for OLS under joint X-Y perturbations. 

\begin{theorem}[Prediction under joint shift; general $(q,r)$]
\label{thm:pred-joint}
Fix $q\in[1,\infty]$, $r\ge 1$ and $B>0$. Suppose $\|\theta\|_{q^\star}\le B$. Then, for any $\rho_{\mathcal{X}},\rho_{\mathcal{Y}}\ge 0$,
\begin{equation}
\label{eq:pred-upper-joint}
\sup_{Q\in\Ball_{\rm J}^{ q,r}(\mu;\rho_{\mathcal{X}},\rho_{\mathcal{Y}})}\
\cR^{\rm pred}_{q,r}(\widehat m_{\rm LS},\theta;Q)
\ \le\
\Big(\kappa_q(n) \big[s_E(q,r) + \rho_{\mathcal{Y}} + B n^{\chi_+(q,r)} \rho_{\mathcal{X}}\big]\Big)^{\!r},
\end{equation}
where $\kappa_q(n)=n^{2|1/2-1/q|}$ and $s_E(q,r)=(\E\|E\|_q^{ r})^{1/r}$.
\end{theorem}

\begin{proof}
Let $\pi$ be any coupling certifying $Q\in\Ball_{\rm J}^{ q,r}(\mu;\rho_{\mathcal{X}},\rho_{\mathcal{Y}})$, and write $\Delta X:=X'-X$, $\Delta Y:=Y'-Y$. For the least-squares predictor $\widehat m_{\rm LS}(X',Y')=P_{X'}Y'$,
\[
\widehat m_{\rm LS}(X',Y')-X'\theta
 = P_{X'}(Y'-X'\theta)
 = P_{X'}\big((Y-X\theta)+\Delta Y-\Delta X \theta\big).
\]
By Lemma~\ref{lem:opnorm}, $\|P_{X'}\|_{q\to q}\le \kappa_q(n)$. Using the triangle inequality in $\ell_q$,
\[
\|P_{X'}(Y'-X'\theta)\|_q
\ \le \kappa_q(n) \big(\|E\|_q+\|\Delta Y\|_q+\|\Delta X \theta\|_q\big).
\]
By Minkowski’s inequality in $L_r(\pi)$,
\[
\Big(\E_\pi\|P_{X'}(Y'-X'\theta)\|_q^{ r}\Big)^{\!1/r}
\ \le \kappa_q(n) \Big( (\E\|E\|_q^{ r})^{1/r}
\ + (\E_\pi\|\Delta Y\|_q^{ r})^{1/r}
\ + (\E_\pi\|\Delta X \theta\|_q^{ r})^{1/r} \Big).
\]
By the joint-ball constraints,
$(\E_\pi\|\Delta Y\|_q^{ r})^{1/r}\le \rho_{\mathcal{Y}}$, and by Lemma~\ref{lem:DX-theta-sharp},
$(\E_\pi\|\Delta X \theta\|_q^{ r})^{1/r}\le \|\theta\|_{q^\star} n^{\chi_+(q,r)} \rho_{\mathcal{X}}\le B n^{\chi_+(q,r)} \rho_{\mathcal{X}}$.
Combining these bounds yields \eqref{eq:pred-upper-joint}.
\end{proof}

\begin{theorem}[Estimation under joint shift; general $(q,r)$]
\label{thm:est-joint}
Fix $q\in[1,\infty]$, $r\ge 1$ and $B>0$. Suppose $\|\theta\|_{q^\star}\le B$ and, along the coupling, $s_{\min}(X),s_{\min}(X')\ge \tau>0$. Then,
\begin{equation}
\label{eq:est-upper-joint}
\sup_{Q\in\Ball_{\rm J}^{ q,r}(\mu;\rho_{\mathcal{X}},\rho_{\mathcal{Y}})}\
\cR^{\rm est}_{q,r}(\widehat\theta_{\rm LS},\theta;Q)
\ \le\
\Big(K_{n,p,q} \tfrac{1}{\tau} \big[s_E(q,r) + \rho_{\mathcal{Y}} + B n^{\chi_+(q,r)} \rho_{\mathcal{X}}\big]\Big)^{\!r}.
\end{equation}
\end{theorem}

\begin{proof}
Due to the singular value lower bound, $\mathrm{rank}(X')=p$ along the coupling.
Then
\(
\widehat\theta_{\rm LS}(X',Y')-\theta
 = X'^\dagger(Y'-X'\theta).
\)
Hence,
\(
\|\widehat\theta_{\rm LS}(X',Y')-\theta\|_q
\ \le \|X'^\dagger\|_{q\to q}\|Y'-X'\theta\|_q.
\)
Since
\(
\|X'^\dagger\|_{q\to q}\le K_{n,p,q} \tau^{-1},
\)
due to Lemma~\ref{lem:opnorm},
 the result follows from Theorem \ref{thm:pred-joint}.
\end{proof}

\subsubsection{Lower Bounds under the Joint Norm}
\label{lbsj}

Next, we go on to provide the lower bounds for OLS under joint $X-Y$ perturbations, 
These will follow from the results below. 
We abbreviate $s_E=s_E(q,r)$ and $s_{\rm pred}=s_{\rm pred}(q,r)$.

\begin{theorem}[Prediction lower bound under JC]
\label{thm:JC-lb-struct-proof}
Fix $q\in[1,\infty]$, $r\in[1,\infty)$, $B>0$, and budgets $\rho_{\mathcal{X}},\rho_{\mathcal{Y}}\ge0$. 
Suppose that, conditional on $X$, the distribution of $P_X E$ is symmetric about zero. 
Then, for any predictor $\widehat m$, \eqref{eq:JC-lb-joint-struct} holds.
\end{theorem}

\begin{proof}
We prove the 
validity of the 
two terms on the right-hand side separately by exhibiting two feasible adversarial subfamilies in the JC ball.

Let $b:=P_X1_n\in\R^n$ (if $b=0$, the intercept term below vanishes and there is nothing to prove for that term), and let
$v_b=X^\dagger b$. 

\medskip
\noindent\textit{Adversary and column update.}
For parameters $\alpha_{\mathcal{Y}},\gamma_{\mathcal{Y}},\tilde\eta\ge0$ and an independent Rademacher random variable $S\in\{\pm1\}$, define the perturbations
\begin{equation}\label{eq:adversary-CS}
\Delta Y\ := \alpha_{\mathcal{Y}} E + S \gamma_{\mathcal{Y}} b,
\qquad
X'\ :=\ XW,\quad W:=I_p - S \tilde\eta  v_b v_b^\top.
\end{equation}
Then $\Delta X:=X'-X=-S \tilde\eta (Xv_b)v_b^\top=-S \tilde\eta b v_b^\top$, so each row
is perturbed by 
$\Delta x_i^\top=-S \tilde\eta b_i v_b^\top\in\R^{1\times p}$.

\medskip
\noindent\textit{Invertibility and column space preservation.}
Since $W=I_p - S \tilde\eta v_b v_b^\top$ is a rank‑one update of the identity, it is invertible iff $1 - S \tilde\eta v_b^\top v_b\neq 0$, i.e., if $|\tilde\eta| \|v_b\|_2^2 \neq 1$.
We choose
\begin{equation*}
\tilde\eta_0 :=
\frac{\rho_{\mathcal{X}}}{\ \|v_b\|_q \big(n^{-1/r}\|b\|_r\big)}.
\end{equation*}
If 
$|\tilde\eta_0| \|v_b\|_2^2 \neq 1$,
then we choose $\tilde\eta =\tilde\eta_0$.
Otherwise we choose $\tilde\eta =\tilde\eta_0-\Upsilon$, for $\Upsilon>0$ arbitrarily small.
With this choice, $W$ is invertible a.s., hence
\begin{equation}\label{eq:colspace-preserve}
\mathrm{col}(X')=\mathrm{col}(X)\qquad\text{and}\qquad P_{X'}=P_X\quad\text{a.s.}
\end{equation}

\medskip
\noindent\textit{Separated budgets.}
The joint budgets in \eqref{eq:Ball-XY} are obeyed whenever
\begin{equation}\label{eq:budgets-CS}
\big(\E\|\Delta Y\|_q^r\big)^{1/r} \le \alpha_{\mathcal{Y}} s_E+\gamma_{\mathcal{Y}} \|b\|_q \le \rho_{\mathcal{Y}},
\,
\Big(\E\tfrac1n\textstyle\sum_i \|\Delta x_i\|_q^r\Big)^{1/r}
 = |\tilde\eta| \|v_b\|_q \Big(\tfrac1n\sum_i |b_i|^r\Big)^{1/r} \le \rho_{\mathcal{X}}.
\end{equation}

\medskip
We also have $X'\theta=XW\theta=X\theta - S \tilde\eta \langle v_b,\theta\rangle b$, and therefore
\begin{equation}\label{eq:resid-CS}
Y'-X'\theta=(1+\alpha_{\mathcal{Y}}) E + S \Gamma b,
\quad \Gamma:=\gamma_{\mathcal{Y}}+\tilde\eta \langle v_b,\theta\rangle .
\end{equation}

\medskip
\emph{Reduction to equivariant predictors.}
We call a measurable predictor $\widehat m$ $\mathrm{col}(X)$–\emph{equivariant} if
\begin{equation}\label{eq:eqv-def}
\widehat m(X',Y'+t)=\widehat m(X',Y')+t\qquad\forall t\in\mathrm{col}(X)\quad\text{a.s.}
\end{equation}
The loss and the adversaries above enjoy an additive translation symmetry along $\mathrm{col}(X)$: both the target $X'\theta$ and the perturbation $S \Gamma b$ lie in $\mathrm{col}(X)$, and adding any $t\in\mathrm{col}(X)$ to $Y'$ while simultaneously adding $t$ to the predictor output does not change the loss. The Hunt–Stein principle (Lemma \ref{lem:equiv-minimax}) applied to the addition group restricted to $\mathrm{col}(X)$ shows that, for the present adversarial family, it suffices to consider predictors satisfying \eqref{eq:eqv-def}.

Let $\widehat m$ satisfy \eqref{eq:eqv-def}. Using \eqref{eq:resid-CS} and \eqref{eq:eqv-def},
\begin{equation}\label{eq:eqv-decomp}
\widehat m(X',Y')-X'\theta
=\widehat m\big(X', (1+\alpha_{\mathcal{Y}})E+S \Gamma b\big)
=\underbrace{\widehat m\big(X', (1+\alpha_{\mathcal{Y}})E\big)}_{=:A} + S \Gamma b,
\end{equation}
where $A$ is independent of $S$.

From Koskela's inequality \citep{koskela1979}, which is a generalization of Clarkson's inequality \citep{clarkson1936uniformly}, 
it follows that
\[
\frac12\Big(\|A+\Gamma b\|_q^r+\|A-\Gamma b\|_q^r\Big)
\ge 
w_{r,q}(\|A\|_q^r+\|\Gamma b\|_q^r).
\]
Next, observe that 
$A=\widehat m(X',(1+\alpha_{\mathcal{Y}})E)$ and equivariance implies $\E\|A\|_q^r\ge (1+\alpha_{\mathcal{Y}})^q \E\|P_XE\|_q^r$. 
Taking expectations,
\begin{equation}\label{eq:LB-two-pieces-Clark}
\E\|\widehat m(X',Y')-X'\theta\|_q^r\ge\ 
w_{r,q}\left[
(1+\alpha_{\mathcal{Y}})^q \E\|P_XE\|_q^r\ + \Gamma^q\|b\|_q^r\right].
\end{equation}

\medskip
\noindent\textit{Instantiation (i): Noise–projection term.}
Choose $\alpha_{\mathcal{Y}}:=\rho_{\mathcal{Y}}/s_E$ and $\gamma_{\mathcal{Y}}=\tilde\eta=0$. Then $\Gamma\equiv 0$ and \eqref{eq:budgets-CS} holds. 
By equivariance, 
$A=\widehat m(X',(1+\alpha_{\mathcal{Y}})E)$ differs from $(1+\alpha_{\mathcal{Y}})P_XE$ by a (possibly $X'$–dependent) value in $\mathrm{col}(X)$, i.e., 
$A=(1+\alpha_{\mathcal{Y}})P_XE+c_{X'}$. 
Due to our assumptions, 
the map $c\mapsto \E\|(1+\alpha_{\mathcal{Y}})P_XE + c\|_q^r$, where the expectation is evaluated conditionally on $X$,
is minimized at $c=0$. Hence
conditioning on $X$ and $X'$, it follows that 
\[
\E\|A\|_q^r\ge\ (1+\alpha_{\mathcal{Y}})^r \E\|P_XE\|_q^r
 = \big[\Lambda (s_E+\rho_{\mathcal{Y}})\big]^r,
\]
which is the first term in the lower bound.

\medskip
\noindent\textit{Instantiation (ii): Coherent intercept term.}
Choose $\alpha_{\mathcal{Y}}:=0$ and $\gamma_{\mathcal{Y}}:=\rho_{\mathcal{Y}}/\|b\|_q$.
By Hölder, $\sup_{\theta\in\Theta_B}\langle v_b,\theta\rangle=B \|v_b\|_q$, hence
\[
\sup_{\theta\in\Theta_B}\Gamma \|b\|_q
\ge\
\frac{\rho_{\mathcal{Y}}}{\|b\|_q} \|b\|_q + \tilde\eta B \|v_b\|_q \|b\|_q
\ =\
\rho_{\mathcal{Y}} + B \tilde\eta \|v_b\|_q \|b\|_q.
\]
This becomes
\[
\rho_{\mathcal{Y}} + B \rho_{\mathcal{X}} 
\frac{\|b\|_q}{ n^{-1/r}\|b\|_r } - \Upsilon B \tilde\eta \|v_b\|_q \|b\|_q.
\]
Finally, by norm comparison on $\R^n$,
\[
\frac{\|b\|_q}{ n^{-1/r}\|b\|_r }\ge\ n^{\chi_-(q,r)},
\qquad
\chi_-(q,r):=\min\{1/q, 1/r\},
\]
so
\[
\sup_{\theta\in\Theta_B}\Gamma \|b\|_q
\ge\
\rho_{\mathcal{Y}} + B n^{\chi_-(q,r)} \rho_{\mathcal{X}} - \Upsilon B \tilde\eta \|v_b\|_q \|b\|_q.
\]
Since this holds for all $\Upsilon>0$, we also have 
$\sup_{\theta\in\Theta_B}\Gamma \|b\|_q
\ge\
\rho_{\mathcal{Y}} + B n^{\chi_-(q,r)} \rho_{\mathcal{X}}$.

\medskip
Combining (i) and (ii) and substituting in
\eqref{eq:LB-two-pieces-Clark}
 proves the result.

\end{proof}

{\bf Estimation lower bounds via prediction–to–estimation transfer.}
The estimation lower bounds follow from the above results.
For any estimator $\widehat\theta$ define the induced predictor $\widehat m^{\widehat\theta}(X',Y'):=X'\widehat\theta(X',Y')$, 
and then use the definition of the operator norm.

\subsubsection{Auxiliary results}

\begin{lemma}[Bounding $\|\Delta X \theta\|_q$ from the joint $X$--budget]
\label{lem:DX-theta-sharp}
Let $\Delta X=(\Delta x_1^\top,\dots,\Delta x_n^\top)^\top$ $\in\R^{n\times p}$ and $\theta\in\R^p$, and fix $q\in[1,\infty]$, $r\ge1$.
Then
\begin{equation*}
(\E_\pi\|\Delta X \theta\|_q^{ r})^{1/r}\le \|\theta\|_{q^\star} n^{\chi_+(q,r)} \E\Big(\frac1n\sum_{i=1}^n \|\Delta x_i\|_q^{ r}\Big)^{\!1/r}.
\end{equation*}
\end{lemma}

\begin{proof}
By H\"older, $|\langle \Delta x_i,\theta\rangle|\le \|\Delta x_i\|_q \|\theta\|_{q^\star}$. Let $a_i:=\|\Delta x_i\|_q$ and $a=(a_1,\dots,a_n)$. Then
\(
\|\Delta X \theta\|_q=\big\|(\langle \Delta x_i,\theta\rangle)_{i=1}^n\big\|_q
\ \le \|\theta\|_{q^\star} \|a\|_q.
\)
By 
norm comparison, $\|a\|_q\le n^{(\frac1q-\frac1r)_+} \| a\|_r$, 
which gives 
$\|\Delta X \theta\|_q
\ \le 
\|\theta\|_{q^\star} n^{(\frac1q-\frac1r)_+}
\Big(\sum_{i=1}^n \|\Delta x_i\|_q^{ r}\Big)^{\!1/r}.$
Raising both sides to the $r$th power and taking expectation under $\pi$,
\[
\E_\pi \|\Delta X \theta\|_q^{ r}
\ \le \|\theta\|_{q^\star}^{ r} n^{r(\frac1q-\frac1r)_+} \E_\pi\sum_{i=1}^n a_i^r
 = \|\theta\|_{q^\star}^{ r} n^{r(\frac1q-\frac1r)_+} n \E_\pi \tfrac1n\sum_{i=1}^n \|\Delta x_i\|_q^{ r},
\]
 and taking the $1/r$th power
  proves the conclusion. 
\end{proof}

\begin{lemma}[Operator–norm reductions]\label{lem:opnorm}
For any $n\times n$ matrix $A$, we have $\|A\|_{q\to q}\le \kappa_q(n) \|A\|_{2\to 2}$. For any $p\times n$ matrix $B$,
$\|B\|_{q\to q}\le K_{n,p,q} \|B\|_{2\to 2}$.
In particular $\|P_X\|_{2\to2}=1$ hence $\|P_X\|_{q\to q}\le \kappa_q(n)$, and if $s_{\min}(X)\ge\tau>0$ then $\|X^\dagger\|_{q\to q}\le K_{n,p,q} \tau^{-1}$. 
\end{lemma}

\emph{Proof.} Combine $\|u\|_2\le n^{|1/2-1/q|}\|u\|_q$ and $\|v\|_q\le n^{|1/2-1/q|}\|v\|_2$ to obtain $\|A\|_{q\to q}\le n^{2|1/2-1/q|}\|A\|_{2\to2}$. For $B:\R^n\to\R^p$ the two dimension changes give the stated $K_{n,p,q}$. The spectral–norm facts are standard.
\qed

\subsection{
Random Design IC: Comparison with Liu and Loh (2021)}
\label{ic-reg}

We now present the implications
for the setting where the individual datapoints are perturbed independently and identically. 
We will reuse notation from Section \ref{jclr}, 
with slight adjustments.
Now, 
$\mu$ will denote
the law of i.i.d.~$(X_i,Y_i)$ under the linear model $Y_i=X_i^\top\theta+E_i$. 
Also $\tilde \mu$ denotes the law of the observed $(X_i',Y_i')$.

For budgets $\ep_{\mathcal{X}},\ep_{\mathcal{Y}}\ge 0$, the i.i.d.\ contamination ball on $(X,Y)$ with separated budgets is
\begin{align*}
\Ball_{\mathrm{IC}}^{ q,r}(\mu;\ep_{\mathcal{X}},\ep_{\mathcal{Y}})
:=
\Big\{\ Q&=\tilde\mu^{\otimes n} :\ \exists\ \pi\in\Gamma(\tilde\mu,\mu)\ \text{ s.t. }\\ 
& (\E_\pi\|X'-X\|_q^r)^{1/r}\le\ep_{\mathcal{X}}, \ (\E_\pi|Y'-Y|^r)^{1/r}\le\ep_{\mathcal{Y}}\ \Big\}.
\end{align*}
 Similarly, we will denote by 
$\Ball_{\mathrm{IC}}^{ q,r}(\mu;\ep_{\mathcal{X}},\ep_{\mathcal{Y}},\tau) $
the set of perturbations were in addition,
 for the certifying coupling, one has $s_{\min}(X)\wedge s_{\min}(X')\ge\tau>0$ almost surely.

Due to a similar argument as the one provided earlier, 
the joint contaminations
from \eqref{eq:Ball-XY}
with parameters $\rx = \ex$ and $\ry = n^{1/q}\ey$
are at least as strong of a contamination then the IC contaminations above.
Thus, the JC upper bounds from Theorem \ref{thm:master} with $\rx = \ex$ and $\ry = n^{1/q}\ey$ apply to the current IC setting.

Our main goal in this section is to be able to compare
our results with those of \cite{Liu2021RobustWE}.
 Since they study 
 the case of Euclidean IC $W_1$ perturbations where $q = 2$ and $r=1$,
 we will provide the consequence of our JC upper bounds for this setting.

\begin{corollary}[IC bounds for random design with $q=2$, $r=1$]\label{cor:IC-q2r1}
Consider IC perturbations with 
budgets $\ep_{\mathcal{X}},\ep_{\mathcal{Y}}\ge0$, and a one–sample law $\mu$, 
for 
 $q=2$ and $r=1$.
For the least–squares predictor $\widehat m_{\mathrm{LS}}(X',Y')$,
\begin{equation}\label{eq:IC-q2r1-upper-pred}
\sup_{Q\in\Ball_{\mathrm{IC}}^{ 2,1}(\mu;\ep_{\mathcal{X}},\ep_{\mathcal{Y}})}\
\cR_{2,1}^{\mathrm{pred}}(\widehat m_{\mathrm{LS}},\theta;Q)
\ \le\
\E\|E\|_2\ 
+
n^{1/2}\ey+ nB\ex.
\end{equation}
For the least–squares estimator $\widehat\theta_{\mathrm{LS}}(X',Y')=X'^\dagger Y'$, then
\begin{equation}\label{eq:IC-q2r1-upper-est}
\sup_{Q\in\Ball_{\mathrm{IC}}^{ 2,1}(\mu;\ep_{\mathcal{X}},\ep_{\mathcal{Y}},\tau)}\
\cR_{2,1}^{\mathrm{est}}(\widehat\theta_{\mathrm{LS}},\theta;Q)
\ \le\
 \E\|(X')^\dagger E\|_2\ + \tfrac{1}{\tau} (n^{1/2}\ey+ nB\ex).
\end{equation}
\end{corollary}

{\bf Relation to \cite{Liu2021RobustWE}.}
We now compare 
our bounds with the results reported as Theorems~15--17 in \cite{Liu2021RobustWE}. 
In their work, the joint law on $(X,Y)\in\R^p\times\R$ is
\(
X\sim\mathcal N(0,\Sigma^\star)\),
\(Y\mid X\sim\mathcal N(X^\top\theta^\star,\sigma^2)\),
and the observed i.i.d.\ sample $(X_i',Y_i')\sim P$ satisfies a Wasserstein constraint $W_{2,1}(P,P_{\theta^\star})\le \ep$. 
Due to the triangle inequality, this corresponds to our setting with 
any 
$\ex, \ey$ such that
$\ep=\ey+ B\ex$.

Their results establish that the minimax optimal rate (in probability)
is $\Theta(\max(\sqrt{p/n}, \ep))$ both for the $\ell_2$ estimation error and
a prediction error of the form $\E_{X_{n+1}}(X_{n+1}^\top \hat{\theta} - Y_{n+1})^2$.
 This prediction error is evaluated on a new test datapoint and is thus different from ours, which is evaluated in-sample on the training datapoints. In fact, as \cite{Liu2021RobustWE} already observe, this out-of-sample prediction error has the same order as the estimation error when the feature population covariance matrix is well conditioned. Therefore, we will focus on comparing the results for the estimation error.

For Gaussian data with $E \sim \N(0, I_n)$, 
$\E\|X^\dagger E\|^2_2 = p/(n-p-1)$, 
so 
\begin{align*}
\E\|(X')^\dagger E\|_2&\le \sqrt{p/(n-p-1)}+\E\|((X')^\dagger-X^\dagger) E\|_2\\
&\le \sqrt{p/(n-p-1)}+\sqrt{\E\|(X')^\dagger-X^\dagger\|^2_{\Fr}}
\end{align*}
Moreover, due to standard lower bounds on eigenvalues of random matrices \citep[see e.g.,][]{vershynin2018high}, the smallest singular value $s_{\min}(X)$ is expected to be of order $\sqrt{n}$ when $n\ge cp$ for some $c<1$.
 Therefore, it is reasonable to take $\tau = \Theta(n^{1/2})$.
Therefore, 
due to \eqref{eq:IC-q2r1-upper-est},
our rate
becomes\footnote{We use the Bachmann-Landau asymptotic notation $O(\cdot)$ to absorb constant factors. } $O(\max(\sqrt{\E\|(X')^\dagger-X^\dagger\|^2_{\Fr}},\sqrt{p/(n-p)}, \ey, n^{1/2}B\ex))$.
This 
recovers that in \cite{Liu2021RobustWE} when 
the following conditions hold: 
$n\ge cp$ for some $c<1$,
$\sqrt{\E\|(X')^\dagger-X^\dagger\|^2_{\Fr}}
=O(\max(\sqrt{p/n}, B\ex+\ey))$,
and $\sqrt{n} B\ex
= O(\max(\sqrt{p/n}, \ey))
$.

It is reassuring that we obtain the same rates under these conditions. 
Moreover, it is reasonable that our results come under additional conditioning of the original and perturbed feature matrices, because we are using linear regression, which can behave poorly if the data is ill-conditioned.
The condition that the perturbation of $X$ must be small in order to recover the above rate is also reasonable, since linear regression can become unstable if the feature matrix is strongly perturbed towards singularity. 
In contrast, \cite{Liu2021RobustWE} use alternatives forms of empirical risk minimization, which may provide some additional regularization to mitigate the effects of ill-conditioned data. 
In future work, it would be of interest to determine whether certain forms of regularization can also make linear regression more stable.

\section{Simulations}\label{sec: simulations}
To evaluate our theoretical results, we perform simulations to empirically approximate the minimax risk. 
For all simulations, we repeat our results over $5000$ trials and report the mean value. 
We provide all code to replicate our experiments at: \url{https://github.com/patrickrchao/Distribution-Shift-Experiments}.

In situations where we do not 
know the minimax optimal estimators and contaminations,
we consider a set of estimators $\htheta_1,\htheta_2,\ldots,\htheta_{m}$ and a set of perturbation distributions $\nu_1,\nu_2,\ldots,\nu_{l}$ for a perturbation level $\ep$. We compute an empirical loss $\ell_{i,j}=\cL(\htheta_i,\nu_j)$, averaged over the $5000$ trials, for each pair of estimator and perturbation. We then plot $\min_{i} \max_j \ell_{i,j}$, as an approximation to the minimax risk. 
We consider a variety of perturbations and estimators to attempt to adequately capture the minimax risk.

For the following sections, we provide the specific hyperparameters chosen, e.g., the dimension $p$, and the list of estimators and perturbations evaluated.
 Choosing $r = q = 2$, 
we represent each perturbation in the form of the perturbed value $X_i'=f(X_i,\theta)$ where $\mathbb{E}{\|X_i'-X_i\|_2^2}\le \ep^2$, e.g., $X_i'=X_i+\ep \delta$.
Further, we let $e_1$ be the first standard basis vector, with the value equal to unity in the first index and zeros otherwise.

\subsection{Location Simulations}\label{subsec: gaussian simulations}
We provide details for the plot in \cref{fig: gaussian perturbation risk}. In the location estimation problem, we know the exact IC and JC minimax risks, as well as the minimax optimal estimators. We verify that we empirically obtain the same values. 
We sample $X_i\simiid \cN(\theta,\Sigma)$ for $i\in[n]$, where $n=10$, $p=3$, $\theta=0$, and $\Sigma=\mathrm{diag}(1/6,2/6,3/6)$ (so that $\Tr[\Sigma]=1$), and observe $X_i'$, where $W_{2,2}(X_i,X_i')\le \ep$. We provide the perturbations in \cref{table: gaussian perts} and estimators in \cref{table: gaussian estimators}.

{
\renewcommand{\arraystretch}{1.2}

\begin{table}[t]
    \centering
    \caption{Location simulation perturbations. Perturbation 3 uses the values of $\zeta$ and $\psi$ in \cref{eq: zeta psi}.}
    \begin{tabular}{l c @{\hskip 1cm} l }
        \toprule
        Number  & Type & Perturbation \\
        \midrule
        3 & IC                & $X_i'=X_i+\zeta(X_i-\theta)+\psi\delta$\\
        4 & JC                & $X_i'=X_i+\ep \sqrt{n/\Tr[\Sigma]}(\bar X-\theta)$\\
        \bottomrule
    \end{tabular}
    \label{table: gaussian perts}
\end{table}
}

{
\renewcommand{\arraystretch}{1.2}
\begin{table}[t]
\caption{Location simulation estimators.}
    \centering
    \begin{tabular}{l l }
        \toprule
        Number   & Estimator \\
        \midrule
        1               & Sample Mean\\
        2                 & Sample Median (per dimension)\\
        \bottomrule
    \end{tabular}
    \label{table: gaussian estimators}
\end{table}
}

\subsection{Linear Regression Simulations}\label{subsec: lr simulations}
We provide details for the plots in \cref{fig: lr}. 
We sample $Y\sim \cN(X\theta,\Sigma)$ for $X\in \R^{n\times p}$, where $n=10$, $p=5$, $\theta=1_{5}$, 
and where $X$ has i.i.d.~standard normal entries divided by $\sqrt{n}$ so that the columns have squared norm equal to unity in expectation. We observe $Y'$, where $W_{2,2}(Y,Y')\le \ep$. 
We consider squared error, $\rho^2(\htheta,\theta)=\|\htheta-\theta\|_2^2$, and scaled prediction error, $\rho^2_X(\htheta,\theta)=\|X(\htheta-\theta)\|_2^2/n$.
We also consider homoskedastic and heteroskedastic error, where $\Sigma=I_n/100$ and $\Sigma=\diag(1,2,\ldots,10)/200$ 
respectively, chosen so that $\Sigma$ and $\ep$ are on similar scales.
We provide the perturbations in \cref{table: lr perts}.

{
\renewcommand{\arraystretch}{1.2}
\begin{table}[t]
 \caption{Linear regression simulation perturbations. The values of $\zeta$ in Perturbations 3 is chosen to satisfy $\mathbb{E}{\|Y'-Y\|_2^2}\le \ep^2$. Perturbation 4 uses $v_n$, the left singular vector of $X$ corresponding to the smallest singular value.}
    \centering
    \begin{tabular}{l c @{\hskip 1cm} l }
        \toprule
        Number  & Type & Perturbation \\
        \midrule
        1 & JC                & $Y'=Y+\sqrt{n}\ep \cdot e_1$\\
        2 & JC                & $Y'=Y+\sqrt{n}\ep \cdot 1_n/\sqrt{n}$\\
        3 & JC                & $Y'=Y+\zeta (P_X Y -X\theta)$\\
        4 & JC                & $Y'=Y+\ep v_n$\\
        \bottomrule
    \end{tabular}
   
    \label{table: lr perts}
\end{table}
}

\end{document}